\documentclass[preprint,12pt]{elsarticle}

\usepackage{amsthm}
\usepackage{lineno}
\usepackage{bm} 
\usepackage[colorlinks=true]{hyperref}
\usepackage{wrapfig}
\usepackage{booktabs} 
\usepackage{algorithm}
\usepackage{algpseudocode}
\usepackage{graphicx}
\usepackage{amsmath}
\usepackage{amssymb}
\usepackage{amsfonts}
\usepackage{multirow}
\usepackage{graphicx}
\usepackage{colortbl}
\usepackage{subfigure}
\usepackage{indentfirst}
\usepackage{ulem}
\usepackage{url}
\usepackage{verbatim}
\usepackage{adjustbox}
\usepackage{wrapfig}
\newtheorem{theorem}{Theorem}
\newtheorem{definition}{Definition}
\graphicspath{{./images/}}
\newtheorem{assumption}{Assumption}
\newtheorem{property}{Property}

\newtheorem{proposition}{Proposition}

\newcommand{\sgn}{\mathrm{sgn}}
\journal{Journal of Computational Physics}

\begin{document}

\begin{frontmatter}


\title{An Imbalanced Learning-based Sampling Method for Physics-informed Neural Networks}

\author[first]{Jiaqi Luo}
\ead{jluo@fields.utoronto.ca}
\affiliation[first]{organization={The Fields Institute for Research in Mathematical Sciences},
            addressline={222 College Street}, 
            city={Toronto},
            postcode={M5T3J1}, 
            state={Ontario},
            country={Canada}}
            
\author[second]{Yahong Yang}
\ead{yxy5498@psu.edu}
\affiliation[second]{organization={Department of Mathematics,Pennsylvania State University},
addressline={208 McAllister Building},
city = {State College},
postcode={16801},
state={PA},
country={USA}}


\author[third]{Yuan Yuan}
\ead{y.yuan@dukekunshan.edu.cn}
\affiliation[third]{organization={Data Science Research Center, Duke Kunshan University},
            addressline={No.8 Duke Avenue}, 
            city={Kunshan},
            postcode={215000}, 
            state={Jiangsu Province},
            country={China}}
\author[fourth,third]{Shixin Xu\corref{cor}}
\cortext[cor]{Corresponding author}
\ead{shixin.xu@dukekunshan.edu.cn}
\affiliation[fourth]{organization={Zu Chongzhi Center for Mathematics and Computational Sciences, Duke Kunshan University},
            addressline={No.8 Duke Avenue}, 
            city={Kunshan},
            postcode={215000}, 
            state={Jiangsu Province},
            country={China}}
            
\author[second]{Wenrui Hao}
\ead{wxh64@psu.edu}

\begin{abstract}
This paper introduces Residual-based Smote (RSmote), an innovative local adaptive sampling technique tailored to improve the performance of Physics-Informed Neural Networks (PINNs) through imbalanced learning strategies. Traditional residual-based adaptive sampling methods, while effective in enhancing PINN accuracy, often struggle with efficiency and high memory consumption, particularly in high-dimensional problems. RSmote addresses these challenges by targeting regions with high residuals and employing oversampling techniques from imbalanced learning to refine the sampling process. 
Our approach is underpinned by a rigorous theoretical analysis that supports the effectiveness of RSmote in managing computational resources more efficiently. Through extensive evaluations, we benchmark RSmote against the state-of-the-art Residual-based Adaptive Distribution (RAD) method across a variety of dimensions and differential equations. The results demonstrate that RSmote not only achieves or exceeds the accuracy of RAD but also significantly reduces memory usage, making it particularly advantageous in high-dimensional scenarios. These contributions position RSmote as a robust and resource-efficient solution for solving complex partial differential equations, especially when computational constraints are a critical consideration.
\end{abstract}




\end{frontmatter}


\section{Introduction}
\label{s:intro}

Partial differential equations (PDEs) are powerful modeling techniques with numerous applications in materials, engineering, and physics. One of the primary challenges in PDE research is finding numerical solutions. Fortunately, several fantastic methods can achieve excellent performance while benefiting from theoretical guarantees, such as the finite element method \cite{hughes2012finite}, finite difference method \cite{thomas2013numerical}, and spectral method \cite{shen2011spectral}. 
However, some bottlenecks limit their applications.
First, traditional methods encounter the "curse of dimensionality" \cite{hu2024tackling, powell2007approximate}, which poses significant challenges for PDE solvers because the computational cost increases dramatically when the dimensionality of the problem becomes very high, making the computation time of the problem unacceptable.
Second, traditional solvers have difficulties incorporating data from experiments and cannot handle situations where the governing equations are (partially) unknown.
Further, these methods struggle with handling complex geometry.
These limitations restrict the applicability and scalability of many existing numerical techniques.

With remarkable success in computer vision \cite{he2016deep} and natural language processing \cite{vaswani2017attention}, deep learning \cite{lecun2015deep} has been proven to be an efficient approach to solving high-dimensional problems. Besides, deep learning is flexible, mesh-free, and easy to incorporate data. As a result, deep learning approaches offer an alternative for tackling PDE problems that are difficult for conventional methods \cite{long2018pde, long2019pde, sirignano2018dgm, weinan2018deep,zang2020weak}.

In deep learning for PDEs, neural networks are employed to learn and approximate the solution to a given PDE. These neural networks, often referred to as "PDE solvers", are trained on available data or by leveraging the known PDE equations. They take inputs, such as spatial coordinates and initial/boundary conditions, and produce an approximation of the solution. 
One of the most effective deep learning-based approaches is called physics-informed neural networks (PINNs) \cite{raissi2019physics, karniadakis2021physics}. 
PINNs aim to learn the underlying physics of a system by combining a neural network with the governing equations that describe that system. They have been applied to tackle diverse problems in computational science and engineering \cite{raissi2020hidden,yazdani2020systems,lu2021physics}. Additionally, PINNs have been extended to solve other types of PDEs, such as integro-differential equations \cite{lu2021deepxde}, fractional PDEs \cite{pang2019fpinns}, stochastic PDEs \cite{zhang2019quantifying}, and nonlinear PDEs \cite{zheng2024hompinns,huang2022hompinns}.

Directly training a PINN to obtain accurate and convergent predictions for a wide range of PDE problems with increasing levels of complexity can be challenging due to the lack of theoretical guarantees \cite{siegel2023greedy,chen2022randomized}.
Residual-based adaptive sampling methods in PINNs can accelerate convergence and improve performance, playing an important role in training deep neural networks \cite{wu2023comprehensive}. The effect of residual points on PINNs is analogous to the effect of mesh points on FEM, suggesting that the location and distribution of these residual points are crucial to the performance of PINNs. However, existing methods face challenges in efficiently handling high-dimensional problems. Global methods that estimate the distribution over the entire domain often require significant computational resources, particularly GPU memory, as evidenced by the increasing memory usage of RAD methods with higher dimensions and sample sizes in our results.

To address these issues, we propose an imbalanced learning-based local adaptive sampling method called RSmote. This method is motivated by the observation that regions with large residuals are typically localized and is based on the theoretical findings that a less smooth part significantly requires more samples. By focusing on these specific regions, we can enhance efficiency while using fewer resources. Our results demonstrate that RSmote consistently achieves better or comparable accuracy to RAD methods across various dimensions, while significantly reducing memory usage.

Specifically, our method first categorizes residual points into two classes: negative (small residuals) and positive (large residuals). We then apply over-sampling techniques from imbalanced learning to generate new samples in the positive class. Finally, we resample the dataset to further optimize memory usage. This approach allows us to maintain high accuracy even in high-dimensional problems (up to d=100 in our experiments) while keeping memory requirements low. The efficiency of RSmote is particularly evident in higher dimensions, where it outperforms the state-of-the-art methods in both accuracy and resource utilization, making it especially valuable for complex, high-dimensional PDEs or scenarios with limited computational resources.

The main contributions are summarized as follows:
\begin{itemize}
    \item  \textbf{Introduction of RSmote}: A novel local adaptive sampling method specifically designed for PINNs, incorporating techniques from imbalanced learning to improve performance in high-dimensional problems.
    \item \textbf{Efficiency Enhancement}: RSmote addresses the inefficiencies of existing global adaptive sampling methods by focusing on localized regions with high residuals, leading to significant improvements in resource utilization and memory efficiency.
    \item \textbf{Theoretical Support}: The paper provides a theoretical analysis to validate the effectiveness of the local sampling method, adding a solid foundation to the empirical observations.
    \item \textbf{Empirical Validation}: Extensive comparisons with the state-of-the-art RAD method demonstrate that RSmote consistently achieves comparable or superior accuracy while significantly reducing memory usage, particularly in high-dimensional cases.
\end{itemize}

\section{Preliminaries}
\label{s:pre}
\subsection{Physics-informed neural networks}
Let $\Omega \subseteq \mathbb{R}^d$ be a spatial domain and $\mathbf{x} \in \Omega$ be a spatial variable.
The PDE problem is to find a solution $u(\mathbf{x})$ such that
\begin{equation}
\label{e.pde}
\begin{aligned} 
&\mathcal{L}u(\mathbf{x}) = f(\mathbf{x}), ~\forall{\mathbf{x}} \in \Omega \\ 
&\mathcal{B}u(\mathbf{x}) = g(\mathbf{x}), ~\forall{\mathbf{x}} \in \partial\Omega,
\end{aligned}  
\end{equation}
where $\mathcal{L}$ is the partial differential operator, $\mathcal{B}$ is the boundary operator, $f(\mathbf{x})$ is the source function, and $g(\mathbf{x})$ represents the
boundary conditions.

We consider a general form of a PINN, where the neural network is used to approximate a function that satisfies a PDE along with some observed data.

Let $\hat{u}(\mathbf{x}; \Theta)$ be the neural network approximated solution with parameters $\Theta$. 
The PINN is trained to minimize the following loss function:
\begin{equation}
\label{e.obj}
   \min_{\Theta} \mathbb{E}_{\mathbf{x}\in \Omega}\|\mathcal{L}\hat{u}(\mathbf{x}; \Theta)-f(\mathbf{x})\| + \gamma\ \mathbb{E}_{\mathbf{x}\in \partial\Omega}\|\mathcal{B}\hat{u}(\mathbf{x}; \Theta)-g(\mathbf{x})\|, 
\end{equation}
where $\|\cdot\|$ is usually the $L_2$-norm and $\gamma$ is a parameter for weighting the sum.

\subsection{Residual-based sampling method for PINNs}

Residual point sampling methods can be divided into two main categories: uniform sampling and nonuniform sampling. Uniform sampling involves placing points across the computational domain either on a uniform grid or scattered randomly according to a continuous uniform distribution \cite{raissi2019physics, pang2019fpinns}. While effective for straightforward PDEs, this method may lack efficiency when dealing with more complex equations \cite{wu2023comprehensive}.

To enhance accuracy, researchers have delved into nonuniform sampling approaches. In 2019, Lu et al. \cite{lu2021deepxde} introduced the first adaptive nonuniform sampling method for PINNs, dubbed residual-based adaptive refinement (RAR), inspired by adaptive mesh refinement in finite element methods. This method identifies regions with significant PDE residuals and adds new points accordingly. Another technique introduced in 2021 \cite{nabian2021efficient} resamples all residual points based on a probability density function (PDF) proportional to the PDE residual.

In the pursuit of greater sampling efficiency and accuracy, recent studies \cite{wu2023comprehensive} propose two novel residual-based adaptive sampling methods: residual-based adaptive distribution (RAD) and residual-based adaptive refinement with distribution (RAR-D). These methods dynamically adjust the distribution of residual points during training based on PDE residuals. Additionally, some researchers \cite{gao2023failure, gao2023failure2} have introduced an effective failure probability derived from the residual as a posterior error indicator to generate new training points, similar to adaptive finite element methods, resulting in improved performance.

Other methodologies also utilize residual points but do not estimate the probability distribution. For instance, Gu et al. \cite{gu2021selectnet} introduced SelectNet, a novel self-paced learning framework aimed at enhancing the convergence of network-based least squares models. In DAS-PINNs \cite{tang2023pinns}, researchers model the residual as a probability density function, approximating it with a deep generative model. Gao et al. \cite{gao2023active} leverage active learning to adaptively select new points, while \cite{zeng2022adaptive} proposes various adaptive training strategies, including adaptive loss function, activation function, and sampling methods, to improve accuracy. Moreover, \cite{mao2023physics} employs residual points and gradients to address equations with sharp solutions.

\subsection{Imbalanced Learning}
Imbalanced learning \cite{he2009learning} refers to a scenario in machine learning where the distribution of classes in the training dataset is highly uneven, with one class significantly outnumbering the others. This imbalance can pose challenges for models, as they may become biased toward the majority class and perform poorly on minority classes. 

To address this issue, various techniques have been developed, and one such method is the Synthetic Minority Over-sampling Technique (SMOTE) \cite{chawla2002smote}. SMOTE is an algorithm designed to balance class distribution by generating synthetic samples for the minority class. It works by creating synthetic instances along the line segments connecting existing minority class instances. By doing so, SMOTE enhances the representation of the minority class, enabling the model to better generalize and improve its performance on underrepresented classes.

The SMOTE algorithm focuses on the minority class $C$. For each instance $\mathbf{x_i}$ in $C$, SMOTE generates synthetic instances by interpolating between $\mathbf{x_i}$ and its $k$-nearest neighbors in the feature space.

\textbf{Step 1: Nearest Neighbors}
Use a distance metric (e.g., Euclidean distance) to find the $k$ nearest neighbors of $\mathbf{x_i}$ in the feature space.
\[ \text{NearestNeighbors}(\mathbf{x_i}, k) = \{\mathbf{x_{i1}}, \mathbf{x_{i2}}, \ldots, \mathbf{x_{ik}}\} \]

\textbf{Step 2: Synthetic Instance Generation}
For each neighbor $\mathbf{x_{ij}}$, where $j = 1, 2, \ldots, k$, generate a synthetic instance $\mathbf{\hat{x}_{j}}$ using the following formula:
\[ \mathbf{\hat{x}_{j}} = \mathbf{x_i} + \textcolor{black}{\alpha} (\mathbf{x_{ij}} - \mathbf{x_i}) \]
where \textcolor{black}{$\alpha$} is a random value between 0 and 1.

\textbf{Step 3: Combine Original and Synthetic Instances}
Combine the original instances with the newly generated synthetic instances. This increases the size of the minority class $C$ and balances the class distribution.

This approach is particularly beneficial in situations where the rarity of certain classes can lead to suboptimal model performance and misclassification.

\section{Imbalanced learning-based sampling algorithm}
\label{s:method}

\subsection{Imbalanced phenomenon}
When a deep neural network is trained to approximate the solution of a PDE, a notable phenomenon occurs: the gradual reduction in the number of large residual points. This trend is a key indicator of the network's improving performance and its increasing ability to capture the underlying structure of the PDE solution.

We use the following simple case to further explain this phenomenon. Consider the 1-D elliptic equation:
\begin{equation}
    -\bigtriangleup u = 200(\tanh(10x))^3+200\tanh(10x),~~ x\in [-5, 5]\\
\end{equation}

Fig.~\ref{f.imb1} shows the changes of the PDE dynamic, PDE residuals, and histograms of the residuals with the iteration increasing. In the early iterations (as seen in iteration 100), the PINN struggles to capture the full dynamics of the PDE, leading to noticeable discrepancies between the predicted and numerical solutions, particularly near regions of steep gradients or discontinuities. 

As training progresses to iteration 500, the PINN begins to improve its approximation of the PDE's dynamics, but the residuals still indicate that certain regions are being learned better than others. This imbalance is reflected in the residuals, \textcolor{black}{where the heights of the two bars differ significantly.}

By iteration 1000, the PINN's prediction closely aligns with the numerical solution across most of the domain. However, the residuals and their distribution continue to show a non-uniform pattern, implying that while the overall error is decreasing, certain parts of the solution space are still not as well-learned as others. 
This persistence of imbalanced residuals, even with more training, highlights the challenge of achieving uniform accuracy across the entire domain, which is a key characteristic of the imbalanced phenomenon in training PINNs.

\begin{figure}[!ht]
    \centering
    \subfigure[Iteration=100]{\includegraphics[width=\linewidth]{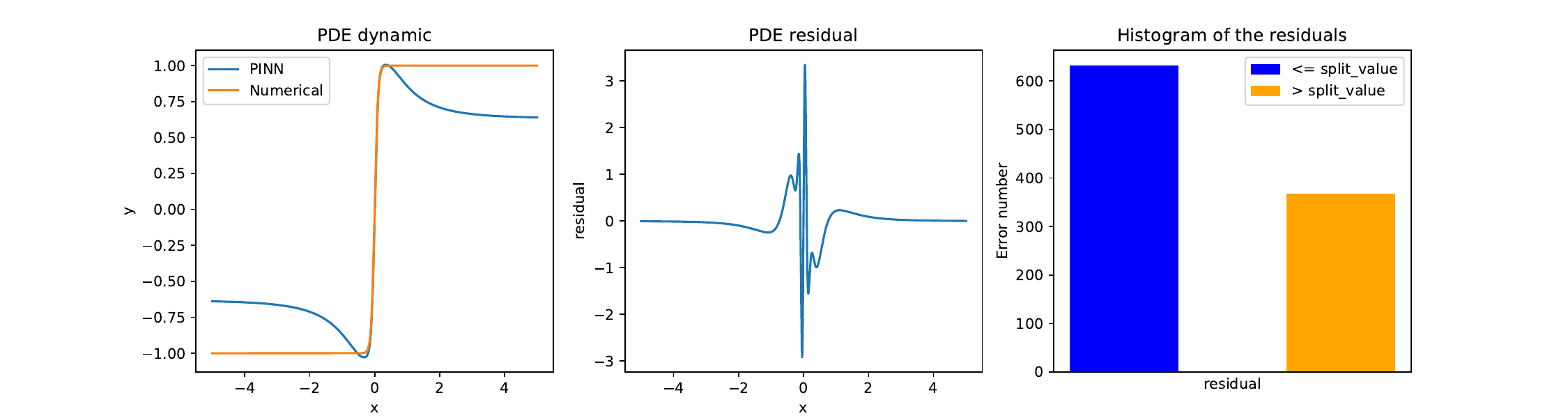}}
    \subfigure[Iteration=500]{\includegraphics[width=\linewidth]{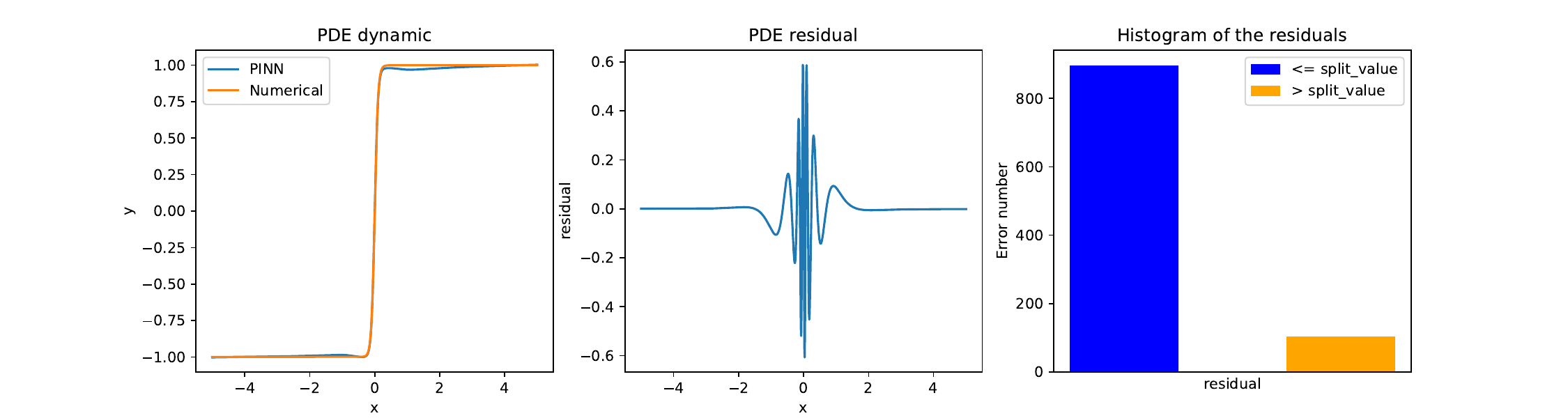}}
    \subfigure[Iteration=1000]{\includegraphics[width=\linewidth]{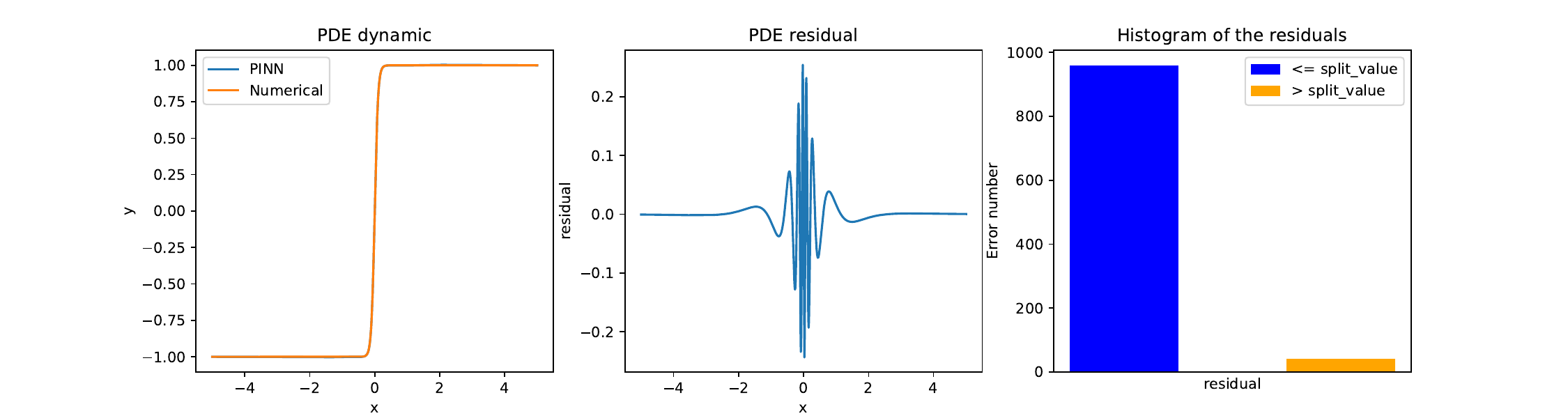}}

    \caption{\textcolor{black}{Evolution of the imbalanced phenomenon in PDE dynamics over iterations. The figure displays the PDE dynamic, PDE residuals, and histograms of the residuals for three different iterations (100, 500, and 1000). The left column shows the comparison between the predicted solution from the PINN and the numerical solution. The middle column visualizes the residuals of the PDE at each point, and the right column presents the distribution of these residuals splited by a fixed threshold. The plots demonstrate how the imbalanced phenomenon evolves as training progresses, with the residuals showing different patterns at each stage. Here the split value is chosen as 0.1. }}
    \label{f.imb1}
\end{figure}

This phenomenon underscores the difficulty of ensuring that the PINN learns the PDE dynamics evenly across the domain, often requiring specialized techniques or loss functions to address the imbalance and achieve a more uniform approximation. Therefore, we conclude that there are only a few large residual points when the network convergences.
The following property gives a mathematical intuition of the imbalanced phenomenon:

\begin{property}[Imbalanced residuals]
Let $r_1, r_2, ..., r_n$ be non-negative residuals. If  $\sum_{i=1}^n r_i < \epsilon$, where $\epsilon$ is the total error tolerance, then for any given threshold $0<\tau<\epsilon$,  the number of residual $r_i$,  that is greater than or equal to $\tau$, denoted by $k_{\tau}$, always has $k_{\tau} < \frac{\epsilon}{\tau}$.
\end{property}

\begin{proof}
Consider the sum of all $r_i$ that are greater than or equal to $\tau$, we have:
\begin{equation}
\label{e.proof}
   \sum_{i=1}^n r_i \geq \sum_{r_i \geq \tau} r_i \geq k_{\tau}\tau \geq \frac{\epsilon}{\tau}\tau = \epsilon 
\end{equation}
This contradicts our given condition. Thus, $k_{\tau} < \frac{\epsilon}{\tau}$.

Since $\tau$ is chosen to be less than $\epsilon$ but not too small, $\frac{\epsilon}{\tau}$ is a small number. This completes the proof that there is a small number of $r_i$ that have big values.
\end{proof}

\subsection{Adaptive sampling algorithm}
Since there are only a few large residual points, according to the continuity of the solution, we can add new samples near these points. This can be achieved through local interpolation, which aligns with the idea of the SMOTE algorithm.

We describe the proposed method for data augmentation based on the above discussion. The process involves iterative steps to enhance the training dataset by focusing on instances with significant residuals.

\paragraph{Residual Classification}
We begin by categorizing the residual points into two classes:
a) Negative class: Points with small residuals;
b) Positive class: Points with large residuals.

This classification is based on the observation that regions with large residuals are typically localized, allowing us to focus our computational efforts on these specific areas. 
A direct way is to use a threshold to split the data. However, in practice, it is hard to determine the threshold.
Instead, we introduce a ratio and use the following simple and easy-to-implement method to split the residual points:

Let $\mathbf{x}_1, \mathbf{x}_2, ..., \mathbf{x}_n$ be the data points and $r_1, r_2, ..., r_n$ be the corresponding residuals. We first sort the residuals in descending order and then select the top $\lambda$×100\% ($\lambda \in (0, 0.5)$) of residuals as the positive samples. The remaining data points are selected as the negative samples. See Fig.\ref{f.ratio} for a further explanation.

\begin{figure}[!ht]
    \centering
    \subfigure[Compute the residuals]{\includegraphics[width=0.45\linewidth]{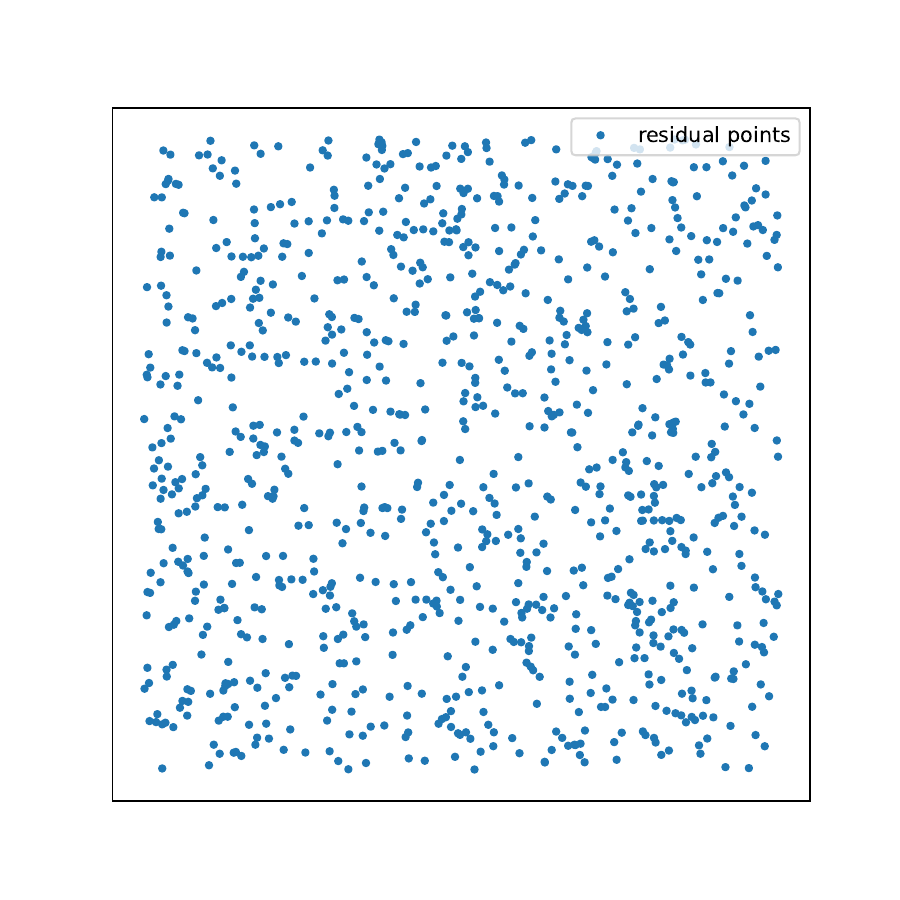}}
    \subfigure[Sort the residuals]{\includegraphics[width=0.45\linewidth]{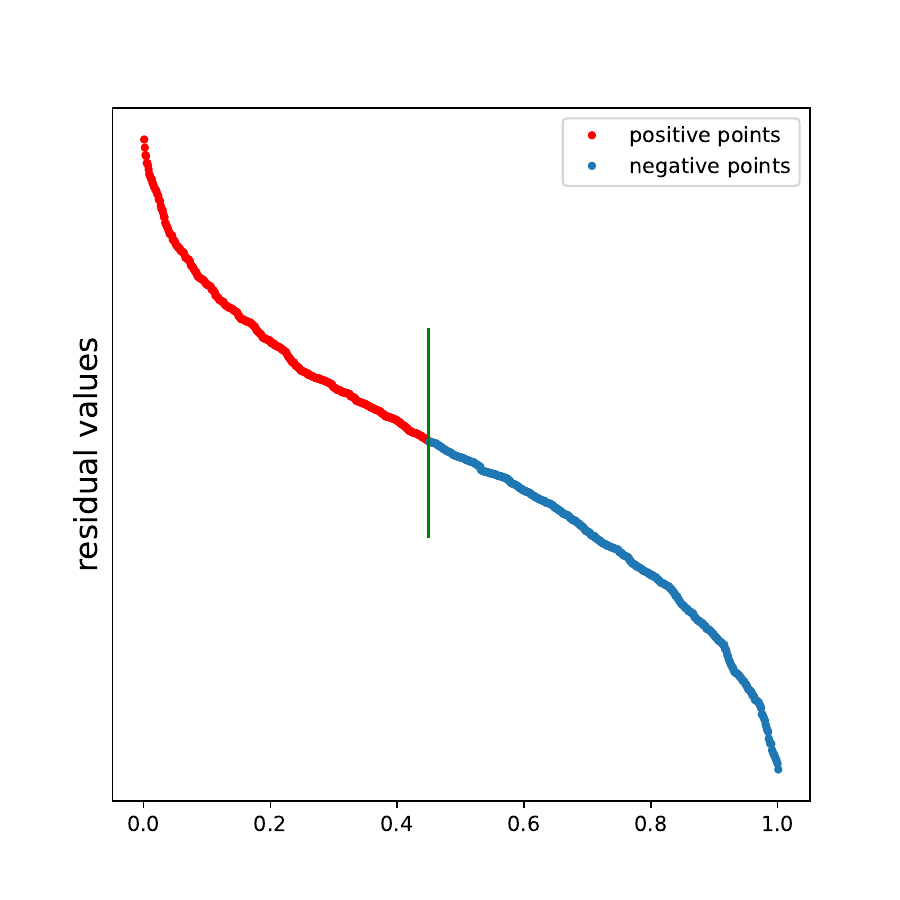}}
    \subfigure[Split the dataset]{\includegraphics[width=0.45\linewidth]{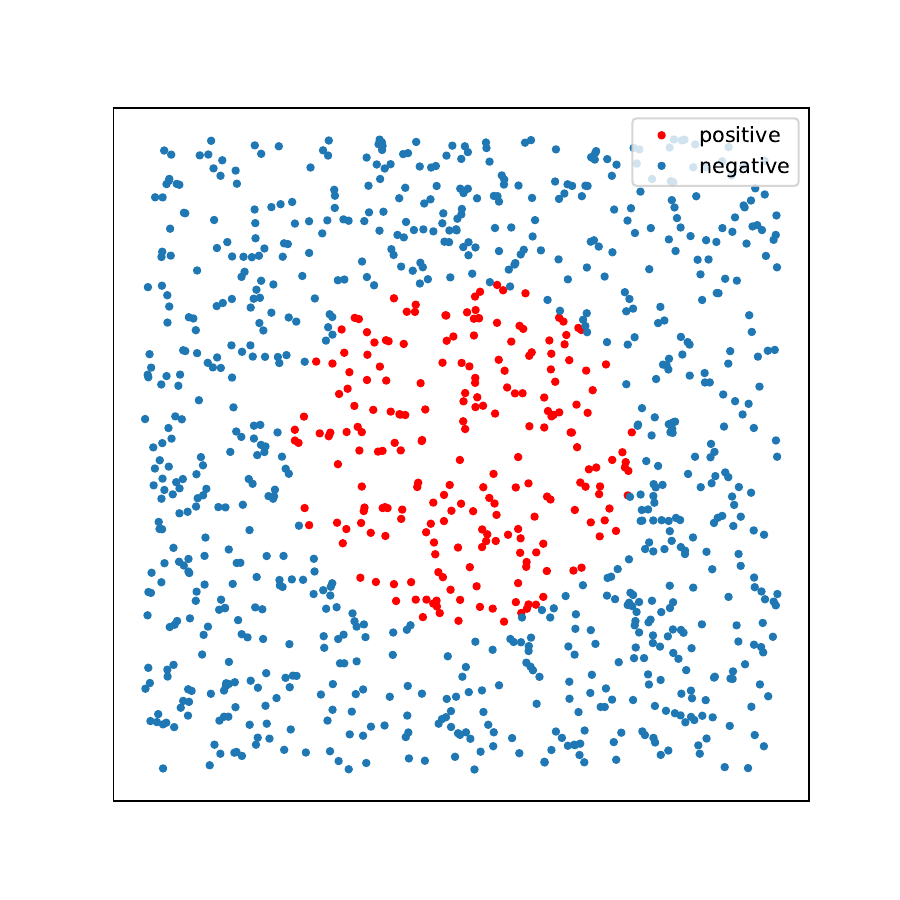}}
    \subfigure[Resample the data]{\includegraphics[width=0.45\linewidth]{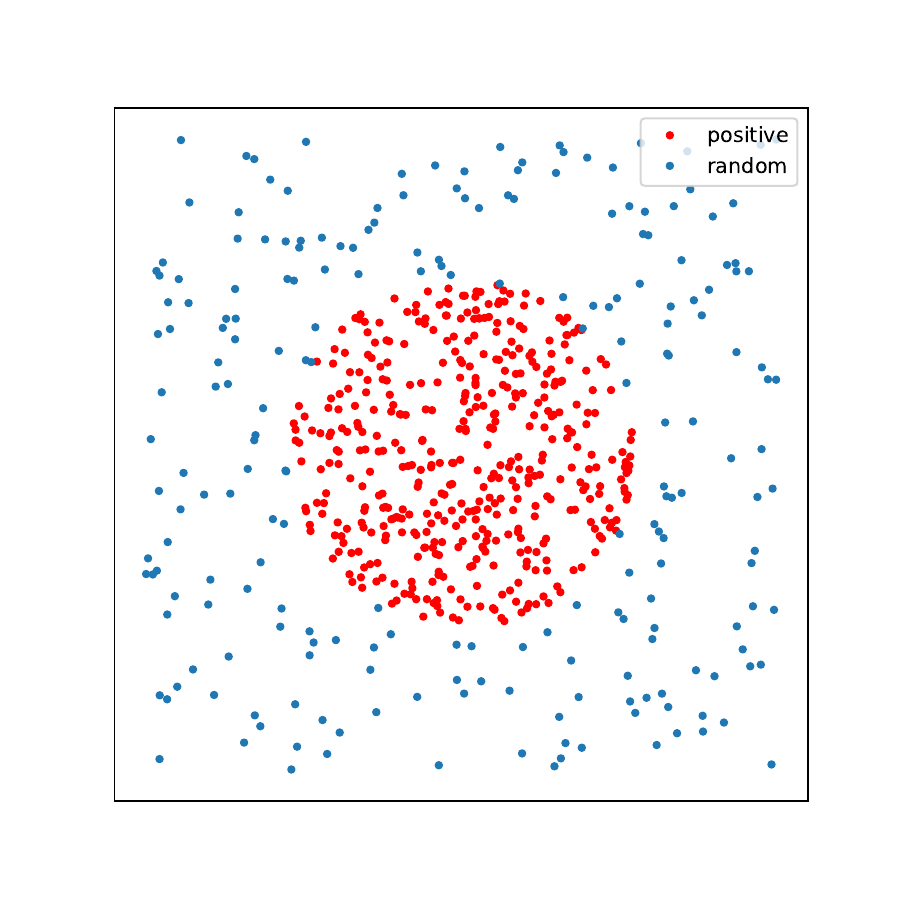}}
    \caption{\textcolor{black}{(a) Compute the absolute residuals. (b) Sort the residuals in descending order and use a ratio, $\lambda$, to split the dataset. The horizontal axis represents the ratio value, while the vertical axis represents the residual values. (c) Select the positive and negative samples based on the ratio. (d) Apply SMOTE to perform resampling and create a new training dataset.}}
    \label{f.ratio}
\end{figure}

Define the label $y_i$ of $\mathbf{x}_i$ as:
\begin{equation}
\label{e.clf}
y_i =
\begin{cases} 
      1, & \text{if }  r_i \text{ is in the top } \lambda \times 100\% \text{ of residuals}, \\
      0, & \text{otherwise}. \\
   \end{cases}
\end{equation}
and we get the training dataset $\mathcal{D} = \{\mathbf{x_i}, y_i\}_{i=1}^{n}$. Since ratio $\lambda$ is less than 0.5, the binary classification dataset $\mathcal{D}$ is imbalanced.

\paragraph{Over-sampling}
We then apply over-sampling techniques to generate new samples in the positive class ($y_i=1$). 
Here, we use the vanilla SMOTE algorithm \cite{chawla2002smote}, for more algorithms, we refer the readers to the survey \cite{kovacs2019empirical}. 

Suppose we have $n_1$ positive samples and $n_2 = n-n_1$ negative samples ($n_1 < n_2$). The SMOTE algorithms generate $n_2-n_1$ new samples and we finally have $n_2$ positive samples and $n_2$ negative samples, respectively.

\paragraph{Dataset Resampling}
We do not use all the $2n_2$ data to continue training. Instead, we only employ $n_2$ positive samples and sample $n_1$ random points in the space.
By doing so, we enable the network to focus more on points with large residuals while preventing the samples from overly concentrating on these points, thereby maintaining a balance in global errors. Moreover, since the number of training data does not change, this step can also help us save the GPU memory, ensuring a balance between accuracy and computational efficiency.

\paragraph{Algorithm}
The RSmote algorithm is implemented in Algorithm \ref{alg:smote}. \textcolor{black}{Here, the termination condition can be either a fixed number of iterations or the achievement of a specified threshold.}

\begin{algorithm}
\caption{Resample with SMOTE}\label{alg:smote}
\begin{algorithmic}[1]
\State Sample the initial training points;
\State Train the PINN for a certain number of iterations;
\While{Not meet the termination condition}
\State Calculate the absolute residuals (Fig.~\ref{f.ratio} (a));
\State Split the data according to the Eq.\eqref{e.clf} (Fig.~\ref{f.ratio} (b-c));
\State Perform SMOTE to oversample the positive class (Fig.~\ref{f.ratio} (d));
\State Generate additional $n_1$ random data points and combine them and $n_2$ positive samples to form a new dataset (Fig.~\ref{f.ratio} (d));
\State Train the PINN for a certain number of iterations;
\EndWhile
\end{algorithmic}
\end{algorithm}

\section{Theoretical analysis}
 \subsection{Error in Solving PDEs}
  
For simplicity, we first consider the Poisson equation for the theoretical part, i.e., we consider the following Poisson equation:
\begin{equation}
\begin{cases}
-\Delta u = f & \text{in } \Omega, \\
u = 0 & \text{on } \partial \Omega.
\end{cases}
\label{PDE}
\end{equation}

Using a neural network structure, the training loss function of PINNs \cite{raissi2019physics} is defined as:
\begin{equation}
\mathcal{R}_S(\Theta) := L_\text{F}(\Theta) + \textcolor{black}\gamma L_\text{B}(\Theta) = \frac{|\Omega|}{M}\sum_{i=1}^M |\Delta \phi(\mathbf{x}_i; \Theta) + f(\mathbf{x})|^2 + \frac{\textcolor{black}\gamma}{M} \sum_{i=1}^{\widehat{M}} |\phi(\mathbf{y}_i; \Theta)|^2.
\label{PINN}
\end{equation}

Here, $L_\text{F}(\Theta)$ represents the residual on the PDE equations, while $L_\text{B}(\Theta)$ represents the boundary/initial condition. The variables $\mathbf{x}_j \in \Omega$ and $\mathbf{y}_j \in \partial \Omega$ are independent and identically distributed (i.i.d.) uniform samples within each domain, and $\textcolor{black}\gamma$ is a constant used to balance the contributions from the domain and boundary terms. The domain $\Omega \subset [0,1]^d$ has a smooth boundary.

In the theoretical part, we consider $\textcolor{black}\gamma=0$, i.e., we study only the interior of the PDEs. This is because we observe that the residual error is primarily in the interior of the domain rather than at the boundary. What we aim to explain in this section is the advantage of sampling methods, therefore, we set $\textcolor{black}\gamma=0$. Consequently, we denote the continuous and discrete loss functions as:
\begin{align}
\Theta_D &:= \arg \inf_{\phi(\mathbf{x}; \Theta) \in \mathcal{F}} \mathcal{R}_D(\Theta) := \arg \inf_{\phi(\mathbf{x}; \Theta) \in \mathcal{F}} \int_{\Omega} |f(\mathbf{x}) + \Delta \phi(\mathbf{x}; \Theta)|^2 \,d \mathbf{x},
\label{thetaD}\\
\Theta_S &:= \arg \inf_{\phi(\mathbf{x}; \Theta) \in \mathcal{F}} \mathcal{R}_S(\Theta) := \arg \inf_{\phi(\mathbf{x}; \Theta) \in \mathcal{F}} \frac{|\Omega|}{M} \sum_{i=1}^M |f(\mathbf{x}_i) + \Delta \phi(\mathbf{x}_i; \Theta)|^2.
\label{thetaS}
\end{align}

The overall inference error is denoted as $\mathbb{E} \mathcal{R}_D(\Theta_S)$, which can be decomposed into two components:
\begin{align}
\mathbb{E} \mathcal{R}_D(\Theta_S) =& \mathcal{R}_D(\Theta_D) + \mathbb{E} \mathcal{R}_S(\Theta_D) - \mathcal{R}_D(\Theta_D) \notag\\&+ \mathbb{E} \mathcal{R}_S(\Theta_S) - \mathbb{E} \mathcal{R}_S(\Theta_D) + \mathbb{E} \mathcal{R}_D(\Theta_S) - \mathbb{E} \mathcal{R}_S(\Theta_S) \notag \\
\le& \underbrace{\mathcal{R}_D(\Theta_D)}_{\text{approximation error}} + \underbrace{\mathbb{E} \mathcal{R}_D(\Theta_S) - \mathbb{E} \mathcal{R}_S(\Theta_S)}_{\text{generalization error}}.
\end{align}
The last inequality arises from $\mathbb{E} \mathcal{R}_S(\Theta_S) \le \mathbb{E} \mathcal{R}_S(\Theta_D)$ due to the definition of $\Theta_S$ and the properties of integration. The expectation $\mathbb{E}$ accounts for the randomness in the sampling process.

\subsection{Neural network spaces and hypothesis spaces}
In this paper, we consider a resampling method where we first sample across the entire domain and then identify regions where the neural network is underperforming, followed by localized sampling. Consequently, we divide the domain $\Omega$ into two parts, $\Omega = \Omega_1 \cup \Omega_2$, where $\Omega_1 \cap \Omega_2 = \emptyset$, $|\Omega_1| = |\Omega_2|$, and both have $C^2$ smooth boundaries. In the first sampling process, the sample points located in $\Omega_1$ are denoted as $M_1$, while those in $\Omega_2$ are denoted as $M_2$, with $M = \frac{1}{2}M_1 = \frac{1}{2}M_2$.

This setup is consistent with our experimental observations, where we found that nearly half of the data points are learned well, while the other half require additional data to achieve satisfactory learning. Our method can be easily extended to more generalized cases where $|\Omega_1| \neq |\Omega_2|$, which would simply require introducing the ratio of the two regions in the analysis.

For the hypothesis space, we consider $\Omega_1$ as the region where learning is challenging and $\Omega_2$ as the region where learning is successful under the initial uniform sampling. The difference in learning difficulty between these two regions can be attributed to the following reasons. First, the regularity of the function in these two parts may differ; smoother regions are generally easier to learn, as seen in cases like the Burgers' equation or the Allen-Cahn equation when $\epsilon = 0$. Second, even if the regularity is the same, a significant difference in norms—such as the Sobolev norm or Hölder norm—between the two regions can also result in different learning difficulties, as in the Allen-Cahn equation for $\epsilon > 0$. Our method applies to both scenarios. Without loss of generality, we will consider the case where the regularity differs between the two regions in this section.

Furthermore, in this paper, the activation function we use is $\tanh(x)$. The second-order derivative we use in solving PDEs is \begin{equation}
\frac{d^2}{d x^2}(\tanh (x))=-2 \tanh (x) \operatorname{sech}^2(x),
\end{equation} which is a highly localized function. Therefore, we assume that the neural network is divided into two parts, $\phi_1$ and $\phi_2$, such that
\[
\phi(\mathbf{x}; \Theta) = \phi_1(\mathbf{x}; \Theta_1) + \phi_2(\mathbf{x}; \Theta_2),
\]
where $\phi_1$ is used to approximate the solution in $\Omega_1$ and is close to zero outside this region. Similarly, $\phi_2$ is used to approximate the solution in $\Omega_2$ and is close to zero outside this region. We then define $\mathcal{F}_1$ as the set of neural networks with support in $\Omega_1$, and $\mathcal{F}_2$ as the set of neural networks with support in $\Omega_2$.

The overall loss function can be divided into two parts:
For the continuous loss functions, they are expressed as
\begin{align}
\Theta_{D,1} &:= \arg \inf_{\phi(\mathbf{x}; \Theta_1) \in \mathcal{F}_1} \mathcal{R}_{D,1}(\Theta_1) := \arg \inf_{\phi(\mathbf{x}; \Theta) \in \mathcal{F}_1} \int_{\Omega_1} |f(\mathbf{x}) + \Delta \phi_1(\mathbf{x}; \Theta_1)|^2 \, d\mathbf{x},
\label{thetaD1}\\
\Theta_{D,2} &:= \arg \inf_{\phi_2(\mathbf{x}; \Theta_2) \in \mathcal{F}_2} \mathcal{R}_{D,2}(\Theta_2) := \arg \inf_{\phi(\mathbf{x}; \Theta) \in \mathcal{F}_2} \int_{\Omega_2} |f(\mathbf{x}) + \Delta \phi_2(\mathbf{x}; \Theta_2)|^2 \, d\mathbf{x}.
\label{thetaD2}
\end{align}
For the discrete loss functions, they are expressed as
\begin{align}
\Theta_{S,1} &:= \arg \inf_{\phi_1(\mathbf{x}; \Theta_1) \in \mathcal{F}_1} \mathcal{R}_{S,1}(\Theta_1) := \arg \inf_{\phi_1(\mathbf{x}; \Theta_1) \in \mathcal{F}_1} \frac{1}{2M_1} \sum_{i=1}^{M_1} |f(\mathbf{x}_i) + \Delta \phi_1(\mathbf{x}_i; \Theta_1)|^2,
\label{thetaS1}\\
\Theta_{S,2} &:= \arg \inf_{\phi_2(\mathbf{x}; \Theta_2) \in \mathcal{F}_2} \mathcal{R}_{S,2}(\Theta_2) := \arg \inf_{\phi_2(\mathbf{x}; \Theta_2) \in \mathcal{F}_2} \frac{1}{2M_2} \sum_{i=M_1+1}^M |f(\mathbf{x}_i) + \Delta \phi_2(\mathbf{x}_i; \Theta_2)|^2.
\label{thetaS2}
\end{align}

Similarly, we can divide the overall error into two parts: the approximation error and the generalization error. Specifically, for $i = 1, 2$, we have
\begin{equation}
    \mathbb{E} \mathcal{R}_{D,i}(\Theta_{S,i}) \le \mathcal{R}_{D,i}(\Theta_{D,i}) + \mathbb{E} \mathcal{R}_{D,i}(\Theta_{S,i}) - \mathbb{E} \mathcal{R}_{S,i}(\Theta_{S,i}).
\end{equation}

For the approximation error, due to the smoothness of each boundary, based on \cite{evans2022partial}, we know that
\[
\mathcal{R}_{D,i}(\Theta_{D,i}) \le \inf_{\phi_i(\mathbf{x}; \Theta_i)\in\mathcal{F}_i} \|\phi_i(\mathbf{x}; \Theta_i) - u_i(\mathbf{x})\|_{W^{2,\infty}(\Omega_i)}^2,
\]
where $u_*$ is the exact solution of Eq.~(\ref{PDE}), $u_1(\mathbf{x}) := u_*(\mathbf{x})$ for $\mathbf{x} \in \Omega_1$, and $u_2(\mathbf{x}) := u_*(\mathbf{x})$ for $\mathbf{x} \in \Omega_2$. The approximation error arises from the construction of the neural networks and depends on the richness and complexity of the neural network spaces. To describe the richness of the neural network spaces $\mathcal{F}_i$, we use a concept called the Vapnik-Chervonenkis (VC) dimension:
\begin{definition}[VC-dimension \cite{abu1989vapnik}]
	Let $\mathcal{H}$ denote a class of functions from $\mathcal{X}$ to $\{0,1\}$. For any non-negative integer $m$, define the growth function of $\mathcal{H}$ as \[\Pi_H(m):=\max_{x_1,x_2,\ldots,x_m\in \mathcal{X}}\left|\{\left(h(x_1),h(x_2),\ldots,h(x_m)\right): h\in \mathcal{H} \}\right|.\] The VC-dimension of $H$, denoted by $\text{VCdim}(\mathcal{H})$, is the largest $m$ such that $\Pi_H(m)=2^m$. For a class $\mathcal{G}$ of real-valued functions, define $\text{VCdim}(\mathcal{G}):=\text{VCdim}(\sgn(\mathcal{G}))$, where $\sgn(\mathcal{G}):=\{\sgn(f):f\in\mathcal{G}\}$ and $\sgn(x)=1[x>0]$.
\end{definition}
The following proposition establishes a link between the approximation ability and the VC dimension. Without loss of generality, we consider the domain to be $[0,1]^d$. For other domains, the difference will only introduce a constant coefficient.
\begin{proposition}\label{link}
 Given any $s, d \in \mathbb{N}^{+}$, there exists a (small) positive constant $C_{s, d}$ determined by $s$ and $d$ such that the following holds: For any $\varepsilon>0$ and a function set $\mathcal{H}$ with all elements defined on $[0,1]^d$, if $\operatorname{VCDim}(\mathcal{H}) \geq 1$ and
\begin{equation}\label{epsilon}
\inf _{\phi \in \mathcal{H}}\|\phi-f\|_{W^{2,\infty}\left([0,1]^d\right)} \leq \varepsilon \quad \text { for any } \|f\|_{ W^{s,\infty}\left([0,1]^d\right)}\le 1
\end{equation}
then $\operatorname{VCDim}(\mathcal{H}_2) \geq C_{s, d} \varepsilon^{-\frac{d} { s-2}} ,$ where \[\mathcal{H}_2:=\left\{p:=\frac{\partial^2 h}{\partial x_i\partial x_j}\mid h\in\mathcal{H}, 1\le i,j\le d\right\}.\]
\end{proposition}
\begin{proof}
    The proof can be divided into two parts. First, we construct a set of functions denoted as
\[
\left\{\lambda_\beta, \beta \in \mathcal{B} \mid \lambda_\beta = \frac{\partial^2 g}{\partial x_1^2},~\|g\|_{W^{s,\infty}\left([0,1]^d\right)} \le 1\right\},
\]
which can shatter $\mathcal{O}\left(\varepsilon^{-\frac{d}{s-2}}\right)$ points, where $\mathcal{B}$ will be defined later. Next, we will demonstrate that we can also find functions in $\mathcal{H}_2$ that can shatter $\mathcal{O}\left(\varepsilon^{-\frac{d}{s-2}}\right)$ points. Based on the definition of VC-dimension, this will complete the proof.

For the first step, fix $i=1,\ldots,d$, and there exists $\widetilde{g} \in C^{\infty}\left(0,1\right)^d$ such that $\frac{\partial^2{\widetilde{g}(\mathbf{0})}}{\partial x_1^2}=1$ and $\widetilde{g}(\boldsymbol{x})=0$ for $\|\boldsymbol{x}\|_2 \geq 1 / 3$. And we can find a constant $\bar{C}_{s,d}>0$ such that $g:=\widetilde{g} / \bar{C}_{s,d}$ with $\|g\|_{W^{s,\infty}\left([0,1]^d\right)} \le 1$.

For any $\varepsilon$, denote $M=\mathcal{O}\left(\varepsilon^{-\frac{1}{s-2}}\right)$ as an integer determined later. Divide $[0,1]^d$ into $M^d$ non-overlapping sub-cubes $\left\{Q_{\boldsymbol{\theta}}\right\}_\theta$ as follows:
				$$
				Q_{\boldsymbol{\theta}}:=\left\{\boldsymbol{x}=\left[x_1, x_2, \cdots, x_d\right]^T \in[0,1]^d\mid x_i \in\left[\frac{\theta_i-1}{M}, \frac{\theta_i}{M}\right], i=1,2, \cdots, d\right\},
				$$
				for any index vector $\boldsymbol{\theta}=\left[\theta_1, \theta_2, \cdots, \theta_d\right]^T \in\{1,2, \cdots, M\}^d$. Denote the center of $Q_{\boldsymbol{\theta}}$ by $\boldsymbol{x}_{\boldsymbol{\theta}}$ for all $\boldsymbol{\theta} \in\{1,2, \cdots, M\}^d$. Define
				$$
				\mathcal{B}:=\left\{\beta: \beta \text { is a map from }\{1,2, \cdots, M\}^d \text { to }\{-1,1\}\right\} .
				$$
				For each $\beta \in \mathcal{B}$, we define, for any $\boldsymbol{x} \in \mathbb{R}^d$, $$
				h_\beta(\boldsymbol{x}):=\sum_{\boldsymbol{\theta} \in\{1,2, \cdots, M\}^d} M^{-s} \beta(\boldsymbol{\theta}) g_{\boldsymbol{\theta}}(\boldsymbol{x}), \quad \text { where } g_{\boldsymbol{\theta}}(\boldsymbol{x})=g\left(M \cdot\left(\boldsymbol{x}-\boldsymbol{x}_{\boldsymbol{\theta}}\right)\right) \text {. }
				$$ Due to $|{\rm supp}~\widetilde{g}(\boldsymbol{x})|\le \frac{2}{3}$ and $|D^{\boldsymbol{\alpha}}	h_\beta(\boldsymbol{x})|\le M^{-s+|\boldsymbol{\alpha}|}\|g\|_{W^{n,\infty}}\le 1$, we obtain that\[|D^{\boldsymbol{\alpha}}h_\beta(\boldsymbol{x})|\le 1\] for any $|\boldsymbol{\alpha}|\le n$
				Therefore, $\|h_{\beta}\|_{W^{s,\infty}([0,1]^d)}\le 1$. And it is easy to check $\{ \frac{\partial^2{{h_\beta}}}{\partial x_1^2}=\lambda_\beta\mid \beta\in\mathcal{B}\}$ can scatters $M^d$ points since $\frac{\partial^2{\widetilde{g}(\mathbf{0})}}{\partial x_1^2}=1$ and $\widetilde{g}(\boldsymbol{x})=0$ for $\|\boldsymbol{x}\|_2 \geq 1 / 3$.

    For the second step, due to (\ref{epsilon}), we have that there exists $p_\beta \in \mathcal{H}_2$ such that 
\begin{equation}
    \|p_\beta - \lambda_\beta\|_{L^\infty([0,1]^d)} < \frac{3}{2}\varepsilon.
\end{equation}
For simplicity in notation, we set 
\begin{equation}
    |p_\beta(\boldsymbol{x}_\beta) - \lambda_\beta(\boldsymbol{x}_\beta)| < \frac{3}{2}\varepsilon,
\end{equation}
otherwise, we simply remove a zero-measure set. Therefore, we have 
\begin{align}
|\lambda_\beta(\boldsymbol{x}_\theta)| \ge \frac{M^{-s+2}}{\bar{C}_{s,d}} \ge 2\varepsilon \ge |p_\beta(\boldsymbol{x}_\theta) - \lambda_\beta(\boldsymbol{x}_\theta)|,
\end{align}
where the penultimate inequality is obtained by setting $M = \lfloor(2\bar{C}_{s,d}\varepsilon)^{-\frac{1}{s-2}}\rfloor$. Therefore, the set $\left\{ p_\beta \mid \beta \in \mathcal{B} \right\}\in\mathcal{H}_2$ can shatter $M^d$ points since $p_\beta(\boldsymbol{x}_\theta)$ and $\lambda_\beta(\boldsymbol{x}_\theta)$ have the same sign at $M^d$ points. Therefore, we have $\operatorname{VCDim}(\mathcal{H}_2) \geq C_{s, d} \varepsilon^{-\frac{d} { s-2}} ,$ where $C_{s, d} =(2\bar{C}_{s,d})^{-\frac{d}{s-2}}-1$.

\end{proof}
The above proposition shows that the VC-dimension can represent the approximation capability of neural networks in the $W^{2,\infty}$-norm. However, it is important to note that the corresponding result for the $W^{2,p}$-norm remains an open question. In \cite{siegel2022optimal}, the authors prove the case for the $L^{p}$-norm in the context of translation-invariant classes of functions. For more general cases, this question is still unresolved.

This proposition provides the optimal approximation rate for neural networks, and achieving this rate has been demonstrated in various works, including those on smooth functions \cite{lu2021deep}, Sobolev spaces \cite{yang2023nearly,yang2023nearlys}, and Korobov spaces \cite{yang2024near}. In this paper, we focus on the sampling process, and thus, we assume that our network can achieve this optimal approximation rate.

\begin{assumption}\label{assump}
    Set $u_*$ is the exact solution of Eq.~(\ref{PDE}), $u_1(\mathbf{x}) := u_*(\mathbf{x})$ for $\mathbf{x} \in \Omega_1$ belongs to $W^{s_1, \infty}(\Omega_1)$, and $u_2(\mathbf{x}) := u_*(\mathbf{x})$ for $\mathbf{x} \in \Omega_2$ belongs to $W^{s_2, \infty}(\Omega_2)$, and $2<s_1 < s_2$. For any $\varepsilon>0$, we can find $\mathcal{F}_1$ and $\mathcal{F}_2$, and constant $C_*$ such that
\[
\mathcal{R}_{D,i}(\Theta_{D,i}) \le\varepsilon^2,
\] for $i=1,2$ with $\operatorname{VCDim}(\mathcal{F}_{1,2})=C_*\varepsilon^{-\frac{d}{s_1-2}}$ and $\operatorname{VCDim}(\mathcal{F}_{2,2})=C_*\varepsilon^{-\frac{d}{s_2-2}},$ where \[\mathcal{F}_{i,2}:=\left\{p:=\frac{\partial^2 h}{\partial x_i\partial x_j}\mid h\in\mathcal{F}_i, 1\le i,j\le d\right\}.\]
\end{assumption}
Based on the above proof in Proposition \ref{link}, we understand that if the norm of the target function $\|f\|_{W^{s,\infty}\left([0,1]^d\right)}$ is large, it will cause $\bar{C}_{s,d}$ to be small, which in turn makes $C_{s,d}$ large. Therefore, even when the regularity is the same, if the norms are significantly different, this will also affect the lower bound of the VC dimension. Our subsequent analysis remains valid for this scenario as well.

\subsection{Sampling}
In this section, we will explain why the training difficulties differ between $\Omega_1$ and $\Omega_2$ and why local sampling is effective in addressing these challenges.

For the generalization error, the pseudo-dimension \cite{pollard1990empirical} or uniform covering numbers (which can be used to bound the generalization error) are often employed, as discussed in \cite{yang2023nearly,yang2023nearlys,schmidt2020nonparametric,jiao2021error,anthony1999neural,yang2024deeper}. In most cases, the pseudo-dimension differs from the covering number. However, in \cite{anthony1999neural} and \cite{bartlett2019nearly}, it has been shown that they are in the same order with respect to the width and depth of neural networks, particularly in deep neural networks. Based on these facts and the results in \cite{yang2023nearlys}, the following proposition is presented:

\begin{proposition}\label{yanggen}
    For fully connected deep neural network sets $\mathcal{F}_1$ and $\mathcal{F}_2$, for any $M_i > 0$, there exists a constant $c$ such that
    \begin{align}
        \mathbb{E} \mathcal{R}_{D,i}(\Theta_{S,i}) - \mathbb{E} \mathcal{R}_{S,i}(\Theta_{S,i}) \le \frac{ \sqrt{c\cdot\operatorname{VCDim}(\mathcal{F}_{i,2})}}{\sqrt{M_i}},\label{generror}
    \end{align}
    where the result is valid up to a logarithmic term with respect to $\operatorname{VCDim}(\mathcal{F}_{i,2})$ and $M_i$.
\end{proposition}

In \cite{gyorfi2002distribution} and \cite{yang2024deeper}, smaller generalization rates with respect to the number of sample points are obtained, but these results come at the cost of a worse order with respect to the VC-dimension. These results are not used here for two reasons. First, both types of errors have the same ratio of order between the VC dimension and the number of sample points, which does not affect our subsequent analysis. Second, the error in those theorems is not straightforward. To achieve a better order for the sampling error, the approximation error must be used to correct it, which complicates the right-hand side of (\ref{generror}) and makes the analysis more difficult.

\begin{theorem}
    For fully connected deep neural network sets $\mathcal{F}_1$ and $\mathcal{F}_2$, suppose Assumption \ref{assump} holds. Then, for any $\varepsilon > 0$, if $\mathcal{R}_{D,i}(\Theta_{D,i}) \le \varepsilon^2$, we have:
\begin{align}
    &\mathbb{E} \mathcal{R}_{D,1}(\Theta_{S,1}) - \mathbb{E} \mathcal{R}_{S,1}(\Theta_{S,1}) \le \varepsilon^2 \quad \text{when} \quad M_1 \ge \eta \bar{M}, \notag\\ 
    &\mathbb{E} \mathcal{R}_{D,2}(\Theta_{S,2}) - \mathbb{E} \mathcal{R}_{S,2}(\Theta_{S,2}) \le \varepsilon^2 \quad \text{when} \quad M_2 \ge \bar{M},\label{result}
\end{align}
where $\bar{M} = cC_*\varepsilon^{-\frac{d}{s_2-2}-4}$ and $\eta = \varepsilon^{\frac{d(s_1-s_2)}{(s_1-2)(s_2-2)}} > 1$, with $c$ and $C_*$ defined in Proposition \ref{yanggen} and Assumption \ref{link}.
\end{theorem}

The proof can be directly obtained from Proposition \ref{yanggen} and Assumption \ref{link}. Based on this theorem, we understand that when the regularity differs between regions, the less smooth part requires significantly more sample points to reduce the generalization error, especially when the dimensionality $d$ is large and the required error $\varepsilon$ is small. This explains why, after performing uniform sampling, the smoother part is learned well while the less smooth part is far from being adequately learned.

To address this issue, more sample points are needed in $\Omega_1$, the less smooth region, to increase the number of samples to $\eta \bar{M}$. In this step, sampling should be done locally rather than globally, as $\eta$ is typically large. If we were to sample globally, we would end up sampling an additional $\eta \bar{M}$ points in $\Omega_2$, which is unnecessary since this region has already been learned well. \textcolor{black}{In other words, this theorem shows that without applying the RSmote technique, achieving the same sampling error requires a significantly larger number of sampling points. Alternatively, if the total number of sampling points is fixed, the RSmote technique enables us to achieve a smaller sampling error. }

\textcolor{black}{In the analysis, we choose the residual division ratio $\lambda$ to be $\frac{1}{2}$. Our analysis can be generalized to any $\lambda$, and the results in Eqs.~(\ref{result}) will be $M_1 \leq 2\eta\lambda\bar{M}$ and $M_1 \leq 2(1-\lambda)\bar{M}$. Therefore, when $\lambda$ is smaller, fewer sampling points are required in the low-regularity region. However, $\lambda$ cannot be too small in the RSmote calculation. The reason is that during resampling, we need to sample more points to match the number of points in the small residual part and the large residual parts. For very small values of $\lambda$, this requires sampling many points in a small region, which can be challenging and may even lead to training failure. Furthermore, in many cases, the solutions of PDEs exhibit low-dimensional structures in small localized regions, such as the solutions of the Allen-Cahn equation, Burgers’ equation, and Green’s functions \cite{hao2024multiscale}. However, due to the continuum of neural networks, they struggle to accurately approximate these structures in a region larger than the exact low-regularity regions. To mitigate this issue, a larger region is needed to effectively cover the low-regularity parts, which necessitates selecting a larger $\lambda$. Thus, in the experiments, we choose $\lambda = 0.45$, which is close to $\frac{1}{2}$. More simulation results can be found in Section 5.7.  
}

\section{Experiments}
\label{s:exps}

\subsection{Setup}
\textcolor{black}{We make comparisons with the method Residual-based adaptive distribution (RAD) proposed in \cite{wu2023comprehensive} and use the term RAD-$N$ to denote the RAD method that uses extra $N$ points to estimate the distribution.} 
We test our method on four PDE examples: the first two are low-dimensional and the rest are high-dimensional.
We use two metrics to evaluate the performance, including $L_2$ relative error, which is defined as $\frac{\|\hat{u}-u\|_2}{\|u\|_2}$, and GPU memory (MB) requirements for each method at different data sizes.

We train all neural network models with Adam optimizer and L-BFGS optimizer with a learning rate 0.001.
\textcolor{black}{In each sampling iteration, we first use the Adam optimizer for 1000 steps and then switch to the L-BFGS optimizer for another 1000 steps to train the model. Each experiment is run 5 times with different random seeds, and the mean and standard deviation of the errors are computed. All methods are trained for 100 sampling iterations, which also serves as the termination condition.}
In all examples, the hyperbolic tangent (tanh) is selected as the activation function. The ratio $\lambda$ is set to 0.45.
Table \ref{T.param} summarizes the network architecture used for each example.
All the experiments are conducted on a workstation equipped with an NVIDIA-3090 GPU.

\begin{table}[!ht]
\centering
\caption{Summary of the experimental setup.}
\begin{tabular}{lccc}
\toprule[2pt]
 Problems & Depth & Width & Dimension (d) \\
\midrule[1pt]

 \textcolor{black}{Laplace Equation} & 3 & 20 & 2 \\
 \textcolor{black}{Burgers' Equation} & 3 & 64 & 2 \\
 Allen-Cahn Equation & 3 & 64 & 2 \\
 Elliptic Equation & 3 & 2d & 10, 20, 100 \\
 Reaction-Diffusion Equation & 3 & 2d & 20, 30, 50 \\
 \bottomrule[2pt]
\end{tabular}
\label{T.param}
\end{table}

\subsection{\textcolor{black}{Laplace equation}}

\textcolor{black}{The Laplace equation is defined as:
\begin{equation}
\label{e.laplace}
\begin{gathered}
r\frac{dy}{dr} + r^2\frac{dy^2}{dr^2} + \frac{dy^2}{d\theta^2} = 0, \\
y(1,\theta) = \cos(\theta), \\
y(r, \theta +2\pi) = y(r, \theta),  \\
r \in [0, 1],  \theta \in [0, 2\pi].
\end{gathered} 
\end{equation}
The reference solution is $y=r\cos(\theta).$
}

The experimental results presented in Table \ref{T.laplace} showcase the performance of three different methods, namely RAD-50000, RAD-100000, and RSmote, under varying amounts of training data (2000, 5000, and 10000 samples). \textcolor{black}{Fig.~\ref{f.laplace} shows the loss curves of different methods with three different amounts of training data, and Fig.~\ref{f.laplace_field} presents the field and absolute difference maps for the best results.
}

\textcolor{black}{
In general, increased sampling tends to reduce errors. RAD performs best with smaller datasets, achieving the lowest error with 2000 training points. However, this advantage diminishes as the dataset size grows. For 5000 and 10000 points, RSmote outperforms the other methods in terms of accuracy.
Memory usage remains consistent across different dataset sizes, with RSmote using the least memory, followed by RAD-50000, and RAD-100000 using the most.
The loss curves of RSmote show more fluctuating compared to those of RAD method. This is because the solution $y=r\cos(\theta)$ is smooth, and the error distribution is not concentrated (see Fig.~\ref{f.laplace_field} (d)(e)). As a result, RSmote may misjudge which regions need refinement, leading to fluctuations in errors during the data updating process. While for the fields and absolute differences, the RSmote also achieves lower difference.
As dataset size increases, all methods show improved accuracy. However, RSmote's relative advantage becomes more pronounced, while RAD methods benefit more from smaller datasets.
}

\begin{table}[!ht]
\renewcommand\arraystretch{1.1}
\centering\caption{Performance \textcolor{black}{(Mean $\pm$ Std.)} comparisons for different training data sizes on the Laplace equation using RAD-50000, RAD-100000, and RSmote methods. The table presents the evaluation scores (lower is better) and memory consumption (in MB) for models trained with 2000, 5000, and 10000 data points (\#Sampling). The \textbf{bold} indicates the best result.}
\begin{adjustbox}{width=\textwidth}
\begin{tabular}{l|cc|cc|cc}
\toprule[2pt]
Training data & \multicolumn{2}{c|}{\#Sampling=2000} & \multicolumn{2}{c|}{\#Sampling=5000} & \multicolumn{2}{c}{\#Sampling=10000} \\
Evaluation &   Score   &  Memory  & Score   &  Memory  & Score   &  Memory \\
\midrule[1pt]
  RAD-50000 & \textcolor{black}{6.59e-5 $\pm$ 2.63e-5}   &   1284       &  \textcolor{black}{6.52e-5 $\pm$ 1.51e-5}  &   1292       &  \textcolor{black}{6.44e-5 $\pm$ 8.73e-5}  &   1320       \\
  
  RAD-100000 & \textcolor{black}{\textbf{5.83e-5 $\pm$ 2.84e-5}}  & 1558   & \textcolor{black}{5.30e-5 $\pm$ 2.19e-5}  &   1562  &  \textcolor{black}{4.56e-5 $\pm$ 1.22e-5}  &    1580      \\
  
\midrule[1pt]
  RSmote&  \textcolor{black}{5.98e-5 $\pm$ 3.15e-5}  &    \textbf{1100}      &  \textcolor{black}{\textbf{4.92e-5 $\pm$ 1.67e-5}}  &   \textbf{1112}  & \textcolor{black}{\textbf{4.16e-5 $\pm$ 1.83e-5}} &  \textbf{1140}  \\ 
  
\bottomrule[2pt]
\end{tabular}
\end{adjustbox}
\label{T.laplace}
\end{table}

\begin{figure}[!h]
    \centering
    \subfigure[\#Sampling=2000]{\includegraphics[width=0.45\linewidth]{ 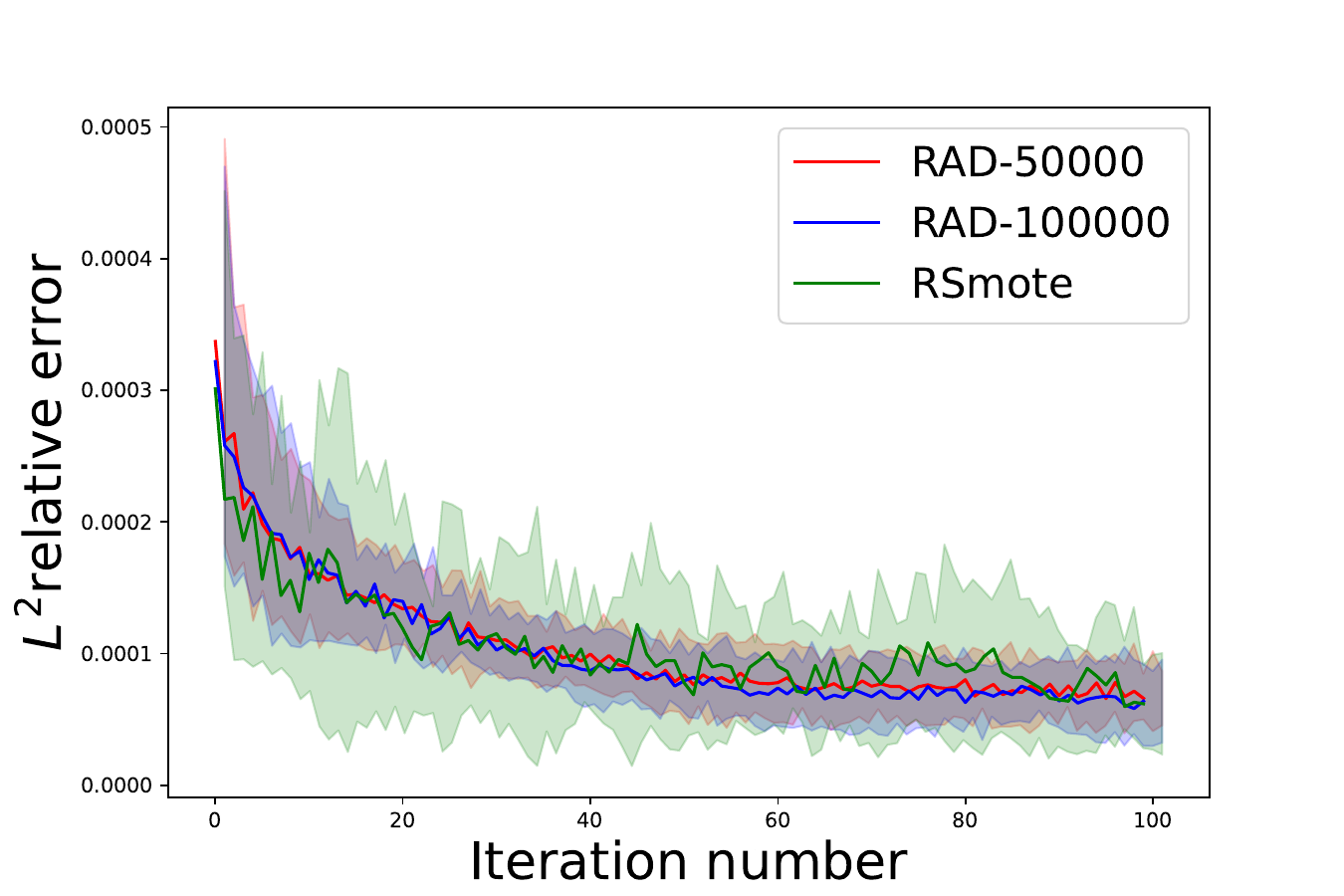}}
    \subfigure[\#Sampling=5000]{\includegraphics[width=0.45\linewidth]{ 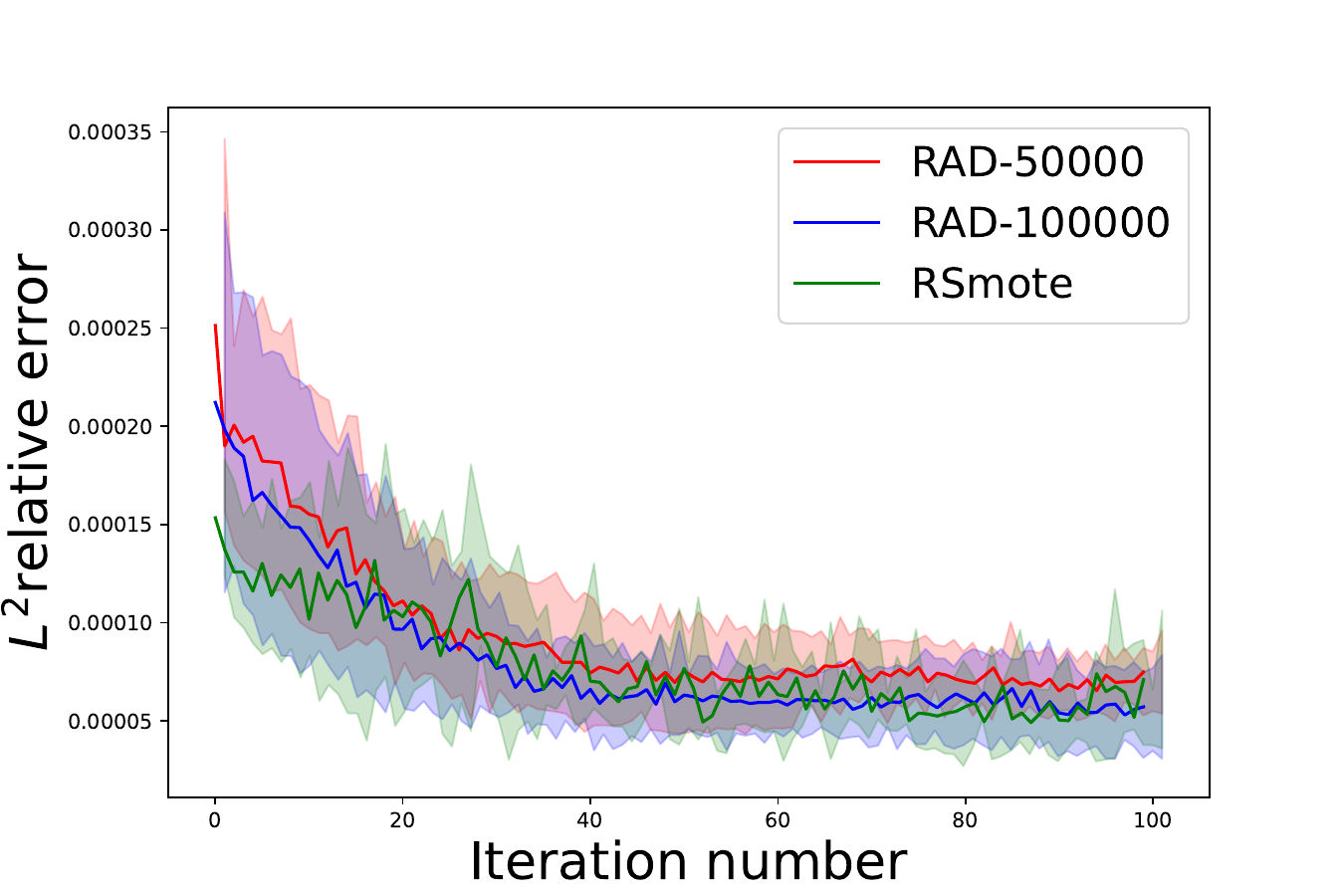}}
    \subfigure[\#Sampling=10000]{\includegraphics[width=0.45\linewidth]{ 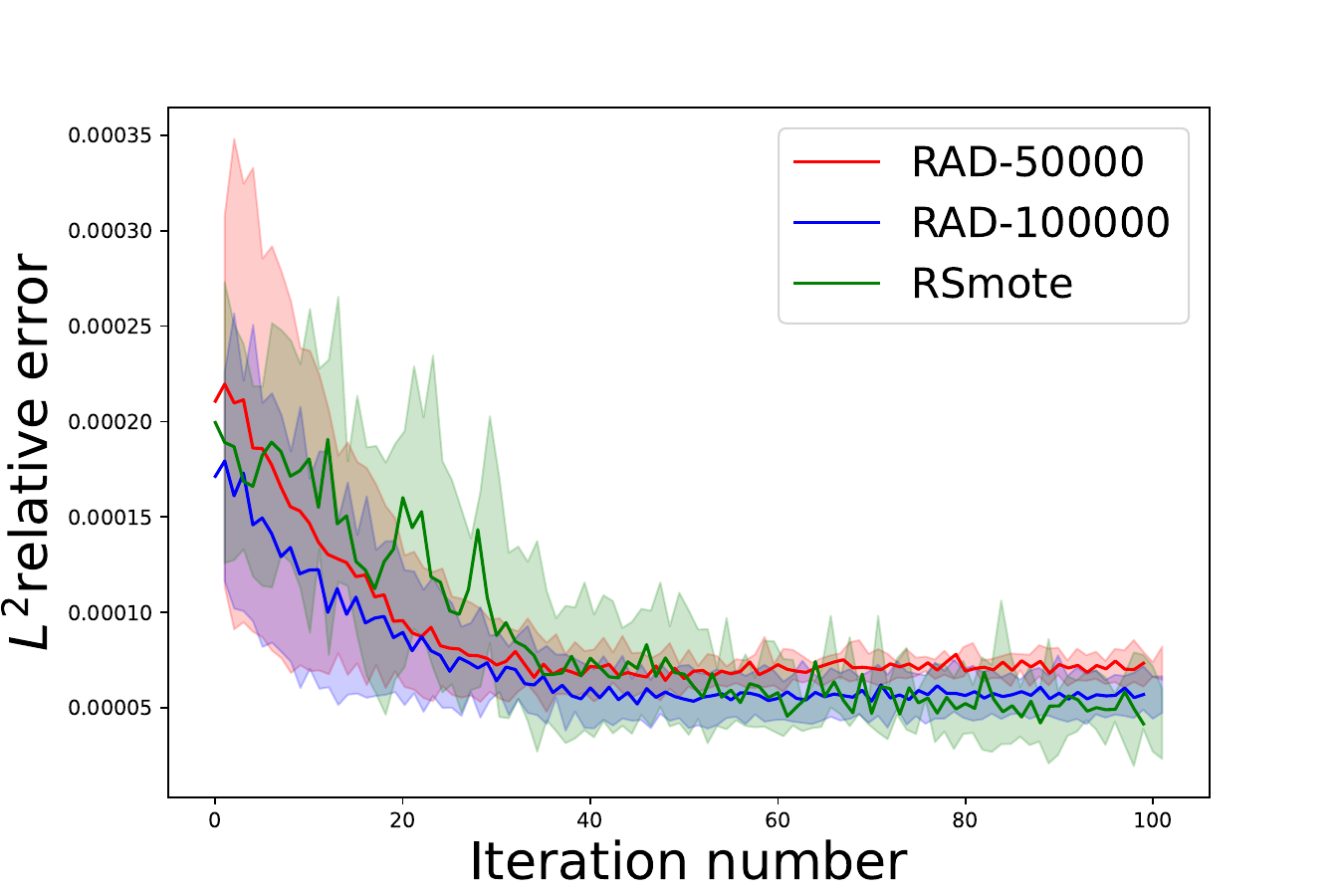}}
    \caption{\textcolor{black}{\textbf{Loss curves for Laplace Equation.} Red line: Mean values of RAD-50000 method;  Blue line: Mean values of RAD-100000 method; Green line: Mean values of RSmote method. The shaded areas represent the corresponding standard deviations.}}
    \label{f.laplace}
\end{figure}

\begin{figure}[!h]
    \centering
    \subfigure[Exact solution]{\includegraphics[width=0.3\linewidth]{ 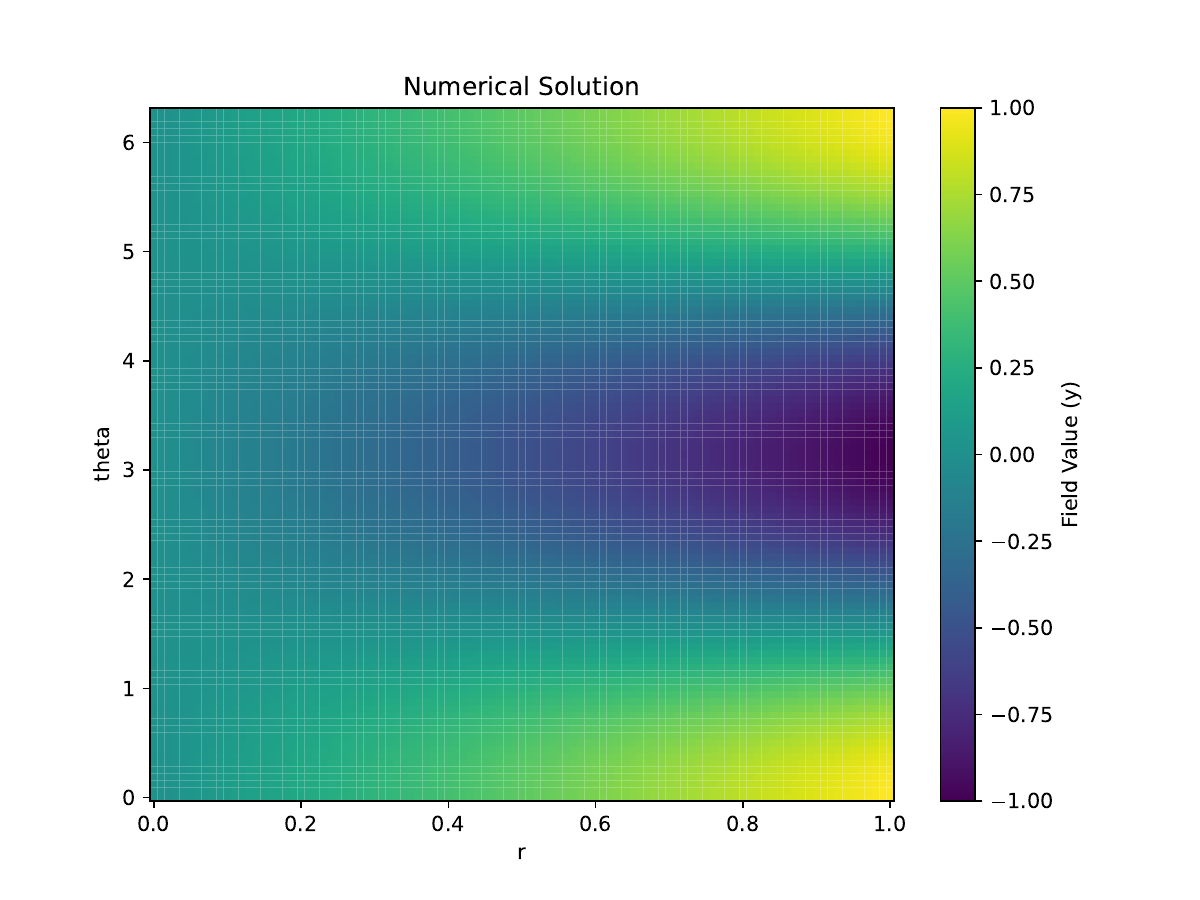}}
    \subfigure[RAD solution]
    {\includegraphics[width=0.3\linewidth]{ 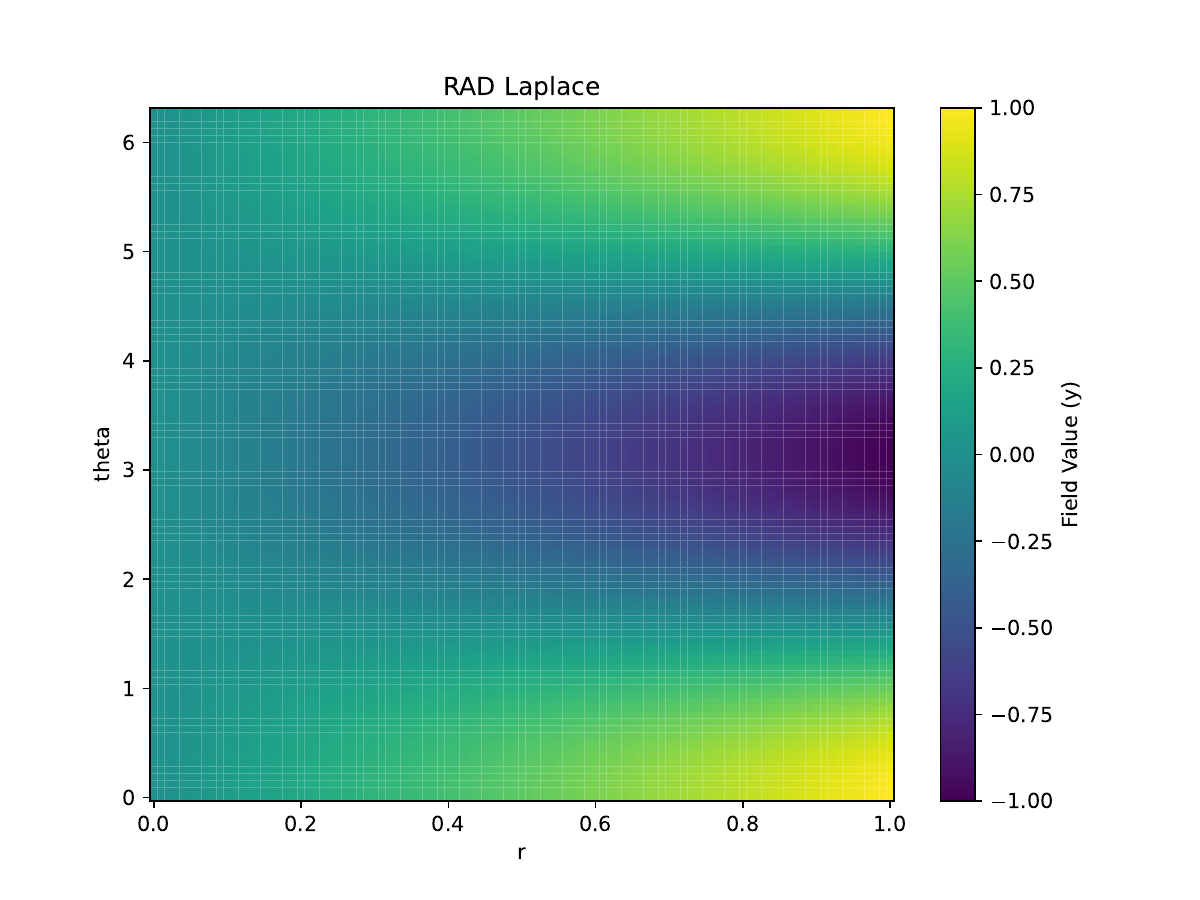}}
    \subfigure[RSmote solution]{\includegraphics[width=0.3\linewidth]{ 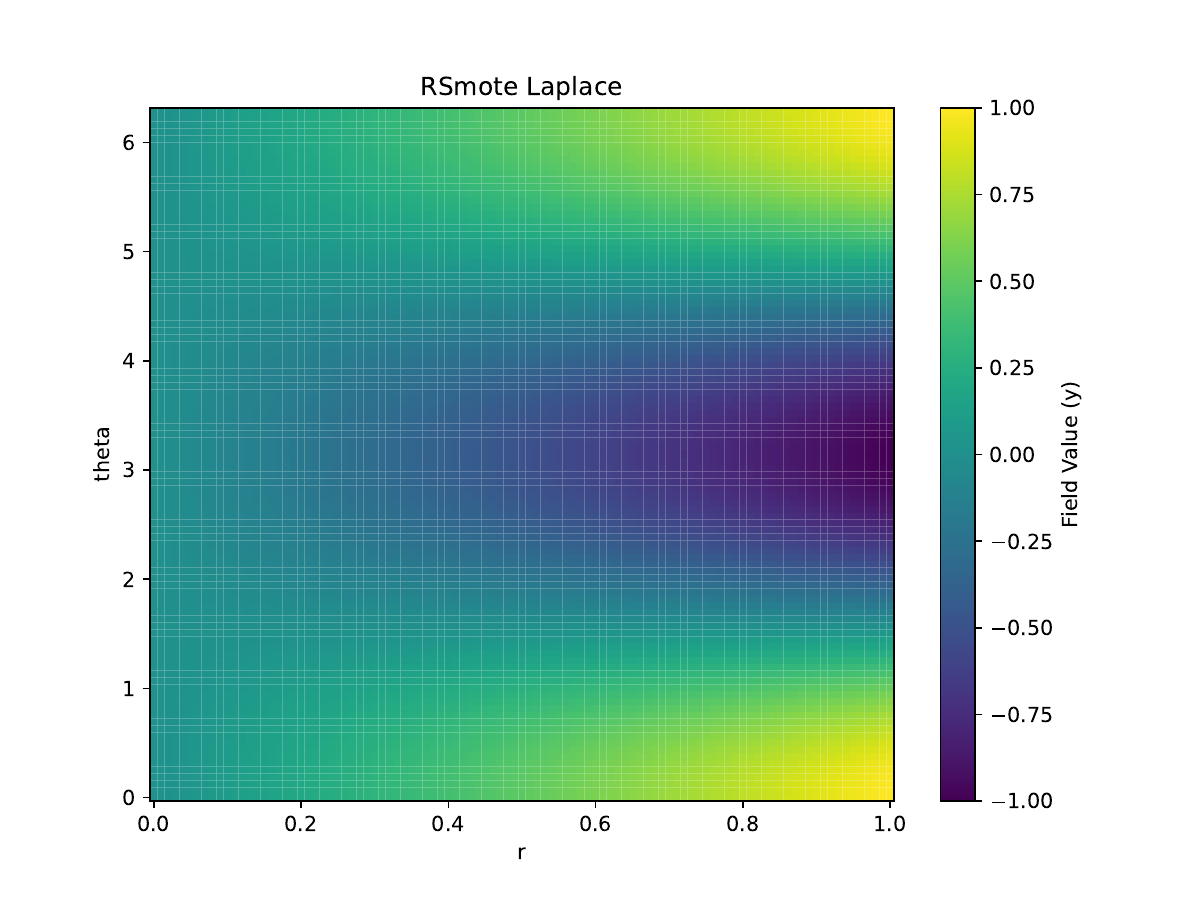}}
    \subfigure[RAD difference]{\includegraphics[width=0.3\linewidth]{ 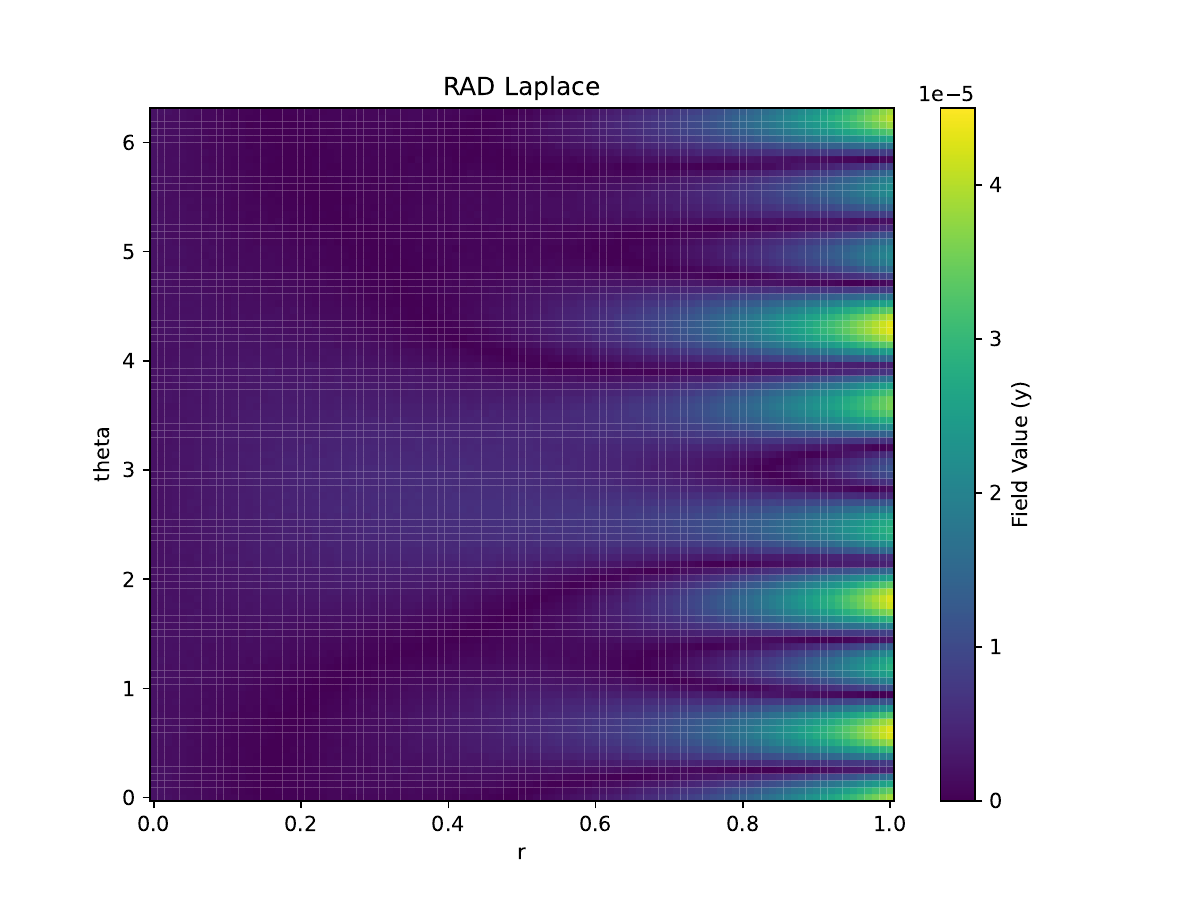}}
    \subfigure[RSmote difference]{\includegraphics[width=0.3\linewidth]{ 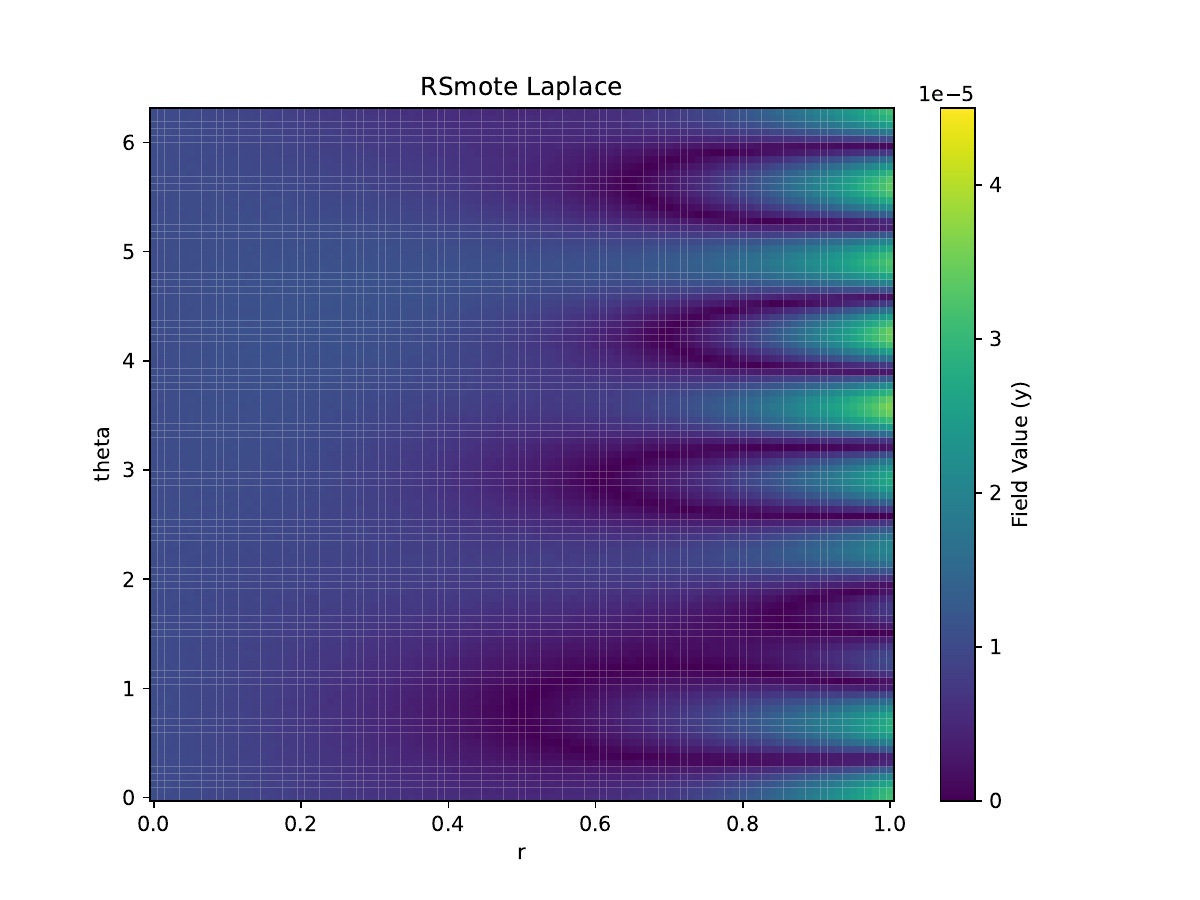}}
    \caption{\textcolor{black}{\textbf{Solution fields for Laplace Equation. (a)-(c): ground truth, RAD solution and RSmote solution; (d)-(e): Absolute differences. }}}
    \label{f.laplace_field}
\end{figure}

\subsection{\textcolor{black}{Burgers' equation}}
\textcolor{black}{
The Burgers’ equation is defined as:
\begin{equation}
\label{e.burges}
\begin{gathered}
u_t + u u_x = \frac{1}{100\pi}u_{xx}, \\
u(x, 0) = -\sin(\pi x), \\
u(-1, t) = u(1, t) = 0,  \\
x\in[-1, 1], t\in[0,1]. 
\end{gathered} 
\end{equation}
}

\textcolor{black}{
The experimental results presented in Table \ref{T.burges} highlight the performance of three methods, namely RAD-50000, RAD-100000, and RSmote, under varying amounts of training data (2000, 5000, and 10000 samples). Fig.~\ref{f.burges} illustrates the loss curves for different methods across the three data sizes, while Fig.~\ref{f.burgers_field} shows the field and absolute difference maps for the best results.}

\textcolor{black}{Overall, increased sampling reduces errors.
RSmote demonstrates superior performance across all datasets, although its improvement is modest for 2000 and 10000 data points. RAD-50000 shows limited benefit from increased data size, while the other two methods exhibit more significant improvements.
Memory usage remains consistent across dataset sizes, with RSmote being the most efficient, followed by RAD-50000, and RAD-100000 consuming the most. In the field maps, RSmote also achieves smaller absolute differences.
The loss curves show similar convergence patterns, with RSmote exhibiting greater fluctuations at 2000 points, RAD displaying more at 5000 points, and both methods experiencing some variability at 10000 points. These fluctuations will be further discussed in \ref{ss.fluctuation}.
}

\begin{table}[!h]
\renewcommand\arraystretch{1.1}
\centering\caption{\textcolor{black}{Performance (Mean $\pm$ Std.) comparisons for different training data sizes on the Burgers' equation using RAD-50000, RAD-100000, and RSmote methods. The table presents the evaluation scores and memory consumption (in MB) for models trained with 2000, 5000, and 10000 data points. The lower the score, the better the performance. The \textbf{bold} indicates the best result.}}
\begin{adjustbox}{width=\textwidth}
\begin{tabular}{l|cc|cc|cc}
\toprule[2pt]
Training data & \multicolumn{2}{c|}{\#Sampling=2000} & \multicolumn{2}{c|}{\#Sampling=5000} & \multicolumn{2}{c}{\#Sampling=10000} \\
Evaluation &   Score   &  Memory  & Score   &  Memory  & Score   &  Memory \\
\midrule[1pt]
  RAD-50000 & \textcolor{black}{8.32e-4 $\pm$ 2.78e-4}   &   \textcolor{black}{1518}       &  \textcolor{black}{5.48e-4 $\pm$ 2.77e-4}  &  \textcolor{black}{1522}        &  \textcolor{black}{4.07e-4 $\pm$ 1.10e-4}  &  \textcolor{black}{1552}       \\
  
  RAD-100000 & \textcolor{black}{7.86e-4 $\pm$ 2.84e-4}  & \textcolor{black}{1882}  & \textcolor{black}{5.01e-4 $\pm$ 5.63e-5}  &   \textcolor{black}{1888}   &  \textcolor{black}{3.90e-4 $\pm$ 1.01e-4}  &    \textcolor{black}{1928}       \\
  
\midrule[1pt]
  RSmote&  \textcolor{black}{\textbf{7.83e-4 $\pm$ 2.48e-4}}  &    \textcolor{black}{\textbf{1154}}      &  \textcolor{black}{\textbf{4.11e-4 $\pm$ 6.31e-5}}  &   \textcolor{black}{\textbf{1160}}   & \textcolor{black}{\textbf{3.87e-4 $\pm$ 7.63e-5}} &  \textcolor{black}{\textbf{1180}}  \\ 
  
\bottomrule[2pt]
\end{tabular}
\end{adjustbox}
\label{T.burges}
\end{table}

\begin{figure}[!h]
    \centering
    \subfigure[\#Sampling=2000]{\includegraphics[width=0.45\linewidth]{ 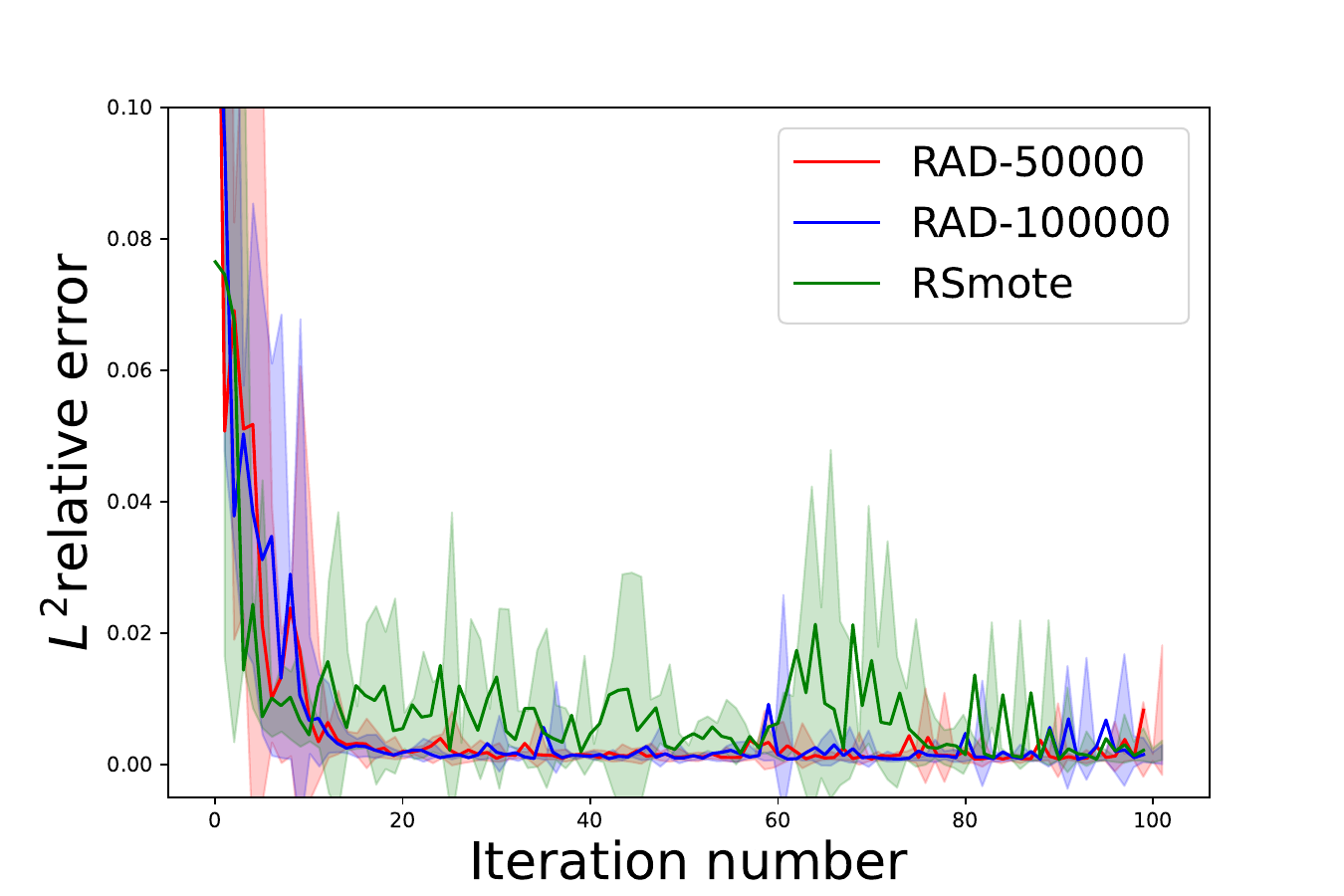}}
    \subfigure[\#Sampling=5000]{\includegraphics[width=0.45\linewidth]{ 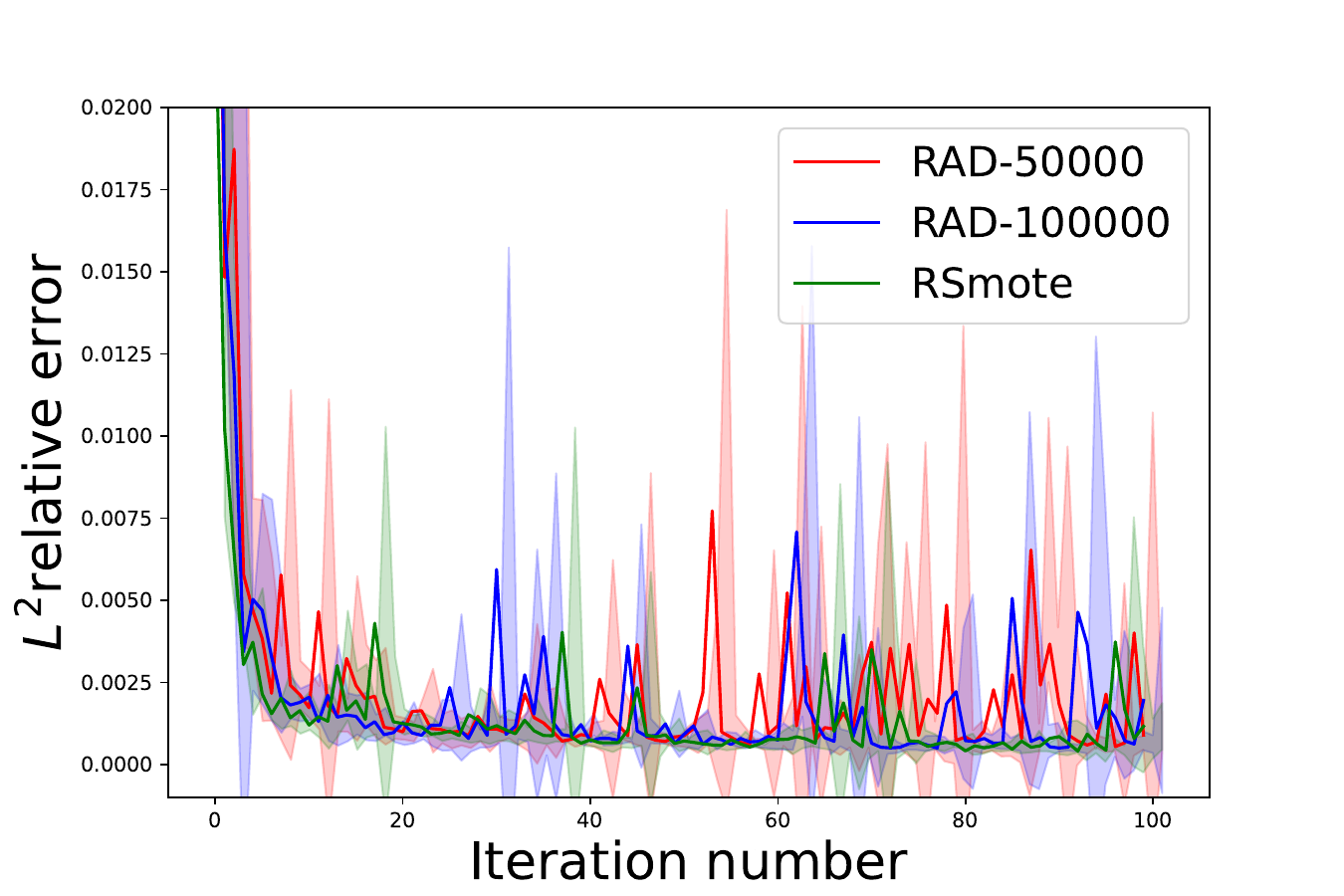}}
    \subfigure[\#Sampling=10000]{\includegraphics[width=0.45\linewidth]{ 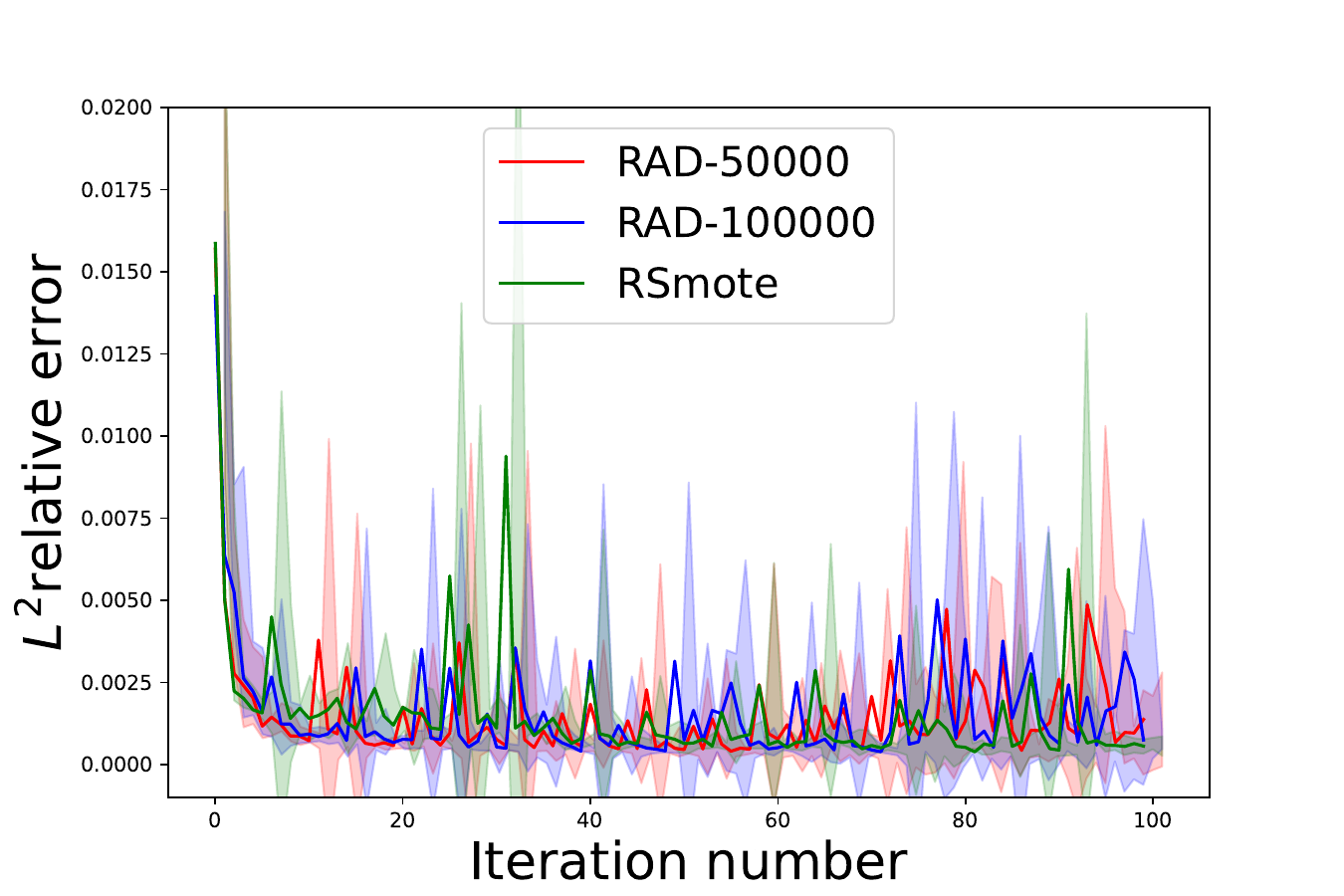}}
    \caption{\textcolor{black}{\textbf{Loss curves for Burgers’ Equation.} Red line: Mean values of RAD-50000 method;  Blue line: Mean values of RAD-100000 method; Green line: Mean values of RSmote method. The shaded areas represent the corresponding standard deviations.}}
    \label{f.burges}
\end{figure}

\begin{figure}[!h]
    \centering
    \subfigure[Exact solution]{\includegraphics[width=0.3\linewidth]{ 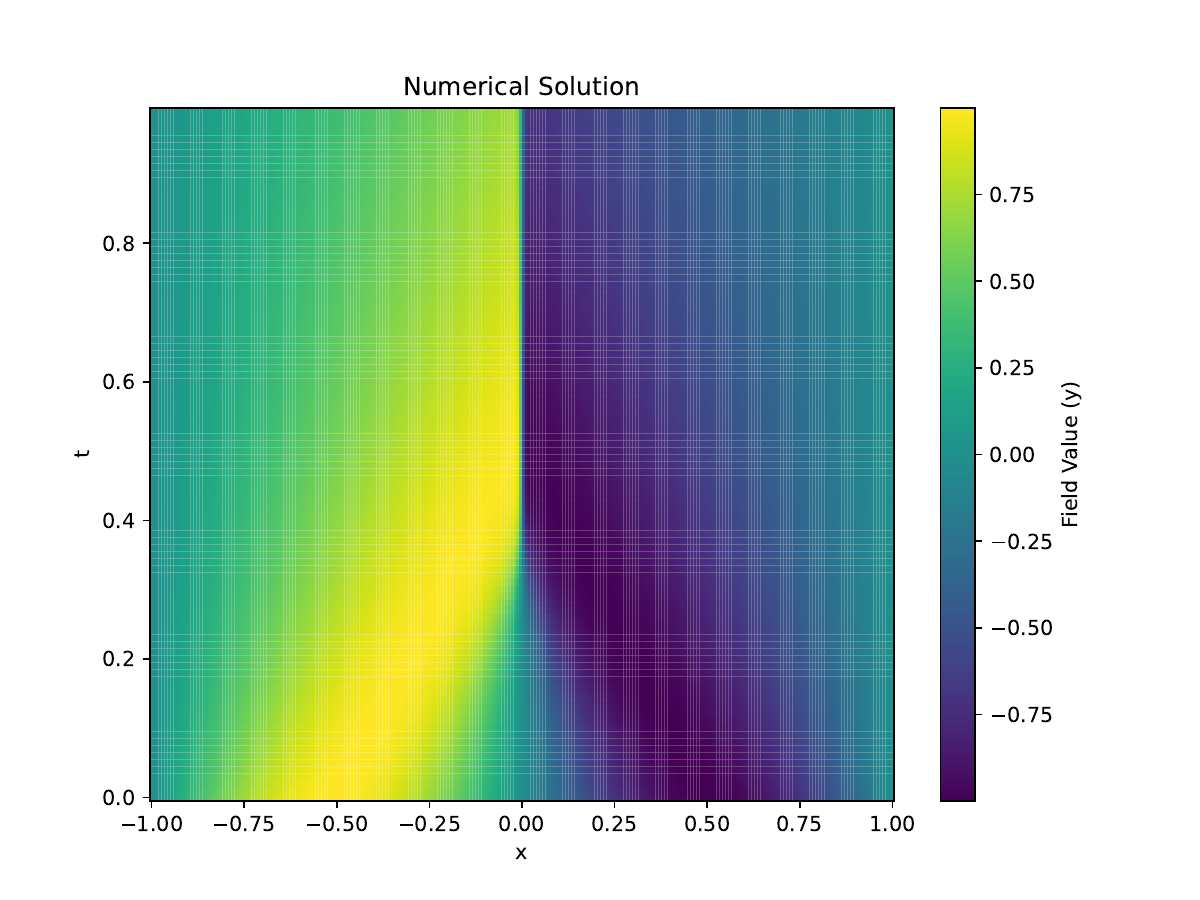}}
    \subfigure[RAD solution]
    {\includegraphics[width=0.3\linewidth]{ 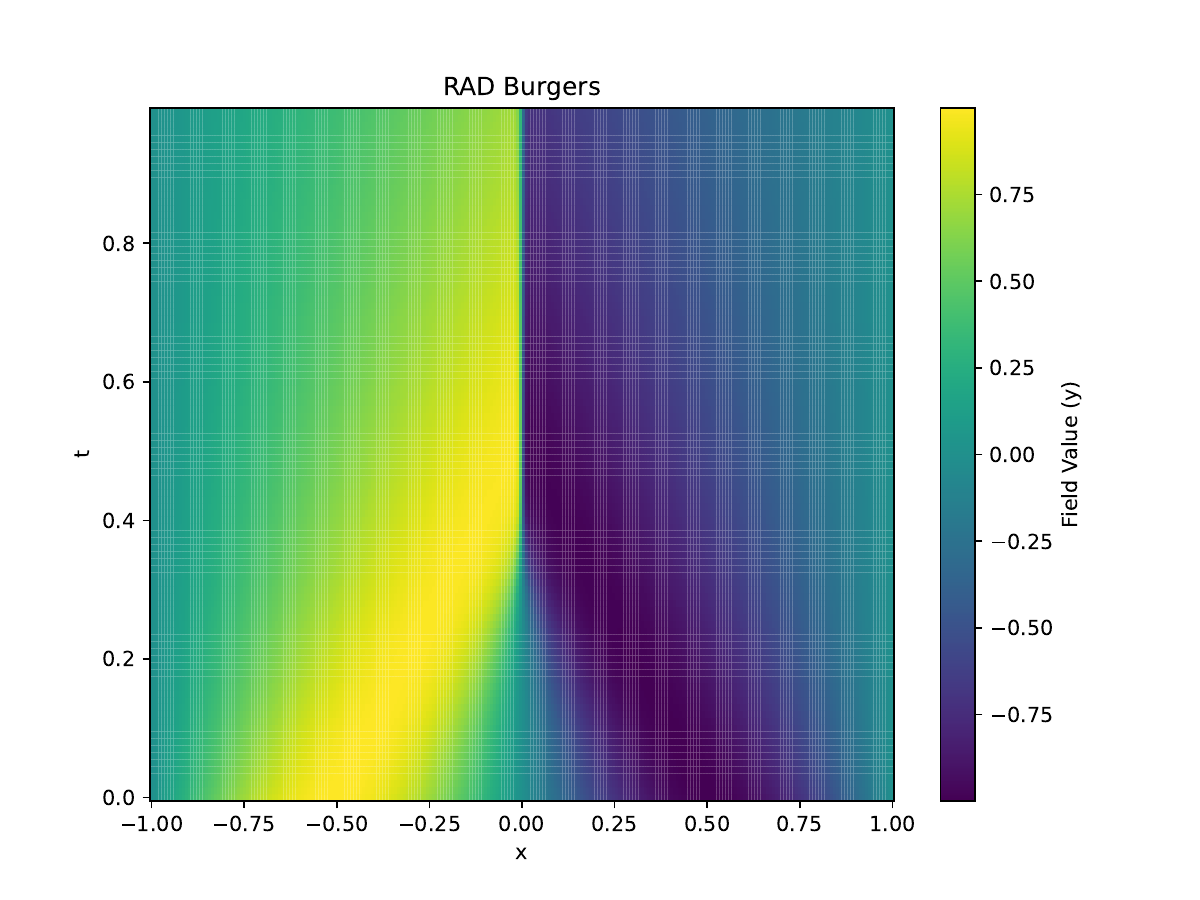}}
    \subfigure[RSmote solution]{\includegraphics[width=0.3\linewidth]{ 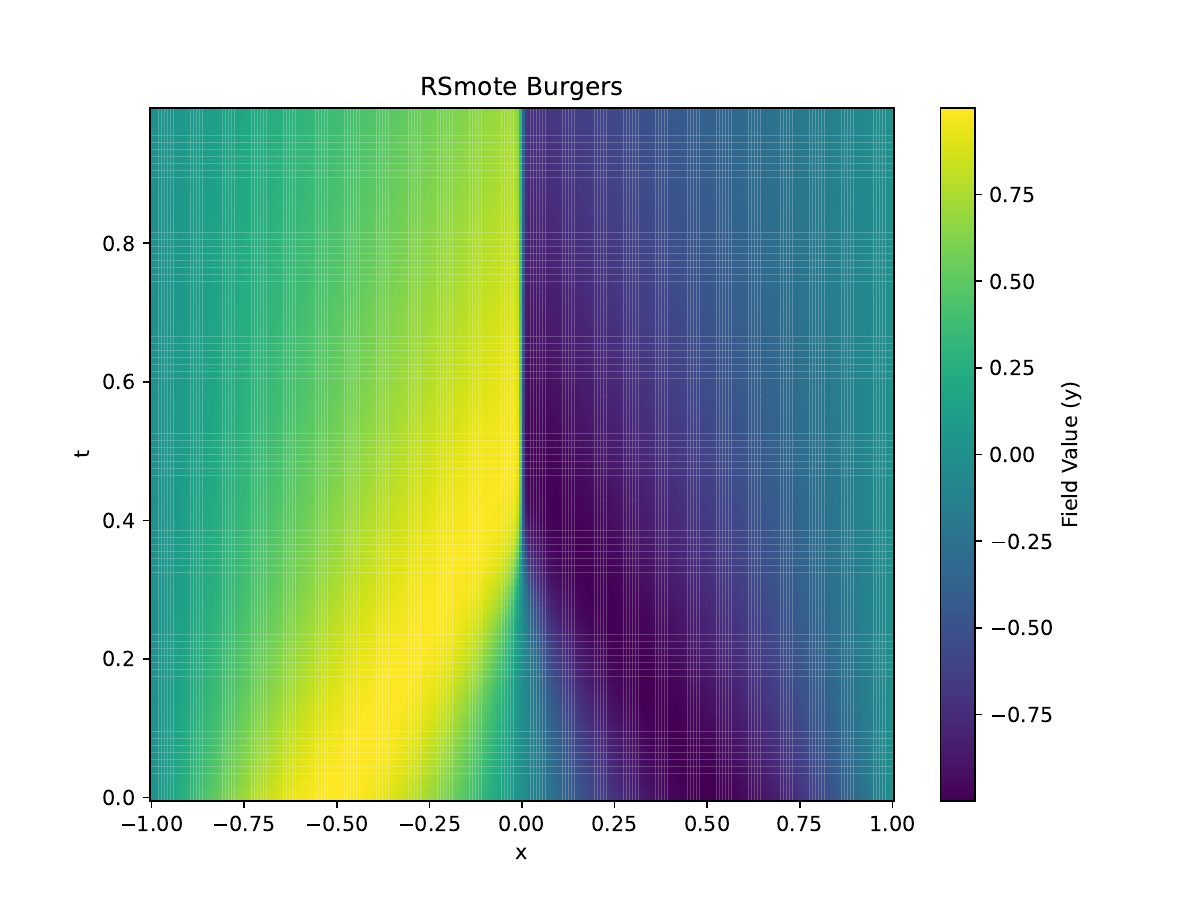}}
    \subfigure[RAD difference]{\includegraphics[width=0.3\linewidth]{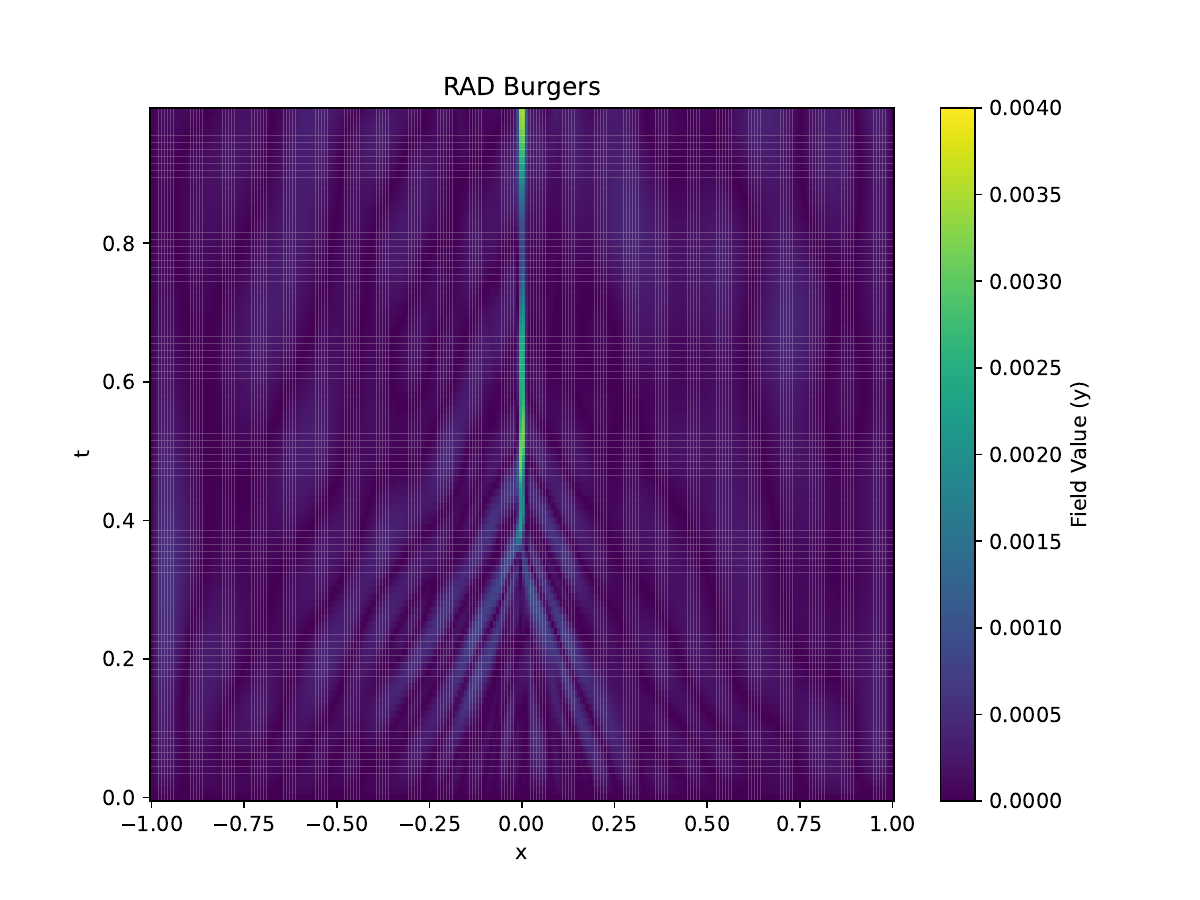}}
    \subfigure[RSmote difference]{\includegraphics[width=0.3\linewidth]{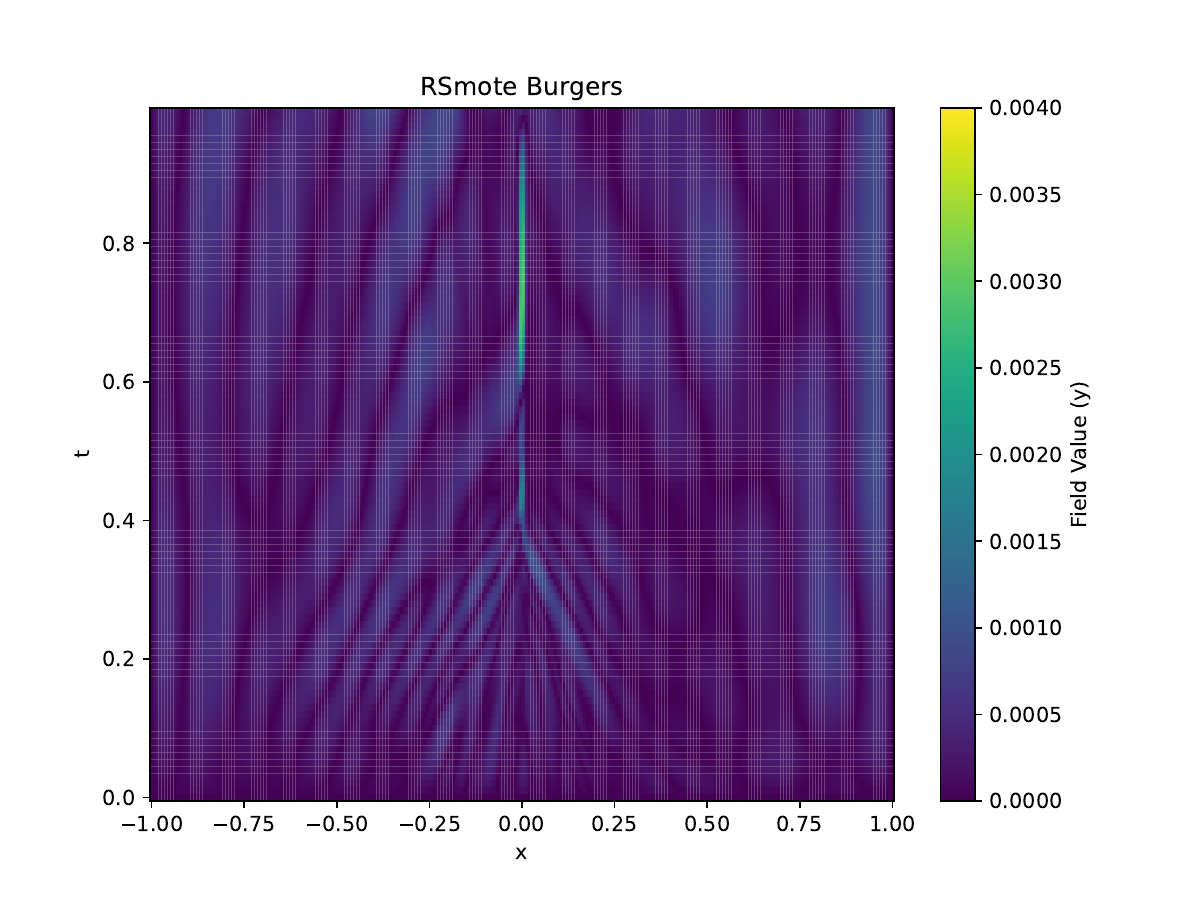}}
    \caption{\textcolor{black}{\textbf{Solution fields for Burgers' Equation. (a)-(c): ground truth, RAD solution and RSmote solution; (d)-(e): Absolute differences. }}}
    \label{f.burgers_field}
\end{figure}

\subsection{Allen-Cahn Equation}

The Allen-Cahn equation is defined as:

\begin{equation}
\label{e.allen}
\begin{gathered}
u_t = 0.001 u_{xx} + 5(u-u^3), \\
u(x, 0) = x^2\cos(\pi x), \\
u(-1, t) = u(1, t) = -1,  \\
x\in[-1, 1], t\in[0,1]. 
\end{gathered} 
\end{equation}

Table \ref{T.allen} presents the results obtained from different methods applied to solve the Allen-Cahn equation, with varying amounts of training data. 
The error is calculated on 20000 test points, \textcolor{black}{Fig.~\ref{f.allen} gives the loss curves and Fig.~\ref{f.allen_field} shows the field maps.}

For the RAD method, we observe that as the training data increases from 2000 to 10000, there is a consistent improvement in the score, indicating better performance in solving the Allen-Cahn equation. Also, the trend in scores aligns with the increase in the number of points used in probability estimation (residual points), showing improved performance. 
However, the memory requirements also exhibit a proportional increase.

In contrast, the RSmote method outperforms the other approaches in terms of both score and memory usage. It achieves significantly lower scores across all data sizes, highlighting its superior accuracy in solving the Allen-Cahn equation. Moreover, it maintains relatively low memory requirements, showcasing efficiency in computational resources.

\begin{table}[!h]
\renewcommand\arraystretch{1.1}
\centering\caption{Performance \textcolor{black}{(Mean $\pm$ Std.)} comparisons for different training data sizes on the Allen-Cahn equation using RAD-50000, RAD-100000, and RSmote methods. The table shows the scores and memory consumption for models trained with 2000, 5000, and 10000 data points. The lower the score, the better the performance. The \textbf{bold} indicates the best result.}
\begin{adjustbox}{width=\textwidth}
\begin{tabular}{l|cc|cc|cc}
\toprule[2pt]
Training data & \multicolumn{2}{c|}{\#Sampling=2000} & \multicolumn{2}{c|}{\#Sampling=5000} & \multicolumn{2}{c}{\#Sampling=10000} \\
Evaluation &   Score   &  Memory  & Score   &  Memory  & Score   &  Memory \\
\midrule[1pt]
  RAD-50000 &  \textcolor{black}{0.0115 $\pm$ 0.0022}       &   1498       &   \textcolor{black}{0.0070 $\pm$ 0.0017}     &    1504      & \textcolor{black}{0.0061 $\pm$ 0.0019}       &  1518        \\
  
  RAD-100000&  \textcolor{black}{0.0101 $\pm$ 0.0027}      &   1860       &  \textcolor{black}{0.0069 $\pm$ 0.0013}     &  1897        &  \textcolor{black}{0.0051 $\pm$ 0.0023}     &   1920       \\
  
\midrule[1pt]
  RSmote&   \textbf{\textcolor{black}{0.0039 $\pm$ 0.0009}}       &   \textbf{1120}       &  \textbf{\textcolor{black}{0.0037 $\pm$ 0.0005}}         &   \textbf{1132}       &   \textbf{\textcolor{black}{0.0029 $\pm$ 0.0002}}        &   \textbf{1174}     \\ 
\bottomrule[2pt]
\end{tabular}
\end{adjustbox}
\label{T.allen}
\end{table}

\begin{figure}[!h]
    \centering
    \subfigure[\#Sampling=2000]{\includegraphics[width=0.45\linewidth]{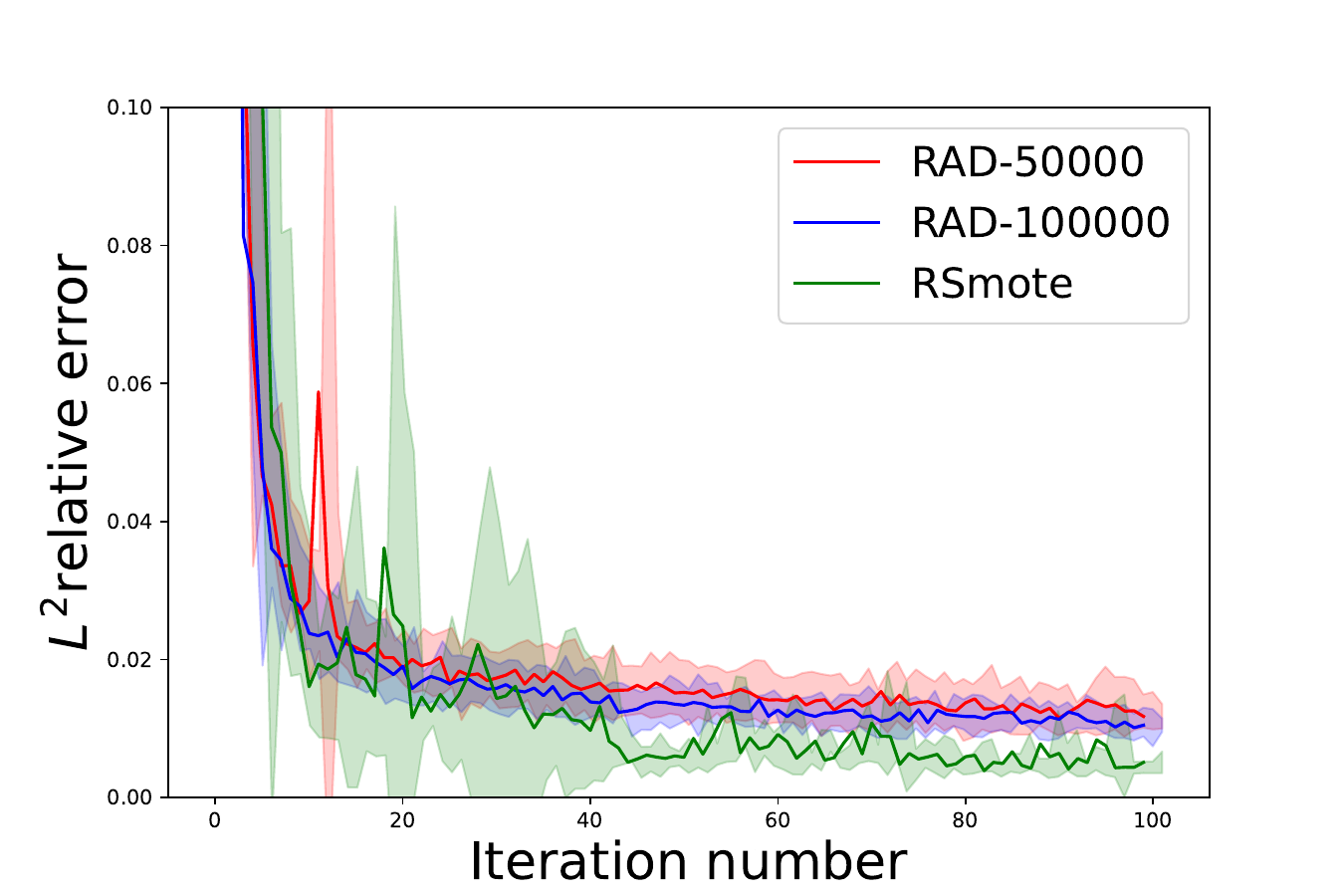}}
    \subfigure[\#Sampling=5000]{\includegraphics[width=0.45\linewidth]{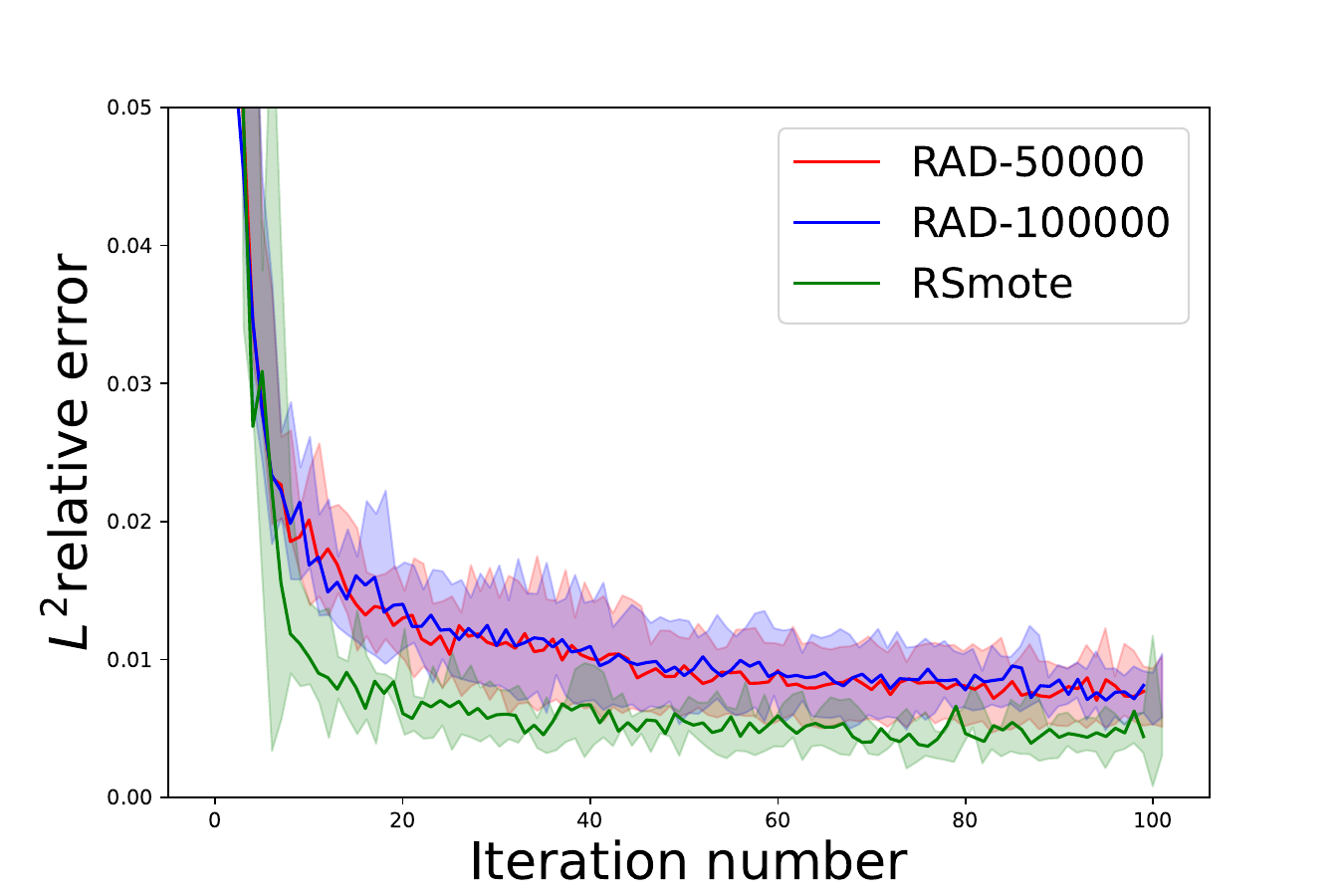}}
    \subfigure[\#Sampling=10000]{\includegraphics[width=0.45\linewidth]{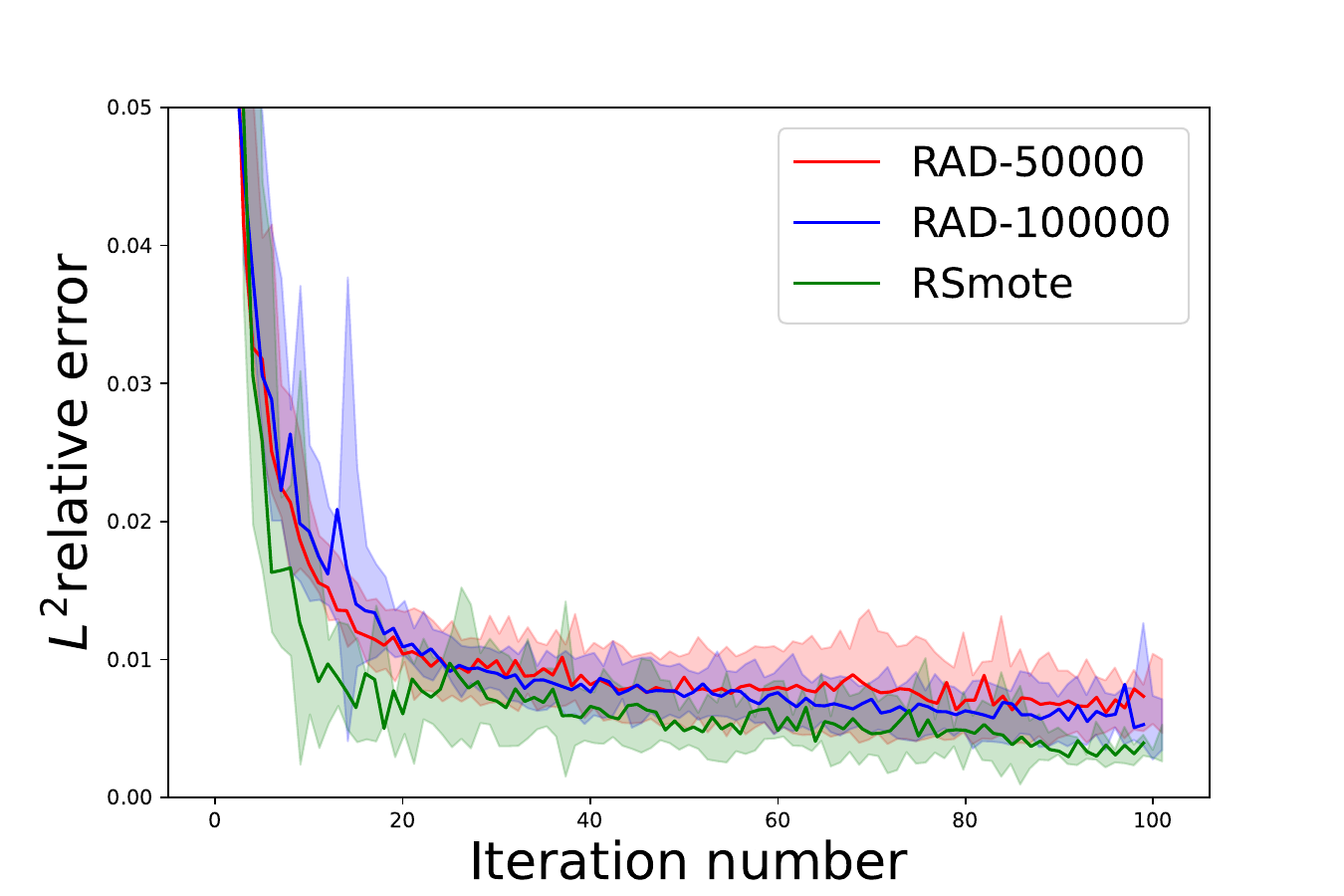}}
    \caption{\textcolor{black}{\textbf{Loss curves for Allen-cahn Equation}. Red line: Mean values of RAD-50000 method;  Blue line: Mean values of RAD-100000 method; Green line: Mean values of RSmote method. The shaded areas represent the corresponding standard deviations.}}
    \label{f.allen}
\end{figure}

\begin{figure}[!h]
    \centering
    \subfigure[Exact solution]{\includegraphics[width=0.3\linewidth]{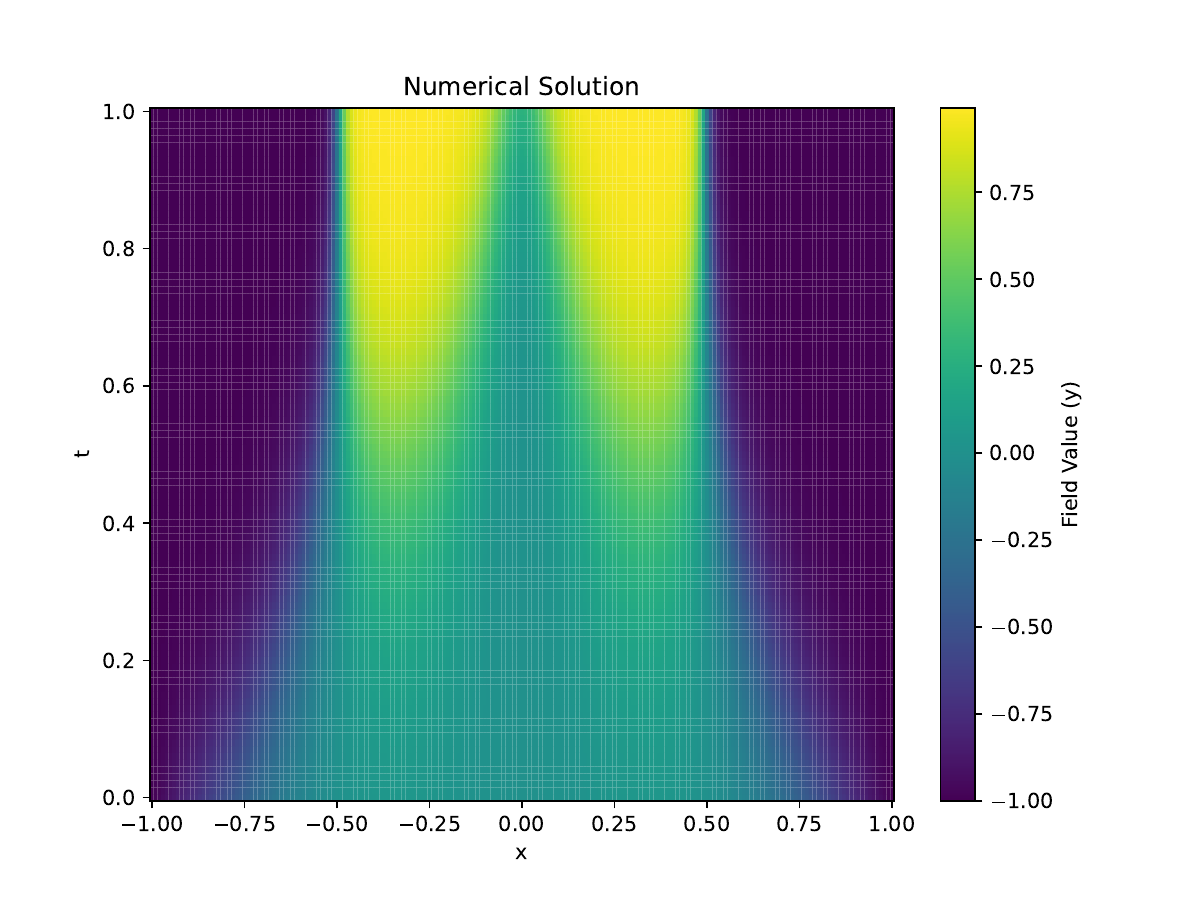}}
    \subfigure[RAD solution]
    {\includegraphics[width=0.3\linewidth]{ 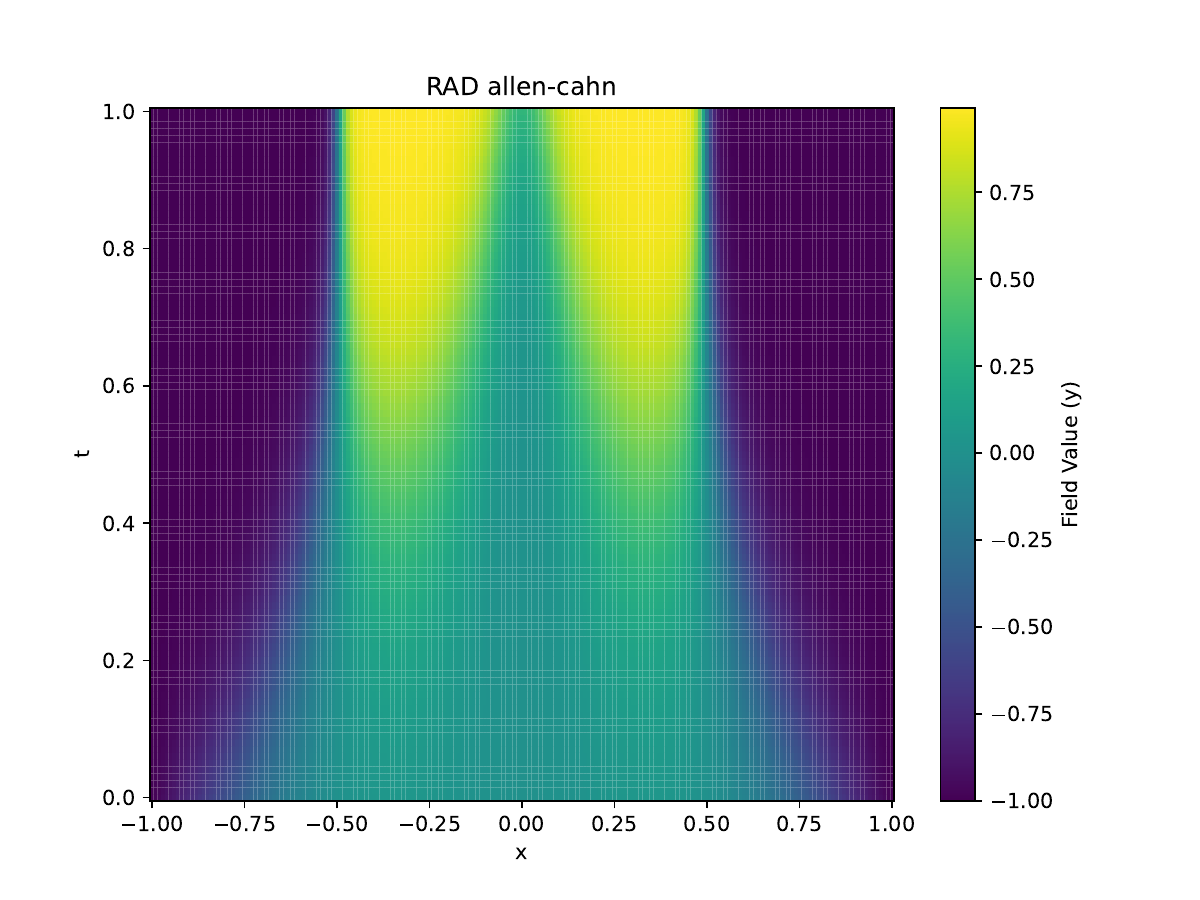}}
    \subfigure[RSmote solution]{\includegraphics[width=0.3\linewidth]{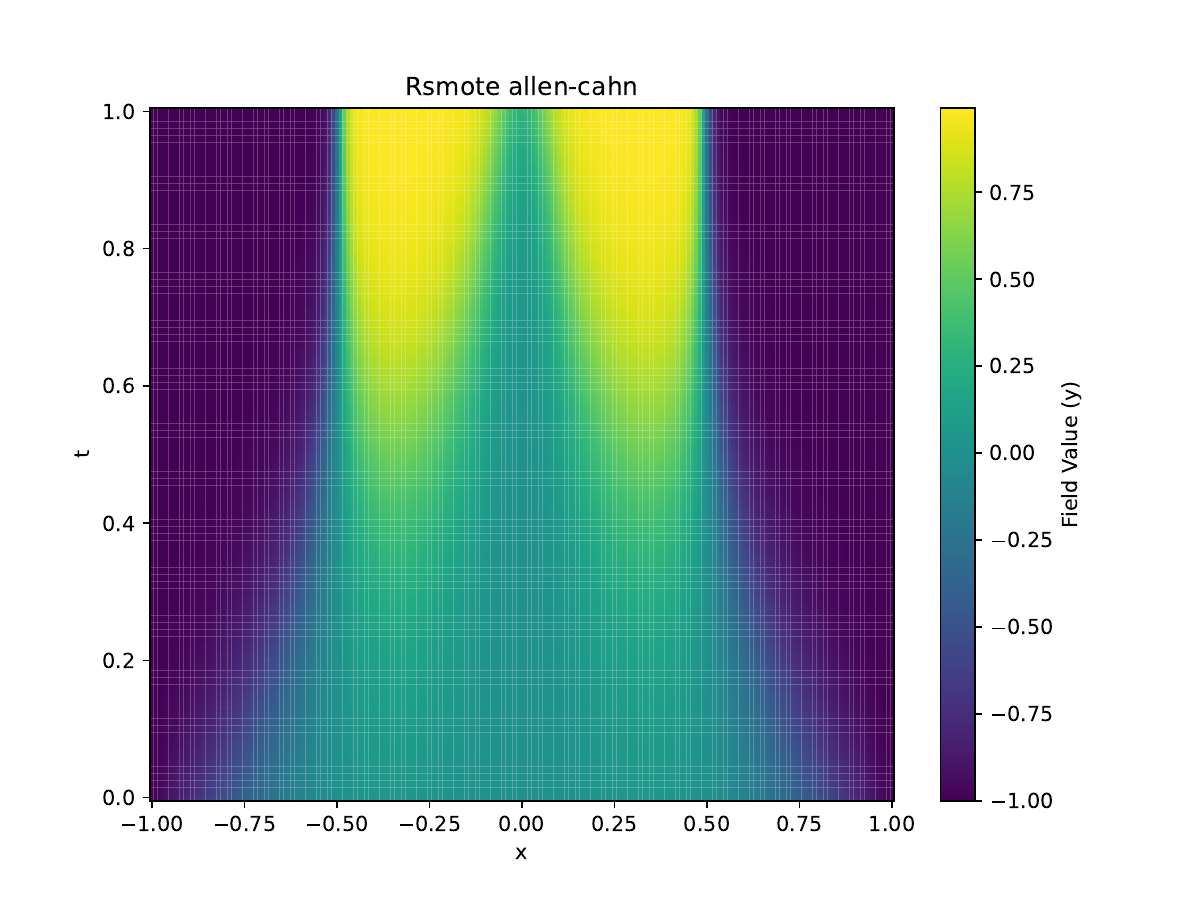}}
    \subfigure[RAD difference]{\includegraphics[width=0.3\linewidth]{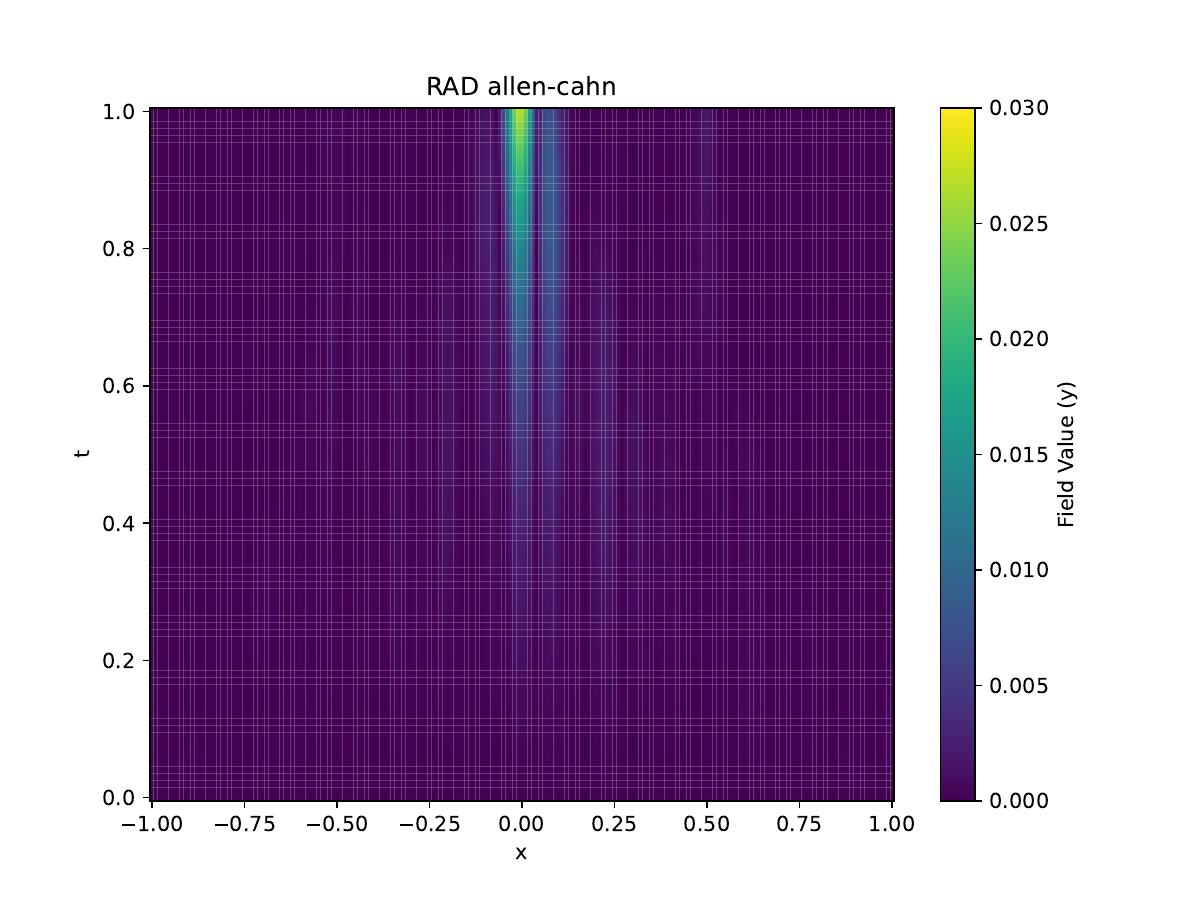}}
    \subfigure[RSmote difference]{\includegraphics[width=0.3\linewidth]{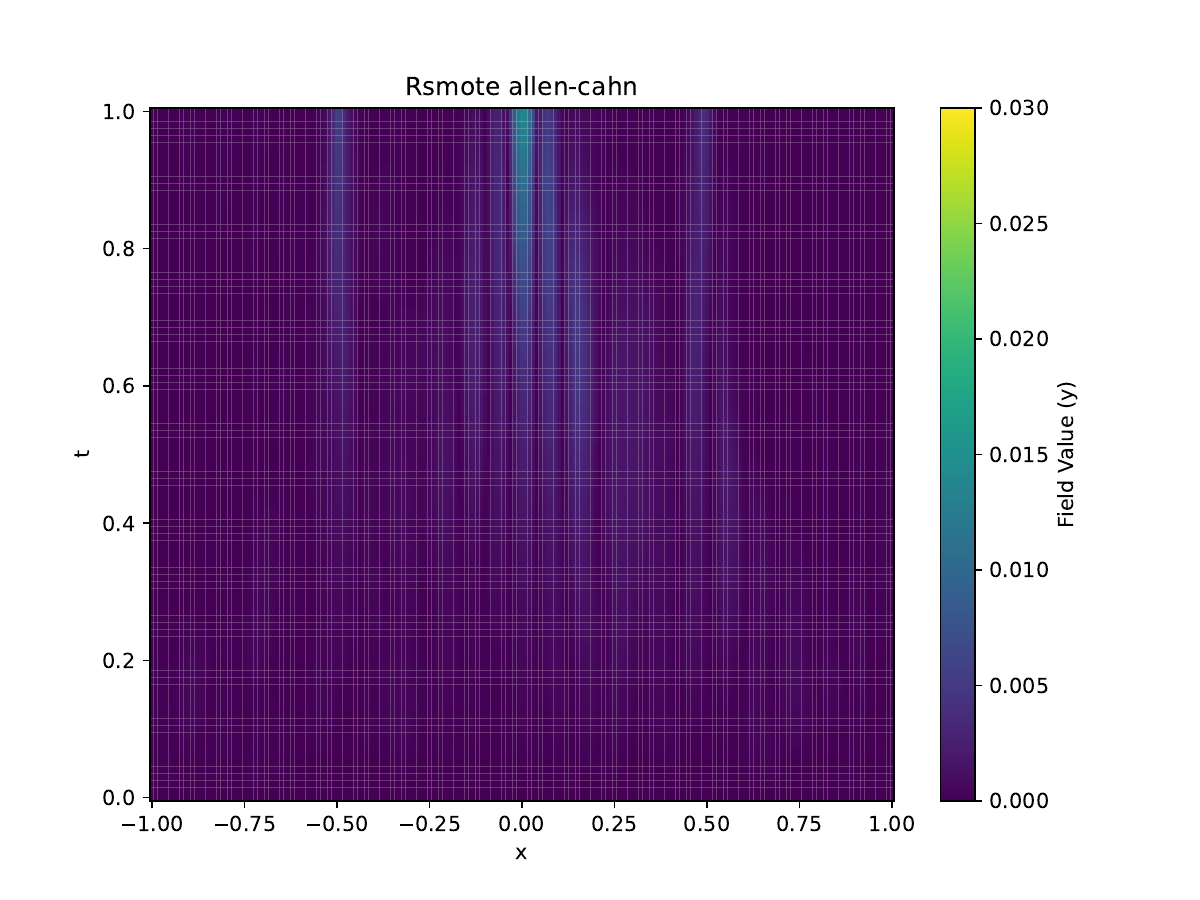}}
    \caption{\textcolor{black}{\textbf{Solution fields for Allen-Cahn Equation. (a)-(c): ground truth, RAD solution and RSmote solution; (d)-(e): Absolute differences.}}}
    \label{f.allen_field}
\end{figure}

\subsection{Elliptic Equation}

\begin{equation}
\label{e.highdimelliptic}
\left\{\begin{array}{l}
-\bigtriangleup u = f,~~ \mathbf{x}\in \Omega\\
u = g,~~ \mathbf{x}\in \partial\Omega, \\
\end{array}\right.    
\end{equation}
where $\Omega=[-1, 1]^d$, $f(\mathbf{x})=\frac{1}{d}(\sin(\frac{1}{d}\sum \limits_{i=1}^{d}x_i)-2)$, which admit the exact solution $u(\mathbf{x})=(\frac{1}{d}\sum \limits_{i=1}^{d}x_i)^2+\sin(\frac{1}{d}\sum \limits_{i=1}^{d}x_i)$.

In this experiment, we evaluated the performance of three different methods, namely RAD-100000, RAD-200000, and RSmote, for an Elliptic equation with different dimensions. The results are presented in Table \ref{T.elliptic} and Fig.~\ref{f.elliptic}. The dimensions of the problem are d=10, d=20, and d=100.

\textcolor{black}{
The experimental results demonstrate the superiority of the proposed RSmote method over the RAD-100000 and RAD-200000 techniques across various dimensions and evaluation metrics. In the low-dimensional case of d=10, RSmote achieved the lowest mean score of 5.9e-4, outperforming both RAD-100000 (1.1e-3) and RAD-200000 (1.0e-3). Moreover, RSmote exhibited a significantly lower memory usage of 1874, compared to 3676 and 4310 for RAD-100000 and RAD-200000, respectively. This trend continued in the higher-dimensional case of d=20, where RSmote attained the best score of 4.5e-3 while requiring substantially less memory (5334) than RAD-100000 (13088) and RAD-200000 (15834). Remarkably, in the high-dimensional scenario of d=100, RSmote was the only method capable of producing a solution, achieving a score of 3.9e-3 with a memory usage of 23018. These results highlight the computational efficiency and scalability of RSmote, making it a superior choice for high-dimensional problems where existing techniques struggle or fail to find solutions.
}

\begin{table}[!ht]
\renewcommand\arraystretch{1.1}
\footnotesize
\centering\caption{Results for a Elliptic equation using RAD-100000, RAD-200000, and RSmote methods. The table gives the $L_2$ relative errors and memory cost for models trained with dimensions d=10, 20, and 100. - represents there exists GPU overflow. The \textbf{bold} indicates the best result.}
\begin{tabular}{l|c|ccc}
\toprule[2pt]
        Dimension  & Metric & RAD-100000 & RAD-200000 & RSmote \\
\midrule[1pt]
\multirow{2}{*}{d=10} & Score & \textcolor{black}{1.1e-3$\pm$ 2.9e-4} & \textcolor{black}{1.0e-3$\pm$ 3.1e-4} & \textcolor{black}{\textbf{5.9e-4$\pm$ 1.5e-4 }} \\
                  & Memory  & 3676 & 4310 & \textbf{1874} \\
                  
\midrule[1pt]
\multirow{2}{*}{d=20} & Score & \textcolor{black}{6.8e-3$\pm$ 2.4e-3} & \textcolor{black}{5.9e-3$\pm$ 2.8e-3} & \textbf{\textcolor{black}{4.5e-3$\pm$ 1.7e-3}} \\
                  & Memory  & 13088 & 15834 & \textbf{5334} \\
                  
\midrule[1pt]
\multirow{2}{*}{d=100} & Score  & - & - & \textbf{\textcolor{black}{3.9e-3$\pm$1.2e-3}} \\
                  & Memory  &-  & - & \textbf{18790} \\
\bottomrule[2pt]
\end{tabular}

\label{T.elliptic}
\end{table}

\begin{figure}[!h]
\centering
\subfigure[Dimension=10]{\includegraphics[width=0.45\linewidth]{ 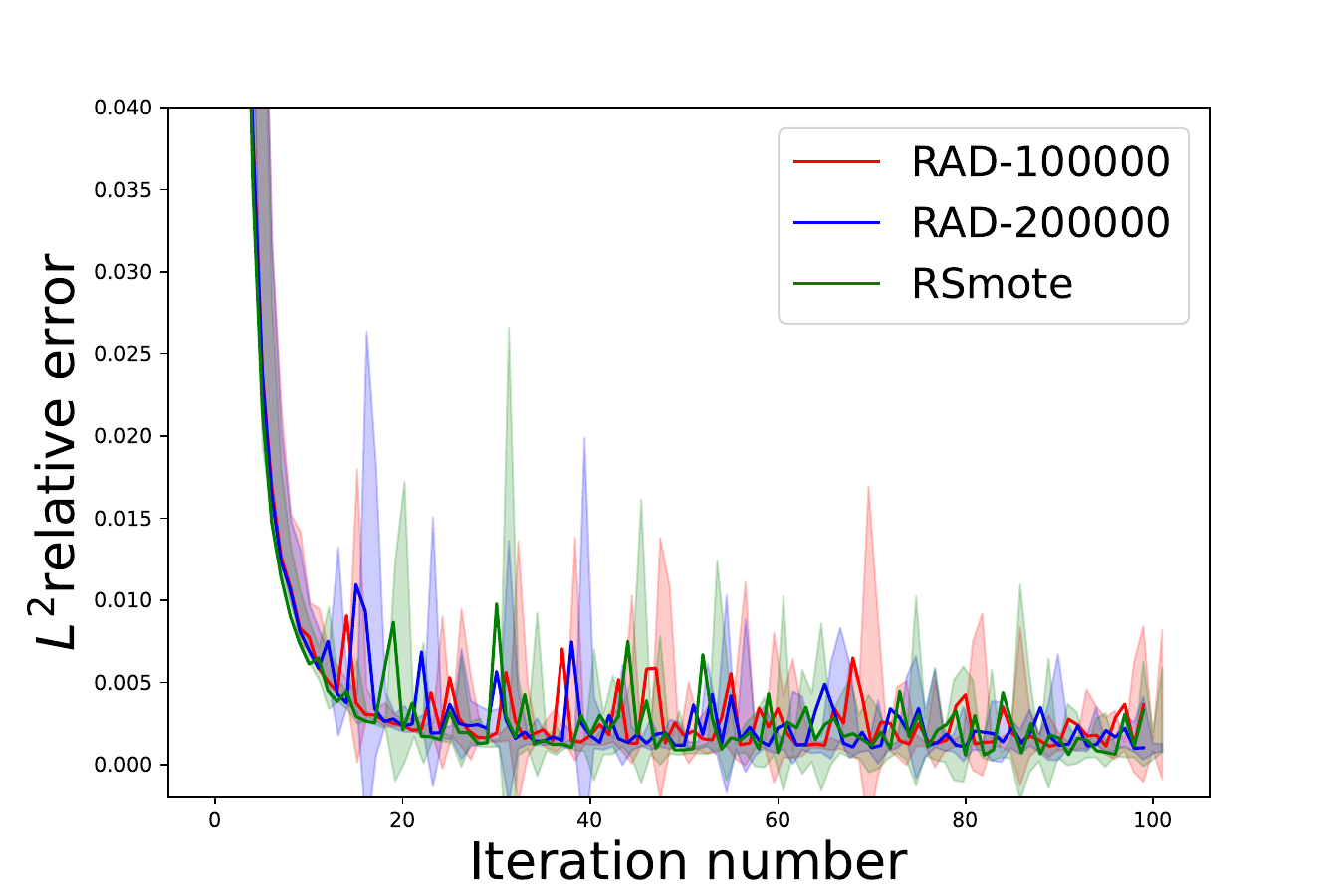}}
\subfigure[Dimension=20]{\includegraphics[width=0.45\linewidth]{ 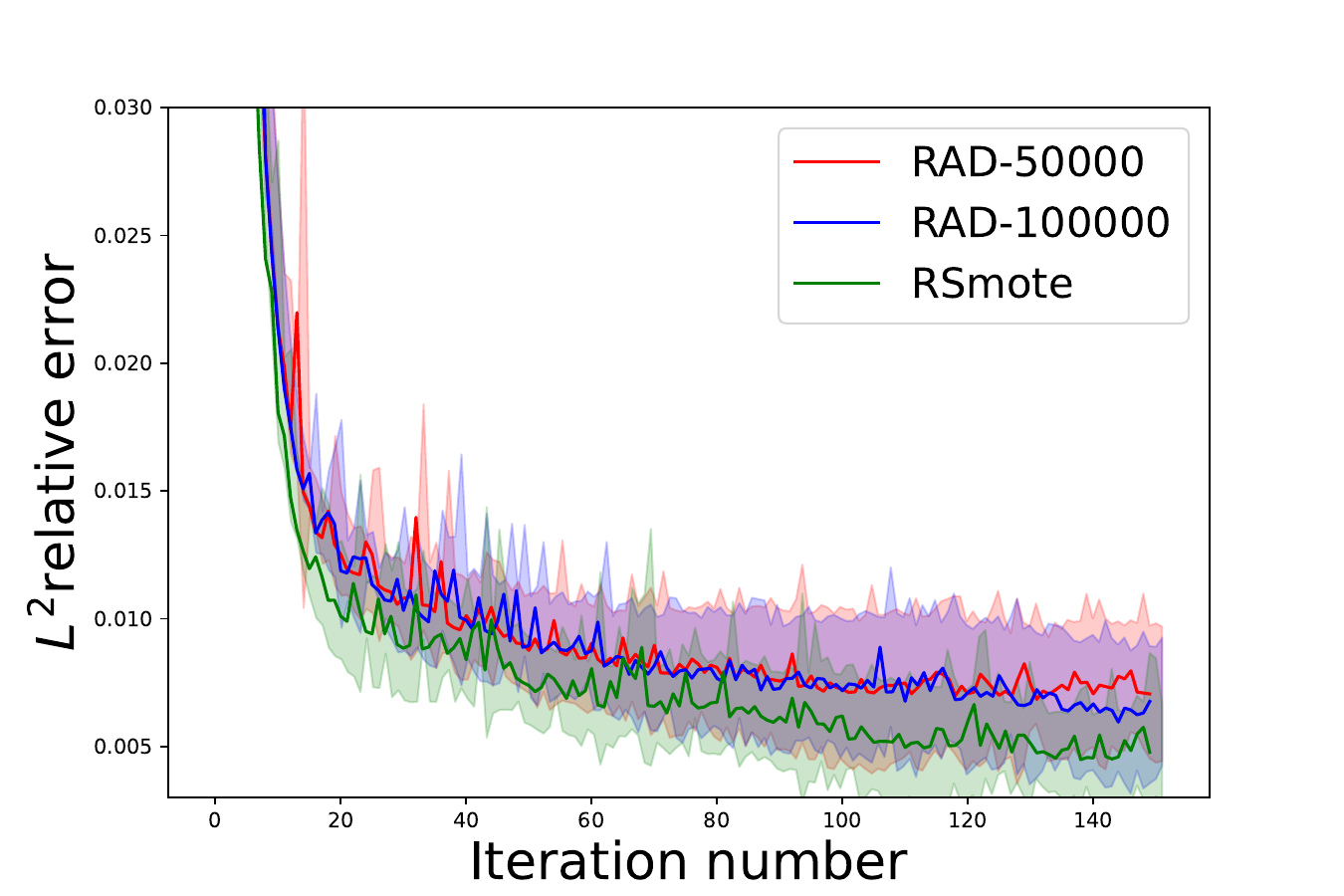}}
\caption{\textcolor{black}{\textbf{Loss curves for Elliptic Equation}. Red line: Mean values of RAD-100000 method;  Blue line: Mean values of RAD-200000 method; Green line: Mean values of RSmote method. The shaded areas represent the corresponding standard deviations.}}
\label{f.elliptic}
\end{figure}

\subsection{Reaction-Diffusion Equation}

\begin{equation}
\label{e.highdimreaction}
\left\{\begin{array}{l}
u_t = \bigtriangleup u - 0.2u - de^{-0.2t},~~ \mathbf{x}\in \Omega \\
u(\mathbf{x}, 0) = \|\mathbf{x}\|^2/2, \\
\end{array}\right. 
\end{equation}
where $d$ is the dimension, $\Omega=[-1, 1]^d$ and $t\in [0,1]$. The analytical solution is $u(\mathbf{x}, t) = \frac{1}{2}\|\mathbf{x}\|^2 e^{-0.2t}$.

In this experiment, we evaluated the performance of two methods, namely RAD-100000 and RSmote, for the Reaction-Diffusion equations with different dimensions. We want to show more samples can lead to better results but require more memory. RSmote will have more advantages because it uses fewer resources and can be applied for higher dimensional PDEs.
The results are presented in Table \ref{T.reaction} and Fig.~\ref{f.reaction}.

\textcolor{black}{
In the 20-dimensional case, RAD achieves a score of 7.4e-3 using 7286 MB of memory and 10000 training data points. Increasing the number of points to 50000 and 80000 improves the score to 7.2e-3 and 7.1e-3, but memory usage rises to 10028 MB and 15476 MB, respectively. In contrast, RSmote outperforms RAD by achieving the best score of 5.2e-3 while using less memory (6374 MB).
This pattern persists in the 30-dimensional case. RAD with 10000 samples scores 1.7e-2 (15182 MB), improving to 1.4e-2 (20710 MB) with 50000 samples. However, RSmote with 80000 points delivers the best score of 8.1e-3 while RAD meets GPU overflow.
For the 50-dimensional case, only RSmote results are available, showing scores of 0.16 and 0.13 for 10000 and 50000 samples, respectively.
}

\textcolor{black}{
The results demonstrate that increased sampling generally improves accuracy, as evidenced by RAD with 80000 samples consistently outperforming RAD with 10000 points in score.} However, this improvement comes at the cost of significantly higher memory usage. RSmote stands out by consistently achieving the best scores while using the least memory across all dimensions and sample sizes, showcasing superior efficiency.

As the dimension increases, the scores (errors) generally decrease, indicating that the problem becomes more challenging. Memory requirements also increase substantially with dimension for all methods. The absence of data for RAD methods at d=50 limits full comparison across all dimensions.

In conclusion, while increasing samples in RAD methods does improve results, it incurs a substantial memory cost. RSmote demonstrates remarkable efficiency, consistently achieving the best accuracy with the lowest memory usage across all tested scenarios. This suggests RSmote may be particularly valuable for high-dimensional problems or when computational resources are limited.

\begin{table}[!h]
\renewcommand\arraystretch{1.1}
\footnotesize
\centering\caption{Results of Reaction-Diffusion equation. The table introduces the scores and memory cost for models trained with different dimensions and diverse data points. - represents there exists GPU overflow. \#Sampling refers to the number of training points. The \textbf{bold} indicates the best result.}
\begin{adjustbox}{width=\textwidth}
\begin{tabular}{l|c|cc|cc|cc}
\toprule[2pt]
\multirow{2}{*}{Dimension} & \multirow{2}{*}{Metric} & \multicolumn{2}{c|}{\#Sampling=10000} & \multicolumn{2}{c|}{\#Sampling=50000} & \multicolumn{2}{c}{\#Sampling=80000}\\
   &    &   RAD-100000 &  RSmote & RAD-100000  & RSmote & RAD-100000  & RSmote\\
   
\midrule[1pt]
\multirow{2}{*}{d=20} & Score & \textcolor{black}{7.4e-3$\pm$1.9e-4} & \textcolor{black}{\textbf{5.4e-3$\pm$2.1e-4}} & \textcolor{black}{7.2e-3$\pm$3.7e-4} & \textcolor{black}{\textbf{5.3e-3$\pm$3.0e-4}} & \textcolor{black}{7.1e-3$\pm$2.8e-4}& \textcolor{black}{\textbf{5.2e-3$\pm$4.4e-4}}\\
     & Memory  & 7286 & \textbf{1858} & 10228 & \textbf{4856}& 15476 & \textbf{6374}\\
                  
\midrule[1pt]
\multirow{2}{*}{d=30} & Score & \textcolor{black}{1.7e-2$\pm$2.8e-3} & \textcolor{black}{\textbf{1.2e-2$\pm$1.5e-3}} & \textcolor{black}{1.4e-2$\pm$2.7e-3}& \textcolor{black}{\textbf{9.0e-3$\pm$1.5e-3}} &- & \textcolor{black}{\textcolor{black}{\textbf{8.1e-3$\pm$5.2e-4}}} \\
    & Memory  & 15182 & \textbf{2700} & 20710 & \textbf{8228}& - & \textbf{12222}\\
                  
\midrule[1pt]
\multirow{2}{*}{d=50} & Score & - & \textcolor{black}{\textbf{0.16$\pm$5.9e-3}} & - & \textcolor{black}{\textbf{0.13$\pm$3.2e-3}} &- &- \\
                  & Memory  & - & \textbf{5692} & - & \textbf{20010} &- &- \\
\bottomrule[2pt]
\end{tabular}
\end{adjustbox}
\label{T.reaction}
\end{table}

\begin{figure}[!ht]
    \centering
    \subfigure[Dimension=20]{\includegraphics[width=0.45\linewidth]{ 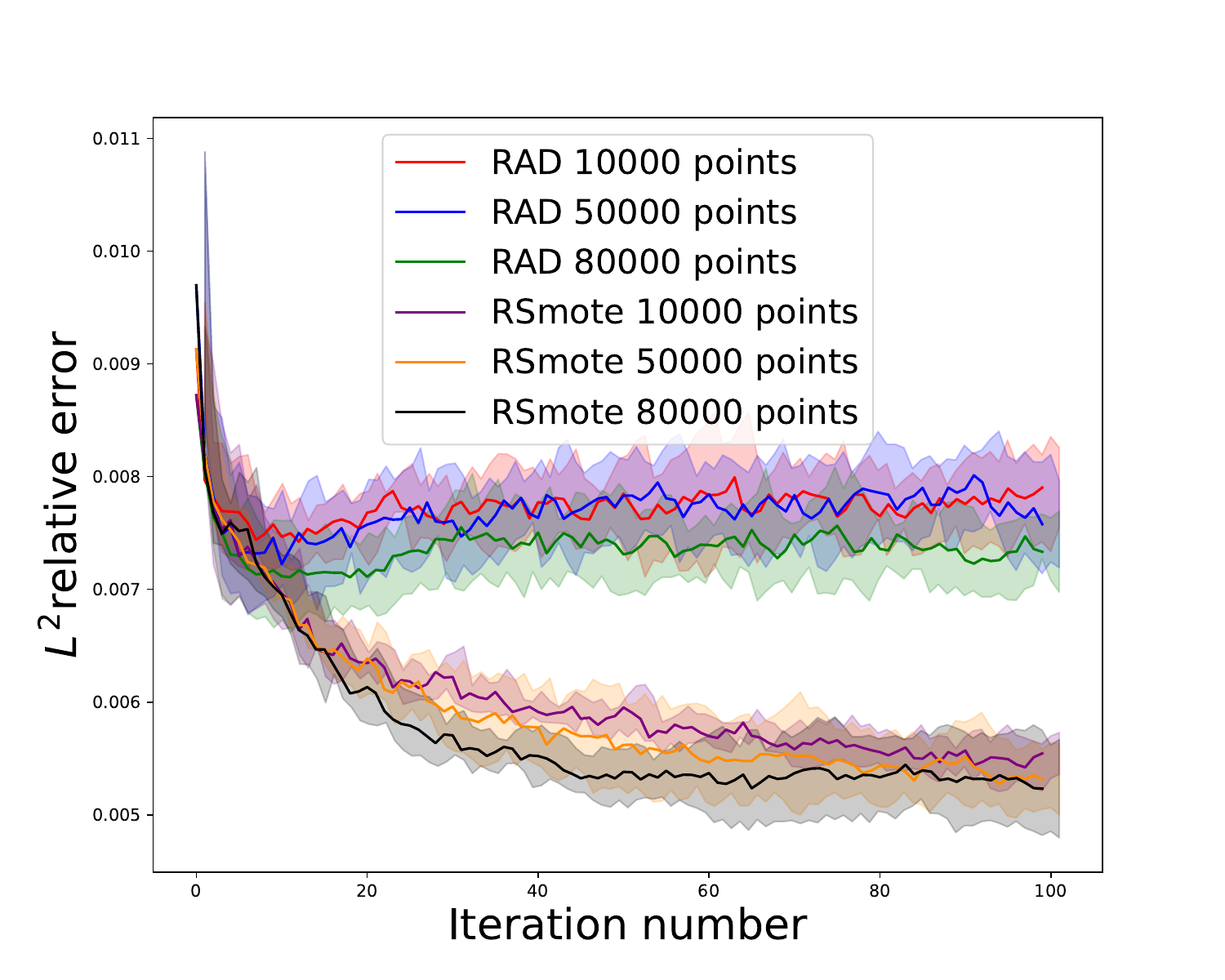}}
    \subfigure[Dimension=30]{\includegraphics[width=0.45\linewidth]{ 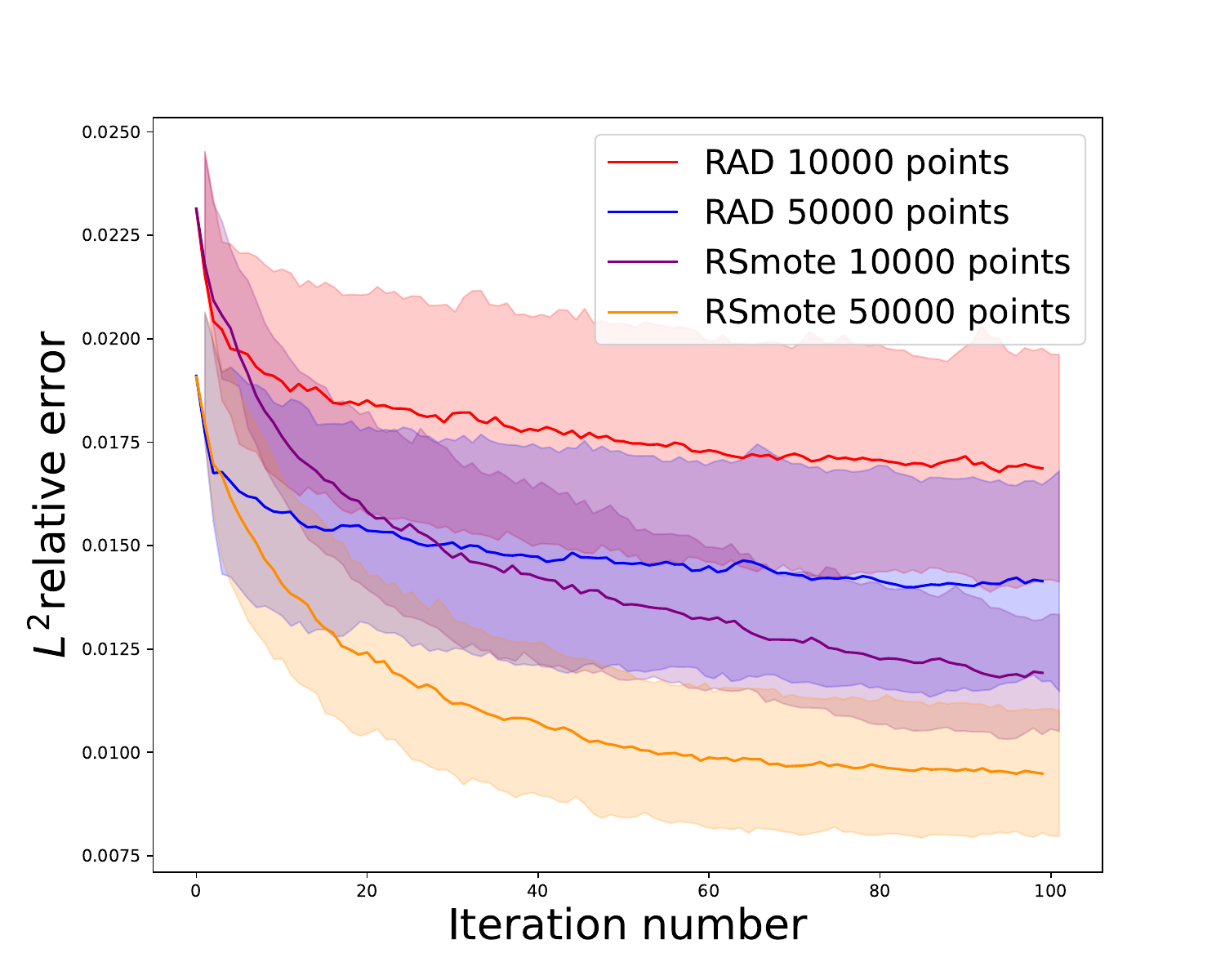}}
    \caption{\textcolor{black}{\textbf{Loss curves for Reaction-Diffusion equation}. Red line: Mean values of RAD-100000 with 10000 training points;  Blue line: Mean values of RAD-100000 with 50000 training points; Green line: Mean values of RAD-100000 with 80000 training points; Purple line: Mean values of RSmote with 10000 training points; Orange line: Mean values of RSmote with 50000 training points; Black line: Mean values of RSmote with 80000 training points. The shaded areas represent the corresponding standard deviations.}}
    \label{f.reaction}
\end{figure}

\subsection{\textcolor{black}{Hyper-parameter selection}}
\textcolor{black}{
In this subsection, we conduct additional experiments to examine the impact of the ratio $\lambda$ on the Allen-Cahn equation. The results are presented in Table \ref{T.allenoscillation} and Fig.~\ref{f.allen_oscillation}. The findings indicate that increasing the ratio significantly improves the model's performance, with the best performance observed at a ratio of 0.45 for both dataset sizes. Additionally, the sensitivity of the ratio decreases as the total number of sampling points increases. 
}

\textcolor{black}{
For a smaller dataset size (\#Sampling = 2000), the mean error decreases significantly as the ratio increases. At a ratio of 0.15, the error is
0.2861 $\pm$ 0.2128, whereas at a ratio of 0.45, the error is drastically reduced to 
0.0039 $\pm$ 0.0009. For a larger dataset size (\#Sampling = 10000), the error also decreases with higher ratios, but at a slower rate. This indicates that higher ratios consistently improve performance, with more pronounced benefits for smaller datasets.
The loss curves further illustrate the convergence behavior of RSmote for different ratios. For both dataset sizes, higher ratios (e.g., 0.45) achieve rapid and stable convergence with smaller error values. Conversely, lower ratios (e.g., 0.15) exhibit slower convergence and higher variability across iterations.}

\textcolor{black}{The results consistently demonstrate that larger sampling ratios enhance the model's stability and accuracy. This improvement arises because smaller ratios lead to highly imbalanced sampling, causing the data to concentrate excessively around large residuals while undersampling other regions. This imbalance ultimately results in diminished performance. However, the performance gain becomes less pronounced with larger datasets. For example, the difference in errors between ratios of 0.35 and 0.45 for \#Sampling = 10000 is relatively small. This suggests that while the RSmote method significantly benefits from higher sampling ratios, the impact is more pronounced for smaller datasets.}

\begin{table}[!h]
\renewcommand\arraystretch{1.1}
\centering\caption{\textcolor{black}{Performance (Mean $\pm$ Std.) comparisons for different training data sizes on the Allen-Cahn equation using the RSmote method with different ratios. The table shows the scores for models trained with 2000 and 10000 data points. The lower the score, the better the performance. The \textbf{bold} indicates the best result.}}
\begin{adjustbox}{width=0.7\textwidth}
\begin{tabular}{l|c|c}
\toprule[2pt]
Ratio & \#Sampling=2000 & \#Sampling=10000 \\
\midrule[1pt]
0.15 &  \textcolor{black}{0.2861 $\pm$ 0.2128}     & \textcolor{black}{0.0051 $\pm$ 0.0013}    \\

0.25 &  \textcolor{black}{0.2927 $\pm$ 0.0944}  &  \textcolor{black}{0.0033 $\pm$ 0.0015}  \\

0.35 &  \textcolor{black}{0.0048 $\pm$ 0.0027} &  \textcolor{black}{0.0030 $\pm$ 0.0004}     \\ 

0.45 &   \textbf{\textcolor{black}{0.0039 $\pm$ 0.0009}} &  \textbf{\textcolor{black}{0.0029 $\pm$ 0.0002}}      \\ 

\bottomrule[2pt]
\end{tabular}
\end{adjustbox}
\label{T.allenoscillation}
\end{table}

\begin{figure}[!h]
    \centering
    \subfigure[\#Sampling=2000]{\includegraphics[width=0.45\linewidth]{ 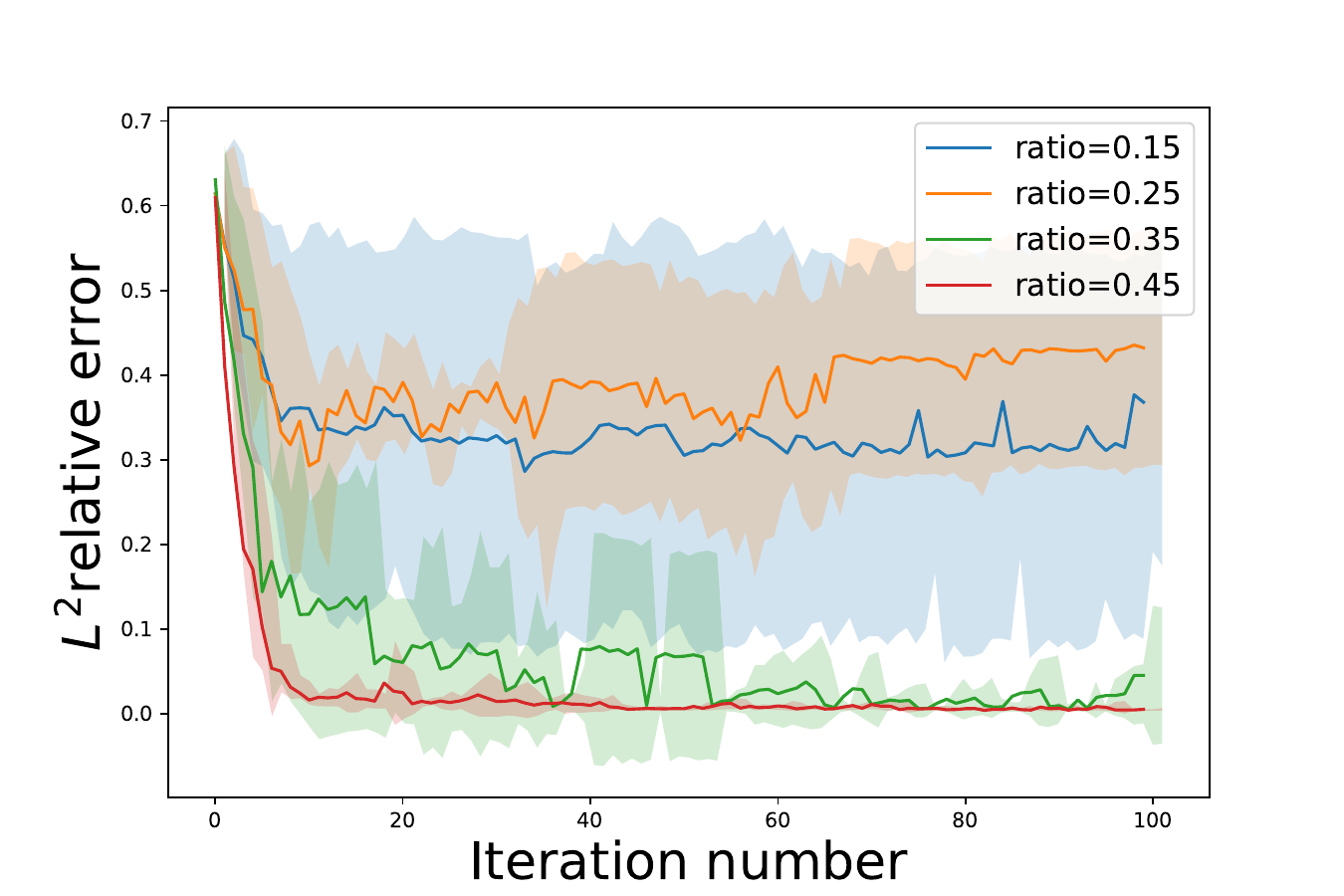}}
    \subfigure[\#Sampling=10000]{\includegraphics[width=0.45\linewidth]{ 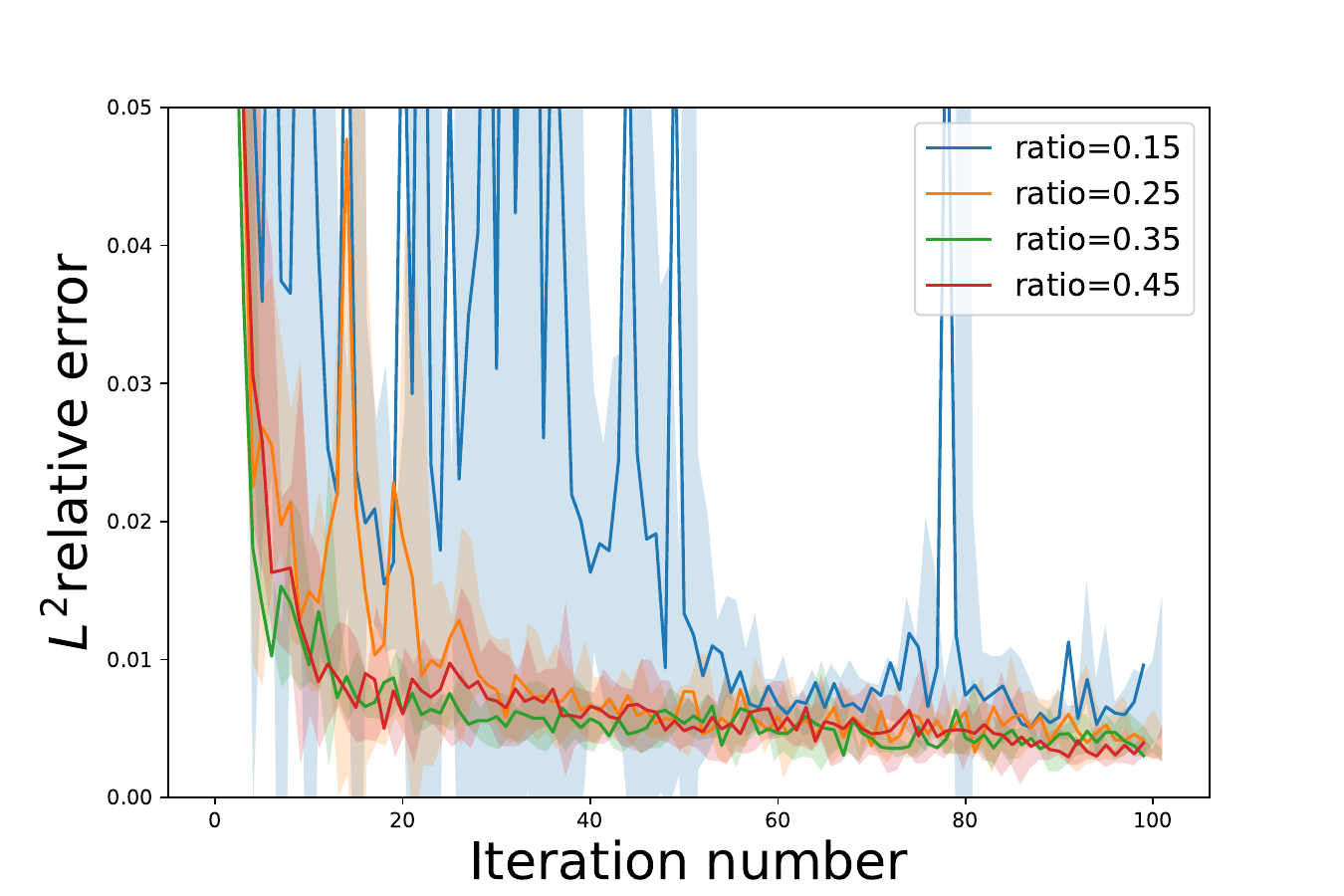}}
    \caption{\textcolor{black}{\textbf{Loss curves for Allen-cahn Equation with different ratios.} Blue line: Mean values of RSmote with a ratio of 0.15; Orange line: Mean values of RSmote with a ratio of 0.25; Green line: Mean values of RSmote with a ratio of 0.35; Red line: Mean values of RSmote with a ratio of 0.45. The shaded areas represent the corresponding standard deviations.}}
    \label{f.allen_oscillation}
\end{figure}

\subsection{\textcolor{black}{Discussion of Fluctuations in Loss Curves}}
\label{ss.fluctuation}

\textcolor{black}{We now discuss the factors contributing to the fluctuations observed in the RSmote method. These factors can be broadly categorized into two aspects: the number of training data points and the nature of the PDE solution.}

\textcolor{black}{Regarding the number of training data points, the loss curves of RSmote exhibit greater fluctuations compared to those of the RAD method when the sample size is 2000. This behavior arises because a small dataset limits the model's ability to approximate the solution accurately and identify regions with large residuals. Additionally, RSmote, being a local algorithm, lacks a mechanism for computing the global residual distribution, as RAD does. Consequently, RSmote requires more iterations to identify regions needing refinement, resulting in fluctuations during the data update process, as shown in Fig.\ref{f.laplace}(a), Fig.\ref{f.burges}(a), and Fig.~\ref{f.allen}(a). However, as the sample size increases, these fluctuations diminish because a larger dataset enables the network to better approximate the solution and more accurately locate regions with large residuals.}

\textcolor{black}{The second factor relates to the characteristics of the PDE itself. For PDEs with smooth solutions, such as the Laplace equation, the residual distribution is less concentrated (see Fig.\ref{f.laplace_field}(d)(e)), making it more challenging for RSmote to correctly identify regions requiring refinement. In contrast, RAD employs a global estimation approach, enabling it to precisely locate regions with large residuals, resulting in smaller fluctuations. However, when the dimensionality increases (see Fig.~\ref{f.elliptic}), both RSmote and RAD exhibit fluctuations. For RSmote, the training data is insufficient to quickly capture the imbalanced properties of the solution, while for RAD, even with a large number of additional data points (100000 and 200000), the sample size remains inadequate to accurately estimate the residual distribution.  When the PDE is particularly challenging, such as the Burgers' equation and Allen-Cahn equation, it includes terms (convection or small diffusion)  that can result in sharp gradients or shock formations. These features make it difficult for the PINN to approximate the sharp transitions uniformly across the domain, leading to oscillations in the loss.  RSmote can efficiently localize the large error region and demonstrates more robust performance with reduced oscillations.
}



\section{Conclusions}
\label{s:conclu}

In this study, we presented RSmote, an innovative approach to adaptive sampling in Physics-Informed Neural Networks. Our method leverages insights from imbalanced learning to address the limitations of global sampling techniques, particularly in high-dimensional spaces. 

The results consistently show that RSmote achieves comparable or better accuracy than the RAD method while using substantially less memory. This efficiency becomes increasingly pronounced as the problem dimension increases, with RSmote maintaining its performance up to 100 dimensions where traditional methods struggle or become computationally prohibitive.

The theoretical analysis and the success of RSmote highlight the potential of local, targeted sampling strategies in PINNs. By focusing computational resources on regions of high residuals, we can achieve high accuracy without the need for exhaustive global sampling. This approach not only improves efficiency but also opens up new possibilities for tackling complex, high-dimensional PDEs that were previously challenging due to computational constraints.

Future work could explore the application of RSmote to a broader range of PDEs and investigate its performance in even higher-dimensional spaces. Additionally, further refinement of the over-sampling and resampling techniques could potentially yield even greater improvements in efficiency and accuracy.

\section*{Acknowledgment}
The research results of this article are partially supported by the National Natural Science Foundation of China  12071190, 12271492 (J.Luo and S. Xu).

\section*{Author contributions}
\textbf{Jiaqi Luo}: Methodology, Software, Writing$ - $original draft, Investigation.
\textbf{Yahong Yang}: Methodology, Theory, Writing$ - $original draft, Investigation.
\textbf{Yuan Yuan}: Software, Writing $ - $ original draft, Visualization.
\textbf{Wenrui Hao}: Conceptualization, Methodology $ - $ original draft, Supervision.
\textbf{Shixin Xu}: Conceptualization, Methodology, Writing $ - $ original draft, Supervision.

\section*{Conflicts of interest}
The authors declare that they have no known competing financial interests or personal relationships that could have appeared to influence the work reported in this paper.

\section*{Declaration of generative AI and AI-assisted technologies in the writing process}
During the preparation of this work, the authors used ChatGPT to improve the language and readability. After using this tool/service, the author(s) reviewed and edited the content as needed and take(s) full responsibility for the content of the published article.

\bibliographystyle{elsarticle-num} 
\bibliography{references}

\begin{thebibliography}{10}
\expandafter\ifx\csname url\endcsname\relax
  \def\url#1{\texttt{#1}}\fi
\expandafter\ifx\csname urlprefix\endcsname\relax\def\urlprefix{URL }\fi
\expandafter\ifx\csname href\endcsname\relax
  \def\href#1#2{#2} \def\path#1{#1}\fi

\bibitem{hughes2012finite}
T.~J. Hughes, The finite element method: linear static and dynamic finite
  element analysis, Courier Corporation, 2012.

\bibitem{thomas2013numerical}
J.~W. Thomas, Numerical partial differential equations: finite difference
  methods, Vol.~22, Springer Science \& Business Media, 2013.

\bibitem{shen2011spectral}
J.~Shen, T.~Tang, L.-L. Wang, Spectral methods: algorithms, analysis and
  applications, Vol.~41, Springer Science \& Business Media, 2011.

\bibitem{hu2024tackling}
Z.~Hu, K.~Shukla, G.~E. Karniadakis, K.~Kawaguchi, Tackling the curse of
  dimensionality with physics-informed neural networks, Neural Networks 176
  (2024) 106369.

\bibitem{powell2007approximate}
W.~B. Powell, Approximate Dynamic Programming: Solving the curses of
  dimensionality, Vol. 703, John Wiley \& Sons, 2007.

\bibitem{he2016deep}
K.~He, X.~Zhang, S.~Ren, J.~Sun, Deep residual learning for image recognition,
  in: Proceedings of the IEEE conference on computer vision and pattern
  recognition, 2016, pp. 770--778.

\bibitem{vaswani2017attention}
A.~Vaswani, N.~Shazeer, N.~Parmar, J.~Uszkoreit, L.~Jones, A.~N. Gomez,
  {\L}.~Kaiser, I.~Polosukhin, Attention is all you need, Advances in neural
  information processing systems 30 (2017).

\bibitem{lecun2015deep}
Y.~LeCun, Y.~Bengio, G.~Hinton, Deep learning, nature 521~(7553) (2015)
  436--444.

\bibitem{long2018pde}
Z.~Long, Y.~Lu, X.~Ma, B.~Dong, Pde-net: Learning pdes from data, in:
  International conference on machine learning, PMLR, 2018, pp. 3208--3216.

\bibitem{long2019pde}
Z.~Long, Y.~Lu, B.~Dong, {PDE-Net} 2.0: Learning {PDEs} from data with a
  numeric-symbolic hybrid deep network, Journal of Computational Physics 399
  (2019) 108925.

\bibitem{sirignano2018dgm}
J.~Sirignano, K.~Spiliopoulos, {DGM}: A deep learning algorithm for solving
  partial differential equations, Journal of computational physics 375 (2018)
  1339--1364.

\bibitem{weinan2018deep}
E.~Weinan, B.~Yu, The deep ritz method: A deep learning-based numerical
  algorithm for solving variational problems, Communications in Mathematics and
  Statistics 6~(1) (2018) 1--12.

\bibitem{zang2020weak}
Y.~Zang, G.~Bao, X.~Ye, H.~Zhou, Weak adversarial networks for high-dimensional
  partial differential equations, Journal of Computational Physics 411 (2020)
  109409.

\bibitem{raissi2019physics}
M.~Raissi, P.~Perdikaris, G.~E. Karniadakis, Physics-informed neural networks:
  A deep learning framework for solving forward and inverse problems involving
  nonlinear partial differential equations, Journal of Computational physics
  378 (2019) 686--707.

\bibitem{karniadakis2021physics}
G.~E. Karniadakis, I.~G. Kevrekidis, L.~Lu, P.~Perdikaris, S.~Wang, L.~Yang,
  Physics-informed machine learning, Nature Reviews Physics 3~(6) (2021)
  422--440.

\bibitem{raissi2020hidden}
M.~Raissi, A.~Yazdani, G.~E. Karniadakis, Hidden fluid mechanics: Learning
  velocity and pressure fields from flow visualizations, Science 367~(6481)
  (2020) 1026--1030.

\bibitem{yazdani2020systems}
A.~Yazdani, L.~Lu, M.~Raissi, G.~E. Karniadakis, Systems biology informed deep
  learning for inferring parameters and hidden dynamics, PLoS computational
  biology 16~(11) (2020) e1007575.

\bibitem{lu2021physics}
L.~Lu, R.~Pestourie, W.~Yao, Z.~Wang, F.~Verdugo, S.~G. Johnson,
  Physics-informed neural networks with hard constraints for inverse design,
  SIAM Journal on Scientific Computing 43~(6) (2021) B1105--B1132.

\bibitem{lu2021deepxde}
L.~Lu, X.~Meng, Z.~Mao, G.~E. Karniadakis, {DeepXDE}: A deep learning library
  for solving differential equations, SIAM review 63~(1) (2021) 208--228.

\bibitem{pang2019fpinns}
G.~Pang, L.~Lu, G.~E. Karniadakis, fpinns: Fractional physics-informed neural
  networks, SIAM Journal on Scientific Computing 41~(4) (2019) A2603--A2626.

\bibitem{zhang2019quantifying}
D.~Zhang, L.~Lu, L.~Guo, G.~E. Karniadakis, Quantifying total uncertainty in
  physics-informed neural networks for solving forward and inverse stochastic
  problems, Journal of Computational Physics 397 (2019) 108850.

\bibitem{zheng2024hompinns}
H.~Zheng, Y.~Huang, Z.~Huang, W.~Hao, G.~Lin, Hompinns: Homotopy
  physics-informed neural networks for solving the inverse problems of
  nonlinear differential equations with multiple solutions, Journal of
  Computational Physics 500 (2024) 112751.

\bibitem{huang2022hompinns}
Y.~Huang, W.~Hao, G.~Lin, Hompinns: Homotopy physics-informed neural networks
  for learning multiple solutions of nonlinear elliptic differential equations,
  Computers \& Mathematics with Applications 121 (2022) 62--73.

\bibitem{siegel2023greedy}
J.~W. Siegel, Q.~Hong, X.~Jin, W.~Hao, J.~Xu, Greedy training algorithms for
  neural networks and applications to pdes, Journal of Computational Physics
  484 (2023) 112084.

\bibitem{chen2022randomized}
Q.~Chen, W.~Hao, Randomized newton’s method for solving differential
  equations based on the neural network discretization, Journal of Scientific
  Computing 92~(2) (2022) 49.

\bibitem{wu2023comprehensive}
C.~Wu, M.~Zhu, Q.~Tan, Y.~Kartha, L.~Lu, A comprehensive study of non-adaptive
  and residual-based adaptive sampling for physics-informed neural networks,
  Computer Methods in Applied Mechanics and Engineering 403 (2023) 115671.

\bibitem{nabian2021efficient}
M.~A. Nabian, R.~J. Gladstone, H.~Meidani, Efficient training of
  physics-informed neural networks via importance sampling, Computer-Aided
  Civil and Infrastructure Engineering 36~(8) (2021) 962--977.

\bibitem{gao2023failure}
Z.~Gao, L.~Yan, T.~Zhou, Failure-informed adaptive sampling for pinns, SIAM
  Journal on Scientific Computing 45~(4) (2023) A1971--A1994.

\bibitem{gao2023failure2}
Z.~Gao, T.~Tang, L.~Yan, T.~Zhou, Failure-informed adaptive sampling for pinns,
  part ii: combining with re-sampling and subset simulation, Communications on
  Applied Mathematics and Computation (2023) 1--22.

\bibitem{gu2021selectnet}
Y.~Gu, H.~Yang, C.~Zhou, Selectnet: Self-paced learning for high-dimensional
  partial differential equations, Journal of Computational Physics 441 (2021)
  110444.

\bibitem{tang2023pinns}
K.~Tang, X.~Wan, C.~Yang, Das-pinns: A deep adaptive sampling method for
  solving high-dimensional partial differential equations, Journal of
  Computational Physics 476 (2023) 111868.

\bibitem{gao2023active}
W.~Gao, C.~Wang, Active learning based sampling for high-dimensional nonlinear
  partial differential equations, Journal of Computational Physics 475 (2023)
  111848.

\bibitem{zeng2022adaptive}
S.~Zeng, Z.~Zhang, Q.~Zou, Adaptive deep neural networks methods for
  high-dimensional partial differential equations, Journal of Computational
  Physics 463 (2022) 111232.

\bibitem{mao2023physics}
Z.~Mao, X.~Meng, Physics-informed neural networks with residual/gradient-based
  adaptive sampling methods for solving partial differential equations with
  sharp solutions, Applied Mathematics and Mechanics 44~(7) (2023) 1069--1084.

\bibitem{he2009learning}
H.~He, E.~A. Garcia, Learning from imbalanced data, IEEE Transactions on
  knowledge and data engineering 21~(9) (2009) 1263--1284.

\bibitem{chawla2002smote}
N.~V. Chawla, K.~W. Bowyer, L.~O. Hall, W.~P. Kegelmeyer, {SMOTE}: synthetic
  minority over-sampling technique, Journal of artificial intelligence research
  16 (2002) 321--357.

\bibitem{kovacs2019empirical}
G.~Kov{\'a}cs, An empirical comparison and evaluation of minority oversampling
  techniques on a large number of imbalanced datasets, Applied Soft Computing
  83 (2019) 105662.

\bibitem{evans2022partial}
L.~Evans, Partial differential equations, Vol.~19, American Mathematical
  Society, 2022.

\bibitem{abu1989vapnik}
Y.~Abu-Mostafa, The vapnik-chervonenkis dimension: Information versus
  complexity in learning, Neural Computation 1~(3) (1989) 312--317.

\bibitem{siegel2022optimal}
J.~Siegel, Optimal approximation rates for deep relu neural networks on sobolev
  spaces, arXiv preprint arXiv:2211.14400 (2022).

\bibitem{lu2021deep}
J.~Lu, Z.~Shen, H.~Yang, S.~Zhang, Deep network approximation for smooth
  functions, SIAM Journal on Mathematical Analysis 53~(5) (2021) 5465--5506.

\bibitem{yang2023nearly}
Y.~Yang, H.~Yang, Y.~Xiang, Nearly optimal vc-dimension and pseudo-dimension
  bounds for deep neural network derivatives, Advances in Neural Information
  Processing Systems 36 (2023) 21721--21756.

\bibitem{yang2023nearlys}
Y.~Yang, Y.~Wu, H.~Yang, Y.~Xiang, Nearly optimal approximation rates for deep
  super relu networks on {Sobolev} spaces, arXiv preprint arXiv:2310.10766
  (2023).

\bibitem{yang2024near}
Y.~Yang, Y.~Lu, Near-optimal deep neural network approximation for korobov
  functions with respect to {$L_p$} and {$H_1$} norms, Neural Networks 180
  (2024) 106702.

\bibitem{pollard1990empirical}
D.~Pollard, Empirical processes: theory and applications, Ims, 1990.

\bibitem{schmidt2020nonparametric}
J.~Schmidt-Hieber, Nonparametric regression using deep neural networks with
  relu activation function (2020).

\bibitem{jiao2021error}
Y.~Jiao, Y.~Lai, Y.~Lo, Y.~Wang, Y.~Yang, Error analysis of deep {Ritz} methods
  for elliptic equations, arXiv preprint arXiv:2107.14478 (2021).

\bibitem{anthony1999neural}
M.~Anthony, P.~Bartlett, P.~Bartlett, et~al., Neural network learning:
  Theoretical foundations, Vol.~9, cambridge university press Cambridge, 1999.

\bibitem{yang2024deeper}
Y.~Yang, J.~He, Deeper or wider: A perspective from optimal generalization
  error with {Sobolev} loss, in: Forty-first International Conference on
  Machine Learning, 2024.

\bibitem{bartlett2019nearly}
P.~Bartlett, N.~Harvey, C.~Liaw, A.~Mehrabian, Nearly-tight vc-dimension and
  pseudodimension bounds for piecewise linear neural networks, Journal of
  Machine Learning Research 20~(63) (2019) 1--17.

\bibitem{gyorfi2002distribution}
L.~Gy{\"o}rfi, M.~Kohler, A.~Krzyzak, H.~Walk, et~al., A distribution-free
  theory of nonparametric regression, Vol.~1, Springer, 2002.

\bibitem{hao2024multiscale}
W.~Hao, R.~Li, Y.~Xi, T.~Xu, Y.~Yang, Multiscale neural networks for
  approximating {Green's} functions, arXiv preprint arXiv:2410.18439 (2024).

\end{thebibliography}

\end{document}